\documentclass{article} 

\usepackage{graphicx}
\usepackage{caption}
\graphicspath{
    {figures/}
}

\usepackage{amsmath,amssymb,enumerate}

\usepackage{amsthm}
\usepackage{setspace}
\usepackage{booktabs}
\usepackage[usenames,dvipsnames,svgnames,table]{xcolor}
\usepackage{mathtools}
\usepackage{blindtext}
\usepackage{gensymb}
\usepackage{xparse}
\usepackage{lipsum}
\usepackage{mathrsfs}
\usepackage[mathscr]{euscript}
\usepackage{times}
\usepackage[numbers,sort&compress]{natbib}
\usepackage{multicol}
\definecolor{darkgreen}{rgb}{0,0.6,0}
\definecolor{cvprblue}{rgb}{0.21,0.49,0.74}
\definecolor{linkcolor}{rgb}{1.0, 0.0 ,0.0}
\usepackage[pagebackref=false,breaklinks=true, colorlinks, citecolor=cvprblue, linkcolor=linkcolor, bookmarks=false]{hyperref}
\usepackage{amsfonts}
\usepackage{cleveref}
\usepackage[utf8]{inputenc}
\usepackage[T1]{fontenc}
\usepackage{textcomp}
\usepackage{arydshln}
\usepackage{tabu}
\usepackage{cuted}
\usepackage{flushend}
\usepackage{balance}

\usepackage{thmtools,thm-restate}

\newtheorem{construction}{Construction}
\newtheorem{problem}{Problem}
\newtheorem{theorem}{Theorem}
\newtheorem{lemma}{Lemma}

\newtheorem{remark}{Remark}

\usepackage{enumitem}

\definecolor{note}{rgb}{0.1,0.1,1}
\definecolor{rephase}{rgb}{0.15,0.7,0.15}
\definecolor{bag}{rgb}{0.6,0.6,0.2}

\makeatletter
\renewcommand*\env@matrix[1][c]{\hskip -\arraycolsep
  \let\@ifnextchar\new@ifnextchar
  \array{*\c@MaxMatrixCols #1}}
\makeatother

\makeatletter
\newcommand{\mathleft}{\@fleqntrue\@mathmargin0pt}
\newcommand{\mathcenter}{\@fleqnfalse}
\makeatother

\usepackage{bm}
\usepackage{cleveref}

\usepackage{amsfonts}
\usepackage{amssymb}
\usepackage{amsbsy}
\usepackage{amsmath}
\usepackage{graphicx}
\usepackage{subcaption}
\usepackage{algorithm}
\usepackage{tablefootnote}
\usepackage{algorithm}
\usepackage{multirow}
\usepackage{algpseudocode}
\usepackage{soul}

\usepackage{wrapfig}
\usepackage{diagbox}
\usepackage{appendix}
\usepackage{graphicx}
\usepackage{subcaption}  

\usepackage{wrapfig}
\usepackage{caption} 

\usepackage{adjustbox}
\usepackage{iclr2025_conference,times}
\usepackage[bottom]{footmisc}

\usepackage{amsmath,amsfonts,bm}

\def\eqref#1{equation~\ref{#1}}

\def\1{\bm{1}}

\DeclareMathAlphabet{\mathsfit}{\encodingdefault}{\sfdefault}{m}{sl}
\SetMathAlphabet{\mathsfit}{bold}{\encodingdefault}{\sfdefault}{bx}{n}

\usepackage{url}

\title{Learning Continuous Neural Representation of Stochastic Hybrid Systems} 

\author{Sangli Teng, Hang Liu, Koushil Sreenath}

\renewcommand{\eqref}[1]{\textup{(\ref{#1})}}

\iclrfinalcopy 
\begin{document}
\newcommand{\TableBaselineMeanHorizonOneSecond}{%
\begin{table*}
\centering
\small
\setlength{\tabcolsep}{6pt}
\renewcommand{\arraystretch}{1.1}
\caption{
Baseline comparison of conditional (unconditional) path distribution loss across 11 benchmark problems.
Per-case means and sample standard deviations are reported in \Cref{tab:baseline-conditional-energy-1s,tab:baseline-trajectory-energy-1s}. $\star$ denotes using $\mathcal{L}_x$ and $\mathcal{L}_\mathrm{x}^{\mathrm{traj}}$.
Horizon: $1\,\mathrm{s}$. }
\label{tab:baseline-mean-1s}
\resizebox{\textwidth}{!}{%
\begin{tabular}{l*{6}{c}}
\toprule
 & Ours & \citep{teng2026embedding} & \citep{li2020scalable} & \citep{li2020scalable}$\star$ & \citep{bartosh2025sde} & \citep{kiyohara2025neural} \\
\midrule
b-ball (GMM)
& $\mathbf{0.257}\;(\mathbf{0.214})$ & Diverged & $0.369\;(1.12)$ & $0.338\;(0.879)$ & $0.620\;(4.14)$ & $0.594\;(4.22)$ \\
b-ball (Uniform)
& $\mathbf{0.289}\;(\mathbf{0.249})$ & Diverged & $0.404\;(1.03)$ & $0.380\;(0.873)$ & $0.732\;(4.84)$ & $0.542\;(3.23)$ \\
Torus (GMM)
& $\mathbf{0.603}\;(\mathbf{0.333})$ & Diverged & $1.84\;(6.17)$ & $4.32\;(11.3)$ & Diverged & $0.909\;(6.27)$ \\
Torus (Uniform)
& $\mathbf{0.560}\;(\mathbf{0.304})$ & Diverged & $2.67\;(40.5)$ & $1.51\;(10.3)$ & Diverged & $0.944\;(7.19)$ \\
Klein (GMM)
& $\mathbf{0.612}\;(\mathbf{0.341})$ & Diverged & Diverged & Diverged & $10.2\;(256)$ & $0.903\;(6.05)$ \\
Klein (Uniform)
& $\mathbf{0.560}\;(\mathbf{0.298})$ & Diverged & $1.33\;(15.1)$ & $1.96\;(29.2)$ & $7.84\;(193)$ & $0.919\;(6.66)$ \\
Klein--Torus
& $\mathbf{0.455}\;(\mathbf{0.257})$ & Diverged & $2.67\;(44.3)$ & $2.97\;(52.3)$ & Diverged & $0.989\;(8.25)$ \\
1-ball (GMM)
& $\mathbf{0.569}\;(\mathbf{0.570})$ & Diverged & $0.770\;(1.68)$ & $0.978\;(4.43)$ & $5.69\;(140)$ & Diverged \\
1-ball (Uniform)
& $\mathbf{0.569}\;(\mathbf{0.525})$ & Diverged & $0.781\;(1.46)$ & $0.751\;(1.76)$ & $6.83\;(192)$ & Diverged \\
2-balls (GMM)
& $\mathbf{0.722}\;(\mathbf{1.04})$ & $1.03\;(2.07)$ & $0.869\;(2.22)$ & $0.883\;(2.70)$ & Diverged & Diverged \\
2-balls (Uniform)
& $\mathbf{0.760}\;(\mathbf{1.07})$ & $1.06\;(2.00)$ & $0.985\;(3.39)$ & $0.984\;(4.23)$ & Diverged & Diverged \\
\bottomrule
\end{tabular}%
}
\vspace{-3mm}
\end{table*}
}

\newcommand{\TableAblationSummaryHorizonOneSecond}{%
\begin{table}
\centering
\footnotesize
\setlength{\tabcolsep}{4pt}
\renewcommand{\arraystretch}{1.15}
\caption{
Ablation summary across 11 benchmark settings on the conditional (unconditional) trajectory distribution loss.
Mean change is the average of per-setting percentage changes relative to the default model;
positive values indicate degradation.
Worse cases count the settings with a higher mean test loss than the default model
(out of 11), using the means over five training seeds.
MLP replaces the diffusion sampler with a deterministic MLP encoder;
loss labels identify the removed objectives.
Detailed results are presented in \Cref{tab:ablation-split-latent-conditional-energy-loss-1s,tab:ablation-split-sampler-conditional-energy-loss-1s,tab:ablation-split-latent-x-trajectory-loss-1s,tab:ablation-split-weights-conditional-energy-loss-1s,tab:ablation-split-weights-x-trajectory-loss-1s}.
}
\label{tab:ablation-summary-1s}
\resizebox{\textwidth}{!}{%
\begin{tabular}{l*{11}{c}}
\toprule
Setting
& $d_z=d_x$
& $d_z=2d_x$
& $d_z=3d_x$
& MLP
& $\mathcal L_{\mathrm{x}}^{\mathrm{traj,c}}$
& $\mathcal L_z$
& $\mathcal L_x$
& $\mathcal L_{\mathrm{x}}^{\mathrm{traj}}$
& $\mathcal L_{\mathrm{x}}^{\mathrm{traj,c}},\ \mathcal L_{\mathrm{x}}^{\mathrm{traj}}$
& $\mathcal L_x,\ \mathcal L_{\mathrm{x}}^{\mathrm{traj}}$
& $\mathcal L_{\mathrm{x}}^{\mathrm{traj,c}},\ \mathcal L_x$ \\
\midrule
Mean change (\%)
& $+14.1\;(+71.4)$
& $+2.9\;(+12.2)$
& $+0.3\;(+4.7)$
& $+3.6\;(+14.7)$
& $+6.2\;(+6.6)$
& $+19.5\;(+246.6)$
& $+1.2\;(+8.8)$
& $+0.0\;(+2.1)$
& $+23.3\;(+54.7)$
& $-0.7\;(+3.8)$
& $+5.7\;(+4.4)$ \\
Worse cases (/11)
& $11\;(11)$
& $8\;(6)$
& $5\;(5)$
& $8\;(8)$
& $7\;(7)$
& $11\;(11)$
& $7\;(9)$
& $6\;(5)$
& $11\;(11)$
& $5\;(7)$
& $6\;(5)$ \\
\bottomrule
\end{tabular}%
}
\vspace{-2mm}
\end{table}
}

\newcommand{\TableBaselineConditionalEnergyHorizonOneSecond}{%
\begin{table*}[!b]
\centering
\small
\setlength{\tabcolsep}{6pt}
\renewcommand{\arraystretch}{1.1}
\caption{
Baseline comparison measured by test-set conditional path distribution loss $\mathcal L_{\mathrm{x}}^{\mathrm{traj,c}}$ ($\downarrow$).
Horizon: $1\,\mathrm{s}$.
}
\label{tab:baseline-conditional-energy-1s}
\resizebox{\textwidth}{!}{%
\begin{tabular}{l*{6}{r@{$\,\pm\,$}l}}
\toprule
Case & \multicolumn{2}{c}{Ours}
& \multicolumn{2}{c}{\citep{teng2026embedding}}
& \multicolumn{2}{c}{\citep{li2020scalable}}
& \multicolumn{2}{c}{\citep{li2020scalable}$\star$}
& \multicolumn{2}{c}{\citep{bartosh2025sde}}
& \multicolumn{2}{c}{\citep{kiyohara2025neural}} \\
\midrule
b-ball (GMM)
& $\mathbf{0.257}$ & $0.0158$
& \multicolumn{2}{c}{Diverged}
& $0.369$ & $0.0464$
& $0.338$ & $0.0269$
& $0.620$ & $0.265$
& $0.594$ & $0.00740$ \\
b-ball (Uniform)
& $\mathbf{0.289}$ & $0.0317$
& \multicolumn{2}{c}{Diverged}
& $0.404$ & $0.0856$
& $0.380$ & $0.0984$
& $0.732$ & $0.295$
& $0.542$ & $0.00676$ \\
Torus (GMM)
& $\mathbf{0.603}$ & $0.00347$
& \multicolumn{2}{c}{Diverged}
& $1.84$ & $2.48$
& $4.32$ & $7.81$
& \multicolumn{2}{c}{Diverged}
& $0.909$ & $0.0214$ \\
Torus (Uniform)
& $\mathbf{0.560}$ & $0.00756$
& \multicolumn{2}{c}{Diverged}
& $2.67$ & $3.93$
& $1.51$ & $1.60$
& \multicolumn{2}{c}{Diverged}
& $0.944$ & $0.0184$ \\
Klein (GMM)
& $\mathbf{0.612}$ & $0.0124$
& \multicolumn{2}{c}{Diverged}
& \multicolumn{2}{c}{Diverged}
& \multicolumn{2}{c}{Diverged}
& $10.2$ & $8.15$
& $0.903$ & $0.0120$ \\
Klein (Uniform)
& $\mathbf{0.560}$ & $0.0178$
& \multicolumn{2}{c}{Diverged}
& $1.33$ & $0.835$
& $1.96$ & $2.27$
& $7.84$ & $6.17$
& $0.919$ & $0.0175$ \\
Klein--Torus
& $\mathbf{0.455}$ & $0.0149$
& \multicolumn{2}{c}{Diverged}
& $2.67$ & $1.86$
& $2.97$ & $2.15$
& \multicolumn{2}{c}{Diverged}
& $0.989$ & $0.0386$ \\
1-ball (GMM)
& $\mathbf{0.569}$ & $0.00536$
& \multicolumn{2}{c}{Diverged}
& $0.770$ & $0.0625$
& $0.978$ & $0.250$
& $5.69$ & $3.34$
& \multicolumn{2}{c}{Diverged} \\
1-ball (Uniform)
& $\mathbf{0.569}$ & $0.0188$
& \multicolumn{2}{c}{Diverged}
& $0.781$ & $0.0768$
& $0.751$ & $0.0752$
& $6.83$ & $2.73$
& \multicolumn{2}{c}{Diverged} \\
2-balls (GMM)
& $\mathbf{0.722}$ & $0.0185$
& $1.03$ & $0.0155$
& $0.869$ & $0.0177$
& $0.883$ & $0.0251$
& \multicolumn{2}{c}{Diverged}
& \multicolumn{2}{c}{Diverged} \\
2-balls (Uniform)
& $\mathbf{0.760}$ & $0.0183$
& $1.06$ & $0.00993$
& $0.985$ & $0.0590$
& $0.984$ & $0.0490$
& \multicolumn{2}{c}{Diverged}
& \multicolumn{2}{c}{Diverged} \\
\bottomrule
\end{tabular}%
}
\end{table*}
}

\newcommand{\TableBaselineTrajectoryEnergyHorizonOneSecond}{%
\begin{table*}
\centering
\scriptsize
\setlength{\tabcolsep}{6pt}
\renewcommand{\arraystretch}{1.1}
\caption{
Baseline comparison measured by test-set unconditional path distribution loss $\mathcal L_{\mathrm{x}}^{\mathrm{traj}}$ ($\downarrow$).
Horizon: $1\,\mathrm{s}$.
}
\label{tab:baseline-trajectory-energy-1s}
\resizebox{\textwidth}{!}{%
\begin{tabular}{l*{6}{r@{$\,\pm\,$}l}}
\toprule
Case & \multicolumn{2}{c}{Ours}
& \multicolumn{2}{c}{\citep{teng2026embedding}}
& \multicolumn{2}{c}{\citep{li2020scalable}}
& \multicolumn{2}{c}{\citep{li2020scalable}$\star$}
& \multicolumn{2}{c}{\citep{bartosh2025sde}}
& \multicolumn{2}{c}{\citep{kiyohara2025neural}} \\
\midrule
b-ball (GMM)
& $\mathbf{0.214}$ & $0.0558$
& \multicolumn{2}{c}{Diverged}
& $1.12$ & $0.571$
& $0.879$ & $0.293$
& $4.14$ & $3.75$
& $4.22$ & $0.0888$ \\
b-ball (Uniform)
& $\mathbf{0.249}$ & $0.0819$
& \multicolumn{2}{c}{Diverged}
& $1.03$ & $0.708$
& $0.873$ & $0.673$
& $4.84$ & $4.54$
& $3.23$ & $0.0527$ \\
Torus (GMM)
& $\mathbf{0.333}$ & $0.0133$
& \multicolumn{2}{c}{Diverged}
& $6.17$ & $9.07$
& $11.3$ & $16.3$
& \multicolumn{2}{c}{Diverged}
& $6.27$ & $0.551$ \\
Torus (Uniform)
& $\mathbf{0.304}$ & $0.0319$
& \multicolumn{2}{c}{Diverged}
& $40.5$ & $76.1$
& $10.3$ & $15.1$
& \multicolumn{2}{c}{Diverged}
& $7.19$ & $0.492$ \\
Klein (GMM)
& $\mathbf{0.341}$ & $0.0201$
& \multicolumn{2}{c}{Diverged}
& \multicolumn{2}{c}{Diverged}
& \multicolumn{2}{c}{Diverged}
& $256$ & $231$
& $6.05$ & $0.330$ \\
Klein (Uniform)
& $\mathbf{0.298}$ & $0.0317$
& \multicolumn{2}{c}{Diverged}
& $15.1$ & $17.5$
& $29.2$ & $49.4$
& $193$ & $171$
& $6.66$ & $0.451$ \\
Klein--Torus
& $\mathbf{0.257}$ & $0.0206$
& \multicolumn{2}{c}{Diverged}
& $44.3$ & $41.8$
& $52.3$ & $48.7$
& \multicolumn{2}{c}{Diverged}
& $8.25$ & $1.08$ \\
1-ball (GMM)
& $\mathbf{0.570}$ & $0.0518$
& \multicolumn{2}{c}{Diverged}
& $1.68$ & $0.441$
& $4.43$ & $2.87$
& $140$ & $125$
& \multicolumn{2}{c}{Diverged} \\
1-ball (Uniform)
& $\mathbf{0.525}$ & $0.0579$
& \multicolumn{2}{c}{Diverged}
& $1.46$ & $0.277$
& $1.76$ & $0.894$
& $192$ & $96.4$
& \multicolumn{2}{c}{Diverged} \\
2-balls (GMM)
& $\mathbf{1.04}$ & $0.162$
& $2.07$ & $0.0647$
& $2.22$ & $0.314$
& $2.70$ & $0.834$
& \multicolumn{2}{c}{Diverged}
& \multicolumn{2}{c}{Diverged} \\
2-balls (Uniform)
& $\mathbf{1.07}$ & $0.0852$
& $2.00$ & $0.115$
& $3.39$ & $1.47$
& $4.23$ & $1.42$
& \multicolumn{2}{c}{Diverged}
& \multicolumn{2}{c}{Diverged} \\
\bottomrule
\end{tabular}%
}
\end{table*}
}

\newcommand{\TableLatentConditionalEnergyHorizonOneSecond}{%
\begin{table*}
\centering
\scriptsize
\setlength{\tabcolsep}{4pt}
\renewcommand{\arraystretch}{1.1}
\caption{
Latent-dimension ablation measured by test-set conditional path distribution loss $\mathcal L_{\mathrm{x}}^{\mathrm{traj,c}}$ ($\downarrow$).
Horizon: $1\,\mathrm{s}$.
}
\label{tab:ablation-split-latent-conditional-energy-loss-1s}
\begin{tabular}{l*{4}{r@{$\,\pm\,$}l}}
\toprule
 & \multicolumn{2}{c}{$d_z=d_x$} & \multicolumn{2}{c}{$d_z=2d_x$} & \multicolumn{2}{c}{$d_z=3d_x$} & \multicolumn{2}{c}{$d_z=4d_x$ (default)} \\
\midrule
b-ball (GMM)
& $0.259$ & $0.0172$
& $0.275$ & $0.0247$
& $0.270$ & $0.0227$
& $\mathbf{0.257}$ & $0.0158$ \\
b-ball (Uniform)
& $0.340$ & $0.0712$
& $0.338$ & $0.0656$
& $0.300$ & $0.0212$
& $\mathbf{0.289}$ & $0.0317$ \\
Torus (GMM)
& $0.696$ & $0.0440$
& $0.614$ & $0.0143$
& $0.608$ & $0.00535$
& $\mathbf{0.603}$ & $0.00347$ \\
Torus (Uniform)
& $0.685$ & $0.0171$
& $0.563$ & $0.0121$
& $\mathbf{0.553}$ & $0.00406$
& $0.560$ & $0.00756$ \\
Klein (GMM)
& $0.689$ & $0.0206$
& $0.630$ & $0.0189$
& $\mathbf{0.611}$ & $0.00336$
& $0.612$ & $0.0124$ \\
Klein (Uniform)
& $0.640$ & $0.0354$
& $0.562$ & $0.00371$
& $\mathbf{0.559}$ & $0.00286$
& $0.560$ & $0.0178$ \\
Klein--Torus
& $0.670$ & $0.0568$
& $0.479$ & $0.0384$
& $\mathbf{0.441}$ & $0.0155$
& $0.455$ & $0.0149$ \\
1-ball (GMM)
& $0.628$ & $0.0163$
& $0.571$ & $0.0156$
& $0.571$ & $0.00761$
& $\mathbf{0.569}$ & $0.00536$ \\
1-ball (Uniform)
& $0.604$ & $0.0210$
& $0.561$ & $0.0158$
& $\mathbf{0.560}$ & $0.00796$
& $0.569$ & $0.0188$ \\
2-balls (GMM)
& $0.753$ & $0.0193$
& $\mathbf{0.710}$ & $0.0105$
& $0.713$ & $0.00721$
& $0.722$ & $0.0185$ \\
2-balls (Uniform)
& $0.796$ & $0.0182$
& $\mathbf{0.759}$ & $0.0209$
& $0.766$ & $0.0217$
& $0.760$ & $0.0183$ \\
\bottomrule
\end{tabular}%
\end{table*}
}

\newcommand{\TableLatentTrajectoryEnergyHorizonOneSecond}{%
\begin{table*}
\centering
\scriptsize
\setlength{\tabcolsep}{4pt}
\renewcommand{\arraystretch}{1.1}
\caption{
Latent-dimension ablation measured by test-set unconditional path distribution loss $\mathcal L_{\mathrm{x}}^{\mathrm{traj}}$ ($\downarrow$).
Horizon: $1\,\mathrm{s}$.
}
\label{tab:ablation-split-latent-x-trajectory-loss-1s}
\begin{tabular}{l*{4}{r@{$\,\pm\,$}l}}
\toprule
 & \multicolumn{2}{c}{$d_z=d_x$} & \multicolumn{2}{c}{$d_z=2d_x$} & \multicolumn{2}{c}{$d_z=3d_x$} & \multicolumn{2}{c}{$d_z=4d_x$ (default)} \\
\midrule
b-ball (GMM)
& $0.274$ & $0.0452$
& $0.311$ & $0.107$
& $0.336$ & $0.147$
& $\mathbf{0.214}$ & $0.0558$ \\
b-ball (Uniform)
& $0.680$ & $0.524$
& $0.460$ & $0.290$
& $0.274$ & $0.0762$
& $\mathbf{0.249}$ & $0.0819$ \\
Torus (GMM)
& $0.877$ & $0.902$
& $0.342$ & $0.0256$
& $0.350$ & $0.0165$
& $\mathbf{0.333}$ & $0.0133$ \\
Torus (Uniform)
& $0.509$ & $0.0629$
& $0.302$ & $0.0256$
& $\mathbf{0.286}$ & $0.0155$
& $0.304$ & $0.0319$ \\
Klein (GMM)
& $0.553$ & $0.119$
& $0.381$ & $0.0473$
& $\mathbf{0.339}$ & $0.0272$
& $0.341$ & $0.0201$ \\
Klein (Uniform)
& $0.445$ & $0.0601$
& $0.296$ & $0.0144$
& $\mathbf{0.296}$ & $0.0223$
& $0.298$ & $0.0317$ \\
Klein--Torus
& $0.551$ & $0.135$
& $0.274$ & $0.0250$
& $\mathbf{0.240}$ & $0.00911$
& $0.257$ & $0.0206$ \\
1-ball (GMM)
& $0.841$ & $0.112$
& $\mathbf{0.553}$ & $0.0625$
& $0.572$ & $0.0374$
& $0.570$ & $0.0518$ \\
1-ball (Uniform)
& $0.675$ & $0.116$
& $0.515$ & $0.0642$
& $\mathbf{0.497}$ & $0.0263$
& $0.525$ & $0.0579$ \\
2-balls (GMM)
& $1.28$ & $0.107$
& $\mathbf{0.917}$ & $0.0574$
& $0.957$ & $0.0808$
& $1.04$ & $0.162$ \\
2-balls (Uniform)
& $1.36$ & $0.175$
& $1.08$ & $0.138$
& $1.14$ & $0.0585$
& $\mathbf{1.07}$ & $0.0852$ \\
\bottomrule
\end{tabular}%
\end{table*}
}

\newcommand{\TableSamplerAblationHorizonOneSecond}{%
\begin{table*}
\centering
\scriptsize
\setlength{\tabcolsep}{4pt}
\renewcommand{\arraystretch}{1.1}
\caption{
Sampler ablation measured by test-set conditional and unconditional path distribution losses ($\downarrow$).
Horizon: $1\,\mathrm{s}$.
}
\label{tab:ablation-split-sampler-conditional-energy-loss-1s}
\label{tab:ablation-split-sampler-x-trajectory-loss-1s}
\begin{tabular}{l*{4}{r@{$\,\pm\,$}l}}
\toprule
 & \multicolumn{4}{c}{DDIM} & \multicolumn{4}{c}{MLP} \\
\cmidrule(lr){2-5}\cmidrule(lr){6-9}
Case & \multicolumn{2}{c}{Conditional} & \multicolumn{2}{c}{Unconditional}
& \multicolumn{2}{c}{Conditional} & \multicolumn{2}{c}{Unconditional} \\
\midrule
b-ball (GMM)
& $\mathbf{0.257}$ & $0.0158$
& $\mathbf{0.214}$ & $0.0558$
& $0.285$ & $0.0162$
& $0.344$ & $0.0983$ \\
b-ball (Uniform)
& $\mathbf{0.289}$ & $0.0317$
& $\mathbf{0.249}$ & $0.0819$
& $0.332$ & $0.0303$
& $0.457$ & $0.122$ \\
Torus (GMM)
& $\mathbf{0.603}$ & $0.00347$
& $\mathbf{0.333}$ & $0.0133$
& $0.610$ & $0.00799$
& $0.342$ & $0.0214$ \\
Torus (Uniform)
& $0.560$ & $0.00756$
& $0.304$ & $0.0319$
& $\mathbf{0.560}$ & $0.00377$
& $\mathbf{0.299}$ & $0.0136$ \\
Klein (GMM)
& $\mathbf{0.612}$ & $0.0124$
& $\mathbf{0.341}$ & $0.0201$
& $0.613$ & $0.00592$
& $0.345$ & $0.00873$ \\
Klein (Uniform)
& $\mathbf{0.560}$ & $0.0178$
& $\mathbf{0.298}$ & $0.0317$
& $0.565$ & $0.00473$
& $0.317$ & $0.0110$ \\
Klein--Torus
& $\mathbf{0.455}$ & $0.0149$
& $\mathbf{0.257}$ & $0.0206$
& $0.489$ & $0.0264$
& $0.281$ & $0.0265$ \\
1-ball (GMM)
& $\mathbf{0.569}$ & $0.00536$
& $\mathbf{0.570}$ & $0.0518$
& $0.588$ & $0.00513$
& $0.591$ & $0.0551$ \\
1-ball (Uniform)
& $\mathbf{0.569}$ & $0.0188$
& $\mathbf{0.525}$ & $0.0579$
& $0.591$ & $0.0151$
& $0.578$ & $0.0835$ \\
2-balls (GMM)
& $0.722$ & $0.0185$
& $1.04$ & $0.162$
& $\mathbf{0.707}$ & $0.00640$
& $\mathbf{0.929}$ & $0.0647$ \\
2-balls (Uniform)
& $0.760$ & $0.0183$
& $1.07$ & $0.0852$
& $\mathbf{0.751}$ & $0.0178$
& $\mathbf{1.02}$ & $0.0815$ \\
\bottomrule
\end{tabular}%
\end{table*}
}

\newcommand{\TableLossAblationConditionalEnergyHorizonOneSecond}{%
\begin{table*}
\centering
\scriptsize
\setlength{\tabcolsep}{4pt}
\renewcommand{\arraystretch}{1.1}
\caption{
Loss-function ablation measured by test-set conditional path distribution loss $\mathcal L_{\mathrm{x}}^{\mathrm{traj,c}}$ ($\downarrow$).
Column labels other than Default model identify the removed objectives (their weights are set to zero).
Values are the mean $\pm$ sample standard deviation over five training seeds.
The lowest mean among finite entries in each row is bold.
Diverged denotes a configuration with nonfinite results or either test metric exceeding $10^3$ in any seed.
Horizon: $1\,\mathrm{s}$.
}
\label{tab:ablation-split-weights-conditional-energy-loss-1s}
\resizebox{\textwidth}{!}{%
\begin{tabular}{l*{8}{r@{$\,\pm\,$}l}}
\toprule
Case & \multicolumn{2}{c}{Default model} & \multicolumn{2}{c}{$\mathcal L_{\mathrm{x}}^{\mathrm{traj,c}}$} & \multicolumn{2}{c}{$\mathcal L_z$} & \multicolumn{2}{c}{$\mathcal L_x$} & \multicolumn{2}{c}{$\mathcal L_{\mathrm{x}}^{\mathrm{traj}}$} & \multicolumn{2}{c}{$\mathcal L_{\mathrm{x}}^{\mathrm{traj,c}},\ \mathcal L_{\mathrm{x}}^{\mathrm{traj}}$} & \multicolumn{2}{c}{$\mathcal L_x,\ \mathcal L_{\mathrm{x}}^{\mathrm{traj}}$} & \multicolumn{2}{c}{$\mathcal L_{\mathrm{x}}^{\mathrm{traj,c}},\ \mathcal L_x$} \\
\midrule
b-ball (GMM)
& $0.257$ & $0.0158$
& $0.255$ & $0.00746$
& $0.288$ & $0.0348$
& $0.280$ & $0.0367$
& $0.261$ & $0.0208$
& $0.270$ & $0.0230$
& $0.258$ & $0.00631$
& $\mathbf{0.254}$ & $0.00987$ \\
b-ball (Uniform)
& $0.289$ & $0.0317$
& $0.292$ & $0.0202$
& $0.335$ & $0.0489$
& $0.297$ & $0.0353$
& $0.298$ & $0.0346$
& $0.317$ & $0.0163$
& $0.288$ & $0.0150$
& $\mathbf{0.283}$ & $0.00573$ \\
Torus (GMM)
& $\mathbf{0.603}$ & $0.00347$
& $0.607$ & $0.00351$
& $0.649$ & $0.0384$
& $0.606$ & $0.00463$
& $0.617$ & $0.00411$
& $0.709$ & $0.00138$
& $0.616$ & $0.00403$
& $0.606$ & $0.00167$ \\
Torus (Uniform)
& $0.560$ & $0.00756$
& $0.558$ & $0.00334$
& $0.647$ & $0.0682$
& $\mathbf{0.553}$ & $0.00432$
& $0.555$ & $0.00296$
& $0.709$ & $0.00129$
& $0.559$ & $0.00389$
& $0.554$ & $0.00216$ \\
Klein (GMM)
& $0.612$ & $0.0124$
& $0.610$ & $0.00808$
& $0.709$ & $0.0595$
& $0.607$ & $0.00778$
& $0.614$ & $0.0108$
& $0.708$ & $0.00172$
& $0.616$ & $0.00606$
& $\mathbf{0.605}$ & $0.00217$ \\
Klein (Uniform)
& $0.560$ & $0.0178$
& $0.559$ & $0.00468$
& $0.718$ & $0.0175$
& $0.562$ & $0.0225$
& $0.564$ & $0.0176$
& $0.704$ & $0.000841$
& $0.564$ & $0.0108$
& $\mathbf{0.554}$ & $0.00259$ \\
Klein--Torus
& $0.455$ & $0.0149$
& $0.473$ & $0.0175$
& $0.723$ & $0.0478$
& $0.455$ & $0.0284$
& $0.450$ & $0.00909$
& $0.713$ & $0.00257$
& $\mathbf{0.446}$ & $0.0152$
& $0.471$ & $0.0118$ \\
1-ball (GMM)
& $\mathbf{0.569}$ & $0.00536$
& $0.635$ & $0.0138$
& $0.650$ & $0.0263$
& $0.578$ & $0.0135$
& $0.583$ & $0.0189$
& $0.695$ & $0.0205$
& $0.585$ & $0.0191$
& $0.637$ & $0.0218$ \\
1-ball (Uniform)
& $0.569$ & $0.0188$
& $0.650$ & $0.0335$
& $0.665$ & $0.0341$
& $0.579$ & $0.0191$
& $0.563$ & $0.0233$
& $0.718$ & $0.0196$
& $\mathbf{0.555}$ & $0.0120$
& $0.655$ & $0.0356$ \\
2-balls (GMM)
& $0.722$ & $0.0185$
& $0.850$ & $0.00765$
& $0.825$ & $0.0775$
& $0.724$ & $0.0156$
& $0.706$ & $0.00963$
& $0.904$ & $0.00528$
& $\mathbf{0.697}$ & $0.00816$
& $0.848$ & $0.00563$ \\
2-balls (Uniform)
& $0.760$ & $0.0183$
& $0.916$ & $0.0119$
& $0.876$ & $0.0635$
& $0.755$ & $0.0213$
& $0.724$ & $0.0168$
& $0.957$ & $0.0157$
& $\mathbf{0.717}$ & $0.00731$
& $0.913$ & $0.0109$ \\
\bottomrule
\end{tabular}%
}
\end{table*}
}

\newcommand{\TableLossAblationUnconditionalEnergyHorizonOneSecond}{%
\begin{table*}
\centering
\scriptsize
\setlength{\tabcolsep}{4pt}
\renewcommand{\arraystretch}{1.1}
\caption{
Loss-function ablation measured by test-set unconditional path distribution loss $\mathcal L_{\mathrm{x}}^{\mathrm{traj}}$ ($\downarrow$).
Column labels other than Default model identify the removed objectives (their weights are set to zero).
Values are the mean $\pm$ sample standard deviation over five training seeds.
The lowest mean among finite entries in each row is bold.
Diverged denotes a configuration with nonfinite results or either test metric exceeding $10^3$ in any seed.
Horizon: $1\,\mathrm{s}$.
}
\label{tab:ablation-split-weights-x-trajectory-loss-1s}
\resizebox{\textwidth}{!}{%
\begin{tabular}{l*{8}{r@{$\,\pm\,$}l}}
\toprule
Case & \multicolumn{2}{c}{Default model} & \multicolumn{2}{c}{$\mathcal L_{\mathrm{x}}^{\mathrm{traj,c}}$} & \multicolumn{2}{c}{$\mathcal L_z$} & \multicolumn{2}{c}{$\mathcal L_x$} & \multicolumn{2}{c}{$\mathcal L_{\mathrm{x}}^{\mathrm{traj}}$} & \multicolumn{2}{c}{$\mathcal L_{\mathrm{x}}^{\mathrm{traj,c}},\ \mathcal L_{\mathrm{x}}^{\mathrm{traj}}$} & \multicolumn{2}{c}{$\mathcal L_x,\ \mathcal L_{\mathrm{x}}^{\mathrm{traj}}$} & \multicolumn{2}{c}{$\mathcal L_{\mathrm{x}}^{\mathrm{traj,c}},\ \mathcal L_x$} \\
\midrule
b-ball (GMM)
& $0.214$ & $0.0558$
& $0.215$ & $0.0395$
& $0.411$ & $0.273$
& $0.379$ & $0.247$
& $\mathbf{0.194}$ & $0.0268$
& $0.308$ & $0.157$
& $0.216$ & $0.0409$
& $0.205$ & $0.0457$ \\
b-ball (Uniform)
& $0.249$ & $0.0819$
& $0.279$ & $0.110$
& $0.408$ & $0.169$
& $0.260$ & $0.113$
& $0.268$ & $0.107$
& $0.335$ & $0.0675$
& $\mathbf{0.218}$ & $0.0577$
& $0.233$ & $0.0397$ \\
Torus (GMM)
& $0.333$ & $0.0133$
& $0.323$ & $0.00586$
& $0.842$ & $0.493$
& $0.348$ & $0.0204$
& $0.386$ & $0.0345$
& $0.507$ & $0.0134$
& $0.381$ & $0.0361$
& $\mathbf{0.322}$ & $0.0132$ \\
Torus (Uniform)
& $0.304$ & $0.0319$
& $0.285$ & $0.00401$
& $1.01$ & $0.709$
& $0.299$ & $0.0205$
& $0.302$ & $0.0136$
& $0.545$ & $0.0112$
& $0.322$ & $0.0211$
& $\mathbf{0.282}$ & $0.00376$ \\
Klein (GMM)
& $0.341$ & $0.0201$
& $0.314$ & $0.0113$
& $1.67$ & $1.02$
& $0.339$ & $0.0124$
& $0.344$ & $0.0334$
& $0.515$ & $0.00833$
& $0.392$ & $0.0368$
& $\mathbf{0.313}$ & $0.0125$ \\
Klein (Uniform)
& $0.298$ & $0.0317$
& $0.284$ & $0.0140$
& $1.60$ & $0.489$
& $0.300$ & $0.0373$
& $0.330$ & $0.0412$
& $0.544$ & $0.00919$
& $0.328$ & $0.0268$
& $\mathbf{0.282}$ & $0.00749$ \\
Klein--Torus
& $0.257$ & $0.0206$
& $0.263$ & $0.0119$
& $1.98$ & $0.479$
& $0.263$ & $0.0310$
& $0.248$ & $0.0195$
& $0.590$ & $0.00776$
& $\mathbf{0.246}$ & $0.0108$
& $0.262$ & $0.0112$ \\
1-ball (GMM)
& $\mathbf{0.570}$ & $0.0518$
& $0.608$ & $0.0743$
& $1.17$ & $0.496$
& $0.585$ & $0.0695$
& $0.621$ & $0.137$
& $0.689$ & $0.0694$
& $0.666$ & $0.143$
& $0.600$ & $0.0626$ \\
1-ball (Uniform)
& $0.525$ & $0.0579$
& $0.634$ & $0.127$
& $1.41$ & $0.417$
& $0.543$ & $0.0551$
& $\mathbf{0.510}$ & $0.0392$
& $0.674$ & $0.0516$
& $0.552$ & $0.0938$
& $0.635$ & $0.109$ \\
2-balls (GMM)
& $1.04$ & $0.162$
& $1.18$ & $0.116$
& $3.13$ & $2.58$
& $1.05$ & $0.0740$
& $1.04$ & $0.140$
& $1.36$ & $0.0672$
& $\mathbf{1.00}$ & $0.0304$
& $1.20$ & $0.0831$ \\
2-balls (Uniform)
& $1.07$ & $0.0852$
& $1.48$ & $0.147$
& $3.18$ & $1.75$
& $1.09$ & $0.147$
& $1.01$ & $0.0732$
& $1.58$ & $0.209$
& $\mathbf{0.996}$ & $0.0853$
& $1.49$ & $0.145$ \\
\bottomrule
\end{tabular}%
}
\end{table*}
}

\newcommand{\TableBaselineMeanHorizonThreeSeconds}{%
\begin{table*}
\centering
\small
\setlength{\tabcolsep}{6pt}
\renewcommand{\arraystretch}{1.1}
\caption{
Baseline comparison of conditional (unconditional) path distribution loss across 11 benchmark problems.
Per-case means and sample standard deviations are reported in \Cref{tab:baseline-conditional-energy-3s,tab:baseline-trajectory-energy-3s}. $\star$ denotes using $\mathcal{L}_x$ and $\mathcal{L}_\mathrm{x}^{\mathrm{traj}}$.
Horizon: $3\,\mathrm{s}$. }
\label{tab:baseline-mean-3s}
\resizebox{\textwidth}{!}{%
\begin{tabular}{l*{6}{c}}
\toprule
 & Ours & \citep{teng2026embedding} & \citep{li2020scalable} & \citep{li2020scalable}$\star$ & \citep{bartosh2025sde} & \citep{kiyohara2025neural} \\
\midrule
b-ball (GMM)
& $\mathbf{0.487}\;(\mathbf{0.732})$ & Diverged & $1.60\;(33.0)$ & $1.18\;(19.3)$ & Diverged & $0.897\;(9.62)$ \\
b-ball (Uniform)
& $\mathbf{0.433}\;(\mathbf{0.690})$ & Diverged & $1.42\;(33.4)$ & $1.45\;(36.0)$ & Diverged & $0.693\;(5.50)$ \\
Torus (GMM)
& $\mathbf{0.671}\;(\mathbf{0.615})$ & Diverged & Diverged & Diverged & Diverged & $0.913\;(11.3)$ \\
Torus (Uniform)
& $\mathbf{0.654}\;(\mathbf{0.595})$ & Diverged & Diverged & Diverged & Diverged & $0.954\;(12.9)$ \\
Klein (GMM)
& $\mathbf{0.677}\;(\mathbf{0.626})$ & Diverged & Diverged & Diverged & Diverged & $0.905\;(10.5)$ \\
Klein (Uniform)
& $\mathbf{0.653}\;(\mathbf{0.644})$ & Diverged & Diverged & Diverged & Diverged & $0.922\;(11.6)$ \\
Klein--Torus
& $\mathbf{0.545}\;(\mathbf{0.555})$ & Diverged & Diverged & Diverged & Diverged & $0.994\;(15.0)$ \\
1-ball (GMM)
& $\mathbf{0.523}\;(\mathbf{0.917})$ & Diverged & $2.98\;(101)$ & Diverged & Diverged & Diverged \\
1-ball (Uniform)
& $\mathbf{0.470}\;(\mathbf{0.776})$ & Diverged & Diverged & $3.33\;(147)$ & Diverged & Diverged \\
2-balls (GMM)
& $\mathbf{0.596}\;(\mathbf{1.53})$ & $0.908\;(3.20)$ & $1.19\;(30.6)$ & $1.28\;(44.2)$ & Diverged & Diverged \\
2-balls (Uniform)
& $\mathbf{0.573}\;(\mathbf{1.65})$ & $0.859\;(2.69)$ & $1.91\;(101)$ & $1.83\;(95.7)$ & Diverged & Diverged \\
\bottomrule
\end{tabular}%
}
\vspace{-3mm}
\end{table*}
}

\newcommand{\TableAblationSummaryHorizonThreeSeconds}{%
\begin{table}
\centering
\footnotesize
\setlength{\tabcolsep}{4pt}
\renewcommand{\arraystretch}{1.15}
\caption{
Ablation summary across 11 benchmark settings on the conditional (unconditional) trajectory distribution loss.
Mean change is the average of per-setting percentage changes relative to the default model;
positive values indicate degradation.
Worse cases count the settings with a higher mean test loss than the default model
(out of 11), using the means over five training seeds.
MLP replaces the diffusion sampler with a deterministic MLP encoder;
loss labels identify the removed objectives.
Per-case latent-dimension and loss-ablation results are reported in \Cref{tab:ablation-split-latent-conditional-energy-loss-3s,tab:ablation-split-latent-x-trajectory-loss-3s,tab:ablation-split-weights-conditional-energy-loss-3s,tab:ablation-split-weights-x-trajectory-loss-3s}.
Per-case conditional and unconditional results for the MLP encoder ablation are reported in \Cref{tab:ablation-split-sampler-conditional-energy-loss-3s}.
Horizon: $3\,\mathrm{s}$.
}
\label{tab:ablation-summary-3s}
\resizebox{\textwidth}{!}{%
\begin{tabular}{l*{11}{c}}
\toprule
Setting
& $d_z=d_x$
& $d_z=2d_x$
& $d_z=3d_x$
& MLP
& $\mathcal L_{\mathrm{x}}^{\mathrm{traj,c}}$
& $\mathcal L_z$
& $\mathcal L_x$
& $\mathcal L_{\mathrm{x}}^{\mathrm{traj}}$
& $\mathcal L_{\mathrm{x}}^{\mathrm{traj,c}},\ \mathcal L_{\mathrm{x}}^{\mathrm{traj}}$
& $\mathcal L_x,\ \mathcal L_{\mathrm{x}}^{\mathrm{traj}}$
& $\mathcal L_{\mathrm{x}}^{\mathrm{traj,c}},\ \mathcal L_x$ \\
\midrule
Mean change (\%)
& $+25.9\;(+417.1)$
& $+1.5\;(+15.0)$
& $+0.7\;(+14.4)$
& $+2.9\;(+13.9)$
& $+4.4\;(+2.6)$
& $+63.8\;(+1611.9)$
& $+0.5\;(+6.4)$
& $-0.7\;(+2.7)$
& $+12.5\;(+24.6)$
& $-1.1\;(+5.1)$
& $+3.9\;(-0.0)$ \\
Worse cases (/11)
& $11\;(11)$
& $5\;(7)$
& $5\;(6)$
& $7\;(9)$
& $8\;(5)$
& $11\;(11)$
& $6\;(7)$
& $6\;(7)$
& $11\;(11)$
& $6\;(7)$
& $7\;(4)$ \\
\bottomrule
\end{tabular}%
}
\vspace{-2mm}
\end{table}
}

\newcommand{\TableBaselineConditionalEnergyHorizonThreeSeconds}{%
\begin{table*}[!b]
\centering
\small
\setlength{\tabcolsep}{6pt}
\renewcommand{\arraystretch}{1.1}
\caption{
Baseline comparison measured by test-set conditional path distribution loss $\mathcal L_{\mathrm{x}}^{\mathrm{traj,c}}$ ($\downarrow$).
Horizon: $3\,\mathrm{s}$.
}
\label{tab:baseline-conditional-energy-3s}
\resizebox{\textwidth}{!}{%
\begin{tabular}{l*{6}{r@{$\,\pm\,$}l}}
\toprule
Case & \multicolumn{2}{c}{Ours}
& \multicolumn{2}{c}{\citep{teng2026embedding}}
& \multicolumn{2}{c}{\citep{li2020scalable}}
& \multicolumn{2}{c}{\citep{li2020scalable}$\star$}
& \multicolumn{2}{c}{\citep{bartosh2025sde}}
& \multicolumn{2}{c}{\citep{kiyohara2025neural}} \\
\midrule
b-ball (GMM)
& $\mathbf{0.487}$ & $0.0153$
& \multicolumn{2}{c}{Diverged}
& $1.60$ & $0.647$
& $1.18$ & $0.390$
& \multicolumn{2}{c}{Diverged}
& $0.897$ & $0.0116$ \\
b-ball (Uniform)
& $\mathbf{0.433}$ & $0.0152$
& \multicolumn{2}{c}{Diverged}
& $1.42$ & $1.04$
& $1.45$ & $1.70$
& \multicolumn{2}{c}{Diverged}
& $0.693$ & $0.00792$ \\
Torus (GMM)
& $\mathbf{0.671}$ & $0.00239$
& \multicolumn{2}{c}{Diverged}
& \multicolumn{2}{c}{Diverged}
& \multicolumn{2}{c}{Diverged}
& \multicolumn{2}{c}{Diverged}
& $0.913$ & $0.0227$ \\
Torus (Uniform)
& $\mathbf{0.654}$ & $0.00691$
& \multicolumn{2}{c}{Diverged}
& \multicolumn{2}{c}{Diverged}
& \multicolumn{2}{c}{Diverged}
& \multicolumn{2}{c}{Diverged}
& $0.954$ & $0.0201$ \\
Klein (GMM)
& $\mathbf{0.677}$ & $0.00530$
& \multicolumn{2}{c}{Diverged}
& \multicolumn{2}{c}{Diverged}
& \multicolumn{2}{c}{Diverged}
& \multicolumn{2}{c}{Diverged}
& $0.905$ & $0.0150$ \\
Klein (Uniform)
& $\mathbf{0.653}$ & $0.0108$
& \multicolumn{2}{c}{Diverged}
& \multicolumn{2}{c}{Diverged}
& \multicolumn{2}{c}{Diverged}
& \multicolumn{2}{c}{Diverged}
& $0.922$ & $0.0189$ \\
Klein--Torus
& $\mathbf{0.545}$ & $0.0325$
& \multicolumn{2}{c}{Diverged}
& \multicolumn{2}{c}{Diverged}
& \multicolumn{2}{c}{Diverged}
& \multicolumn{2}{c}{Diverged}
& $0.994$ & $0.0433$ \\
1-ball (GMM)
& $\mathbf{0.523}$ & $0.00783$
& \multicolumn{2}{c}{Diverged}
& $2.98$ & $2.09$
& \multicolumn{2}{c}{Diverged}
& \multicolumn{2}{c}{Diverged}
& \multicolumn{2}{c}{Diverged} \\
1-ball (Uniform)
& $\mathbf{0.470}$ & $0.0265$
& \multicolumn{2}{c}{Diverged}
& \multicolumn{2}{c}{Diverged}
& $3.33$ & $1.84$
& \multicolumn{2}{c}{Diverged}
& \multicolumn{2}{c}{Diverged} \\
2-balls (GMM)
& $\mathbf{0.596}$ & $0.0114$
& $0.908$ & $0.00794$
& $1.19$ & $0.139$
& $1.28$ & $0.309$
& \multicolumn{2}{c}{Diverged}
& \multicolumn{2}{c}{Diverged} \\
2-balls (Uniform)
& $\mathbf{0.573}$ & $0.0195$
& $0.859$ & $0.00995$
& $1.91$ & $0.806$
& $1.83$ & $0.317$
& \multicolumn{2}{c}{Diverged}
& \multicolumn{2}{c}{Diverged} \\
\bottomrule
\end{tabular}%
}
\end{table*}
}

\newcommand{\TableBaselineTrajectoryEnergyHorizonThreeSeconds}{%
\begin{table*}
\centering
\scriptsize
\setlength{\tabcolsep}{6pt}
\renewcommand{\arraystretch}{1.1}
\caption{
Baseline comparison measured by test-set unconditional path distribution loss $\mathcal L_{\mathrm{x}}^{\mathrm{traj}}$ ($\downarrow$).
Horizon: $3\,\mathrm{s}$.
}
\label{tab:baseline-trajectory-energy-3s}
\resizebox{\textwidth}{!}{%
\begin{tabular}{l*{6}{r@{$\,\pm\,$}l}}
\toprule
Case & \multicolumn{2}{c}{Ours}
& \multicolumn{2}{c}{\citep{teng2026embedding}}
& \multicolumn{2}{c}{\citep{li2020scalable}}
& \multicolumn{2}{c}{\citep{li2020scalable}$\star$}
& \multicolumn{2}{c}{\citep{bartosh2025sde}}
& \multicolumn{2}{c}{\citep{kiyohara2025neural}} \\
\midrule
b-ball (GMM)
& $\mathbf{0.732}$ & $0.137$
& \multicolumn{2}{c}{Diverged}
& $33.0$ & $21.3$
& $19.3$ & $12.8$
& \multicolumn{2}{c}{Diverged}
& $9.62$ & $0.281$ \\
b-ball (Uniform)
& $\mathbf{0.690}$ & $0.156$
& \multicolumn{2}{c}{Diverged}
& $33.4$ & $39.3$
& $36.0$ & $65.9$
& \multicolumn{2}{c}{Diverged}
& $5.50$ & $0.0997$ \\
Torus (GMM)
& $\mathbf{0.615}$ & $0.0293$
& \multicolumn{2}{c}{Diverged}
& \multicolumn{2}{c}{Diverged}
& \multicolumn{2}{c}{Diverged}
& \multicolumn{2}{c}{Diverged}
& $11.3$ & $1.01$ \\
Torus (Uniform)
& $\mathbf{0.595}$ & $0.0462$
& \multicolumn{2}{c}{Diverged}
& \multicolumn{2}{c}{Diverged}
& \multicolumn{2}{c}{Diverged}
& \multicolumn{2}{c}{Diverged}
& $12.9$ & $0.977$ \\
Klein (GMM)
& $\mathbf{0.626}$ & $0.0280$
& \multicolumn{2}{c}{Diverged}
& \multicolumn{2}{c}{Diverged}
& \multicolumn{2}{c}{Diverged}
& \multicolumn{2}{c}{Diverged}
& $10.5$ & $0.729$ \\
Klein (Uniform)
& $\mathbf{0.644}$ & $0.0624$
& \multicolumn{2}{c}{Diverged}
& \multicolumn{2}{c}{Diverged}
& \multicolumn{2}{c}{Diverged}
& \multicolumn{2}{c}{Diverged}
& $11.6$ & $0.890$ \\
Klein--Torus
& $\mathbf{0.555}$ & $0.0695$
& \multicolumn{2}{c}{Diverged}
& \multicolumn{2}{c}{Diverged}
& \multicolumn{2}{c}{Diverged}
& \multicolumn{2}{c}{Diverged}
& $15.0$ & $2.14$ \\
1-ball (GMM)
& $\mathbf{0.917}$ & $0.115$
& \multicolumn{2}{c}{Diverged}
& $101$ & $77.5$
& \multicolumn{2}{c}{Diverged}
& \multicolumn{2}{c}{Diverged}
& \multicolumn{2}{c}{Diverged} \\
1-ball (Uniform)
& $\mathbf{0.776}$ & $0.128$
& \multicolumn{2}{c}{Diverged}
& \multicolumn{2}{c}{Diverged}
& $147$ & $106$
& \multicolumn{2}{c}{Diverged}
& \multicolumn{2}{c}{Diverged} \\
2-balls (GMM)
& $\mathbf{1.53}$ & $0.303$
& $3.20$ & $0.280$
& $30.6$ & $12.7$
& $44.2$ & $28.4$
& \multicolumn{2}{c}{Diverged}
& \multicolumn{2}{c}{Diverged} \\
2-balls (Uniform)
& $\mathbf{1.65}$ & $0.398$
& $2.69$ & $0.182$
& $101$ & $78.7$
& $95.7$ & $30.8$
& \multicolumn{2}{c}{Diverged}
& \multicolumn{2}{c}{Diverged} \\
\bottomrule
\end{tabular}%
}
\end{table*}
}

\newcommand{\TableLatentConditionalEnergyHorizonThreeSeconds}{%
\begin{table*}
\centering
\scriptsize
\setlength{\tabcolsep}{4pt}
\renewcommand{\arraystretch}{1.1}
\caption{
Latent-dimension ablation measured by test-set conditional path distribution loss $\mathcal L_{\mathrm{x}}^{\mathrm{traj,c}}$ ($\downarrow$).
Horizon: $3\,\mathrm{s}$.
}
\label{tab:ablation-split-latent-conditional-energy-loss-3s}
\begin{tabular}{l*{4}{r@{$\,\pm\,$}l}}
\toprule
 & \multicolumn{2}{c}{$d_z=d_x$} & \multicolumn{2}{c}{$d_z=2d_x$} & \multicolumn{2}{c}{$d_z=3d_x$} & \multicolumn{2}{c}{$d_z=4d_x$ (default)} \\
\midrule
b-ball (GMM)
& $0.532$ & $0.0294$
& $0.505$ & $0.0345$
& $0.529$ & $0.0523$
& $\mathbf{0.487}$ & $0.0153$ \\
b-ball (Uniform)
& $1.35$ & $1.97$
& $0.486$ & $0.0746$
& $0.438$ & $0.0267$
& $\mathbf{0.433}$ & $0.0152$ \\
Torus (GMM)
& $0.722$ & $0.0559$
& $0.675$ & $0.00497$
& $0.673$ & $0.00219$
& $\mathbf{0.671}$ & $0.00239$ \\
Torus (Uniform)
& $0.701$ & $0.00576$
& $0.652$ & $0.00565$
& $\mathbf{0.645}$ & $0.00320$
& $0.654$ & $0.00691$ \\
Klein (GMM)
& $0.705$ & $0.00810$
& $0.683$ & $0.00786$
& $\mathbf{0.675}$ & $0.00166$
& $0.677$ & $0.00530$ \\
Klein (Uniform)
& $0.686$ & $0.0136$
& $\mathbf{0.651}$ & $0.00353$
& $0.653$ & $0.00316$
& $0.653$ & $0.0108$ \\
Klein--Torus
& $0.699$ & $0.0149$
& $0.578$ & $0.0543$
& $\mathbf{0.534}$ & $0.0246$
& $0.545$ & $0.0325$ \\
1-ball (GMM)
& $0.549$ & $0.0149$
& $\mathbf{0.515}$ & $0.00656$
& $0.522$ & $0.00300$
& $0.523$ & $0.00783$ \\
1-ball (Uniform)
& $0.476$ & $0.0121$
& $\mathbf{0.450}$ & $0.00760$
& $0.458$ & $0.00677$
& $0.470$ & $0.0265$ \\
2-balls (GMM)
& $0.611$ & $0.0105$
& $\mathbf{0.596}$ & $0.0245$
& $0.620$ & $0.0293$
& $0.596$ & $0.0114$ \\
2-balls (Uniform)
& $0.583$ & $0.0140$
& $\mathbf{0.564}$ & $0.0226$
& $0.570$ & $0.0164$
& $0.573$ & $0.0195$ \\
\bottomrule
\end{tabular}%
\end{table*}
}

\newcommand{\TableLatentTrajectoryEnergyHorizonThreeSeconds}{%
\begin{table*}
\centering
\scriptsize
\setlength{\tabcolsep}{4pt}
\renewcommand{\arraystretch}{1.1}
\caption{
Latent-dimension ablation measured by test-set unconditional path distribution loss $\mathcal L_{\mathrm{x}}^{\mathrm{traj}}$ ($\downarrow$).
Horizon: $3\,\mathrm{s}$.
}
\label{tab:ablation-split-latent-x-trajectory-loss-3s}
\begin{tabular}{l*{4}{r@{$\,\pm\,$}l}}
\toprule
 & \multicolumn{2}{c}{$d_z=d_x$} & \multicolumn{2}{c}{$d_z=2d_x$} & \multicolumn{2}{c}{$d_z=3d_x$} & \multicolumn{2}{c}{$d_z=4d_x$ (default)} \\
\midrule
b-ball (GMM)
& $1.39$ & $0.375$
& $1.03$ & $0.377$
& $1.37$ & $0.709$
& $\mathbf{0.732}$ & $0.137$ \\
b-ball (Uniform)
& $28.4$ & $60.3$
& $1.18$ & $0.838$
& $\mathbf{0.644}$ & $0.265$
& $0.690$ & $0.156$ \\
Torus (GMM)
& $2.17$ & $3.18$
& $0.623$ & $0.0324$
& $0.643$ & $0.0275$
& $\mathbf{0.615}$ & $0.0293$ \\
Torus (Uniform)
& $0.785$ & $0.0796$
& $0.580$ & $0.0555$
& $\mathbf{0.560}$ & $0.0218$
& $0.595$ & $0.0462$ \\
Klein (GMM)
& $0.897$ & $0.167$
& $0.689$ & $0.0678$
& $0.638$ & $0.0407$
& $\mathbf{0.626}$ & $0.0280$ \\
Klein (Uniform)
& $0.786$ & $0.0601$
& $\mathbf{0.623}$ & $0.0258$
& $0.631$ & $0.0214$
& $0.644$ & $0.0624$ \\
Klein--Torus
& $0.976$ & $0.170$
& $0.584$ & $0.0639$
& $\mathbf{0.523}$ & $0.0282$
& $0.555$ & $0.0695$ \\
1-ball (GMM)
& $1.14$ & $0.147$
& $\mathbf{0.823}$ & $0.0901$
& $0.951$ & $0.0541$
& $0.917$ & $0.115$ \\
1-ball (Uniform)
& $0.853$ & $0.134$
& $\mathbf{0.687}$ & $0.0627$
& $0.731$ & $0.0242$
& $0.776$ & $0.128$ \\
2-balls (GMM)
& $1.82$ & $0.135$
& $2.41$ & $2.23$
& $2.84$ & $1.40$
& $\mathbf{1.53}$ & $0.303$ \\
2-balls (Uniform)
& $1.67$ & $0.101$
& $1.77$ & $0.800$
& $1.68$ & $0.315$
& $\mathbf{1.65}$ & $0.398$ \\
\bottomrule
\end{tabular}%
\end{table*}
}

\newcommand{\TableSamplerAblationHorizonThreeSeconds}{%
\begin{table*}
\centering
\scriptsize
\setlength{\tabcolsep}{4pt}
\renewcommand{\arraystretch}{1.1}
\caption{
Sampler ablation measured by test-set conditional and unconditional path distribution losses ($\downarrow$).
Horizon: $3\,\mathrm{s}$.
}
\label{tab:ablation-split-sampler-conditional-energy-loss-3s}
\label{tab:ablation-split-sampler-x-trajectory-loss-3s}
\begin{tabular}{l*{4}{r@{$\,\pm\,$}l}}
\toprule
 & \multicolumn{4}{c}{DDIM} & \multicolumn{4}{c}{MLP} \\
\cmidrule(lr){2-5}\cmidrule(lr){6-9}
Case & \multicolumn{2}{c}{Conditional} & \multicolumn{2}{c}{Unconditional}
& \multicolumn{2}{c}{Conditional} & \multicolumn{2}{c}{Unconditional} \\
\midrule
b-ball (GMM)
& $\mathbf{0.487}$ & $0.0153$
& $\mathbf{0.732}$ & $0.137$
& $0.544$ & $0.0305$
& $1.43$ & $0.570$ \\
b-ball (Uniform)
& $\mathbf{0.433}$ & $0.0152$
& $\mathbf{0.690}$ & $0.156$
& $0.474$ & $0.0166$
& $0.971$ & $0.229$ \\
Torus (GMM)
& $\mathbf{0.671}$ & $0.00239$
& $\mathbf{0.615}$ & $0.0293$
& $0.674$ & $0.00270$
& $0.623$ & $0.0338$ \\
Torus (Uniform)
& $0.654$ & $0.00691$
& $\mathbf{0.595}$ & $0.0462$
& $\mathbf{0.649}$ & $0.00262$
& $0.607$ & $0.0386$ \\
Klein (GMM)
& $\mathbf{0.677}$ & $0.00530$
& $\mathbf{0.626}$ & $0.0280$
& $0.677$ & $0.00337$
& $0.644$ & $0.0320$ \\
Klein (Uniform)
& $0.653$ & $0.0108$
& $\mathbf{0.644}$ & $0.0624$
& $\mathbf{0.652}$ & $0.00358$
& $0.674$ & $0.0121$ \\
Klein--Torus
& $\mathbf{0.545}$ & $0.0325$
& $\mathbf{0.555}$ & $0.0695$
& $0.589$ & $0.0435$
& $0.619$ & $0.0820$ \\
1-ball (GMM)
& $\mathbf{0.523}$ & $0.00783$
& $\mathbf{0.917}$ & $0.115$
& $0.543$ & $0.00717$
& $0.946$ & $0.0601$ \\
1-ball (Uniform)
& $\mathbf{0.470}$ & $0.0265$
& $\mathbf{0.776}$ & $0.128$
& $0.479$ & $0.00800$
& $0.795$ & $0.0770$ \\
2-balls (GMM)
& $0.596$ & $0.0114$
& $1.53$ & $0.303$
& $\mathbf{0.591}$ & $0.00780$
& $\mathbf{1.50}$ & $0.122$ \\
2-balls (Uniform)
& $0.573$ & $0.0195$
& $1.65$ & $0.398$
& $\mathbf{0.562}$ & $0.0110$
& $\mathbf{1.51}$ & $0.144$ \\
\bottomrule
\end{tabular}%
\end{table*}
}

\newcommand{\TableLossAblationConditionalEnergyHorizonThreeSeconds}{%
\begin{table*}
\centering
\scriptsize
\setlength{\tabcolsep}{4pt}
\renewcommand{\arraystretch}{1.1}
\caption{
Loss-function ablation measured by test-set conditional path distribution loss $\mathcal L_{\mathrm{x}}^{\mathrm{traj,c}}$ ($\downarrow$).
Column labels other than Default model identify the removed objectives (their weights are set to zero).
Values are the mean $\pm$ sample standard deviation over five training seeds.
The lowest mean among finite entries in each row is bold.
Diverged denotes a configuration with nonfinite results or either test metric exceeding $10^3$ in any seed.
Horizon: $3\,\mathrm{s}$.
}
\label{tab:ablation-split-weights-conditional-energy-loss-3s}
\resizebox{\textwidth}{!}{%
\begin{tabular}{l*{8}{r@{$\,\pm\,$}l}}
\toprule
Case & \multicolumn{2}{c}{Default model} & \multicolumn{2}{c}{$\mathcal L_{\mathrm{x}}^{\mathrm{traj,c}}$} & \multicolumn{2}{c}{$\mathcal L_z$} & \multicolumn{2}{c}{$\mathcal L_x$} & \multicolumn{2}{c}{$\mathcal L_{\mathrm{x}}^{\mathrm{traj}}$} & \multicolumn{2}{c}{$\mathcal L_{\mathrm{x}}^{\mathrm{traj,c}},\ \mathcal L_{\mathrm{x}}^{\mathrm{traj}}$} & \multicolumn{2}{c}{$\mathcal L_x,\ \mathcal L_{\mathrm{x}}^{\mathrm{traj}}$} & \multicolumn{2}{c}{$\mathcal L_{\mathrm{x}}^{\mathrm{traj,c}},\ \mathcal L_x$} \\
\midrule
b-ball (GMM)
& $\mathbf{0.487}$ & $0.0153$
& $0.508$ & $0.0103$
& $0.588$ & $0.0534$
& $0.518$ & $0.0343$
& $0.498$ & $0.0116$
& $0.518$ & $0.0274$
& $0.497$ & $0.0185$
& $0.505$ & $0.00621$ \\
b-ball (Uniform)
& $0.433$ & $0.0152$
& $0.446$ & $0.0268$
& $0.508$ & $0.0556$
& $0.436$ & $0.0227$
& $0.434$ & $0.0140$
& $0.456$ & $0.0131$
& $0.438$ & $0.0137$
& $\mathbf{0.432}$ & $0.0140$ \\
Torus (GMM)
& $\mathbf{0.671}$ & $0.00239$
& $0.673$ & $0.00159$
& $0.731$ & $0.0464$
& $0.673$ & $0.00224$
& $0.678$ & $0.00180$
& $0.707$ & $0.000957$
& $0.677$ & $0.00300$
& $0.672$ & $0.00146$ \\
Torus (Uniform)
& $0.654$ & $0.00691$
& $0.648$ & $0.00229$
& $0.730$ & $0.0486$
& $0.647$ & $0.00285$
& $0.647$ & $0.00234$
& $0.708$ & $0.000933$
& $0.652$ & $0.00284$
& $\mathbf{0.645}$ & $0.00170$ \\
Klein (GMM)
& $0.677$ & $0.00530$
& $0.674$ & $0.00478$
& $0.813$ & $0.0708$
& $0.675$ & $0.00368$
& $0.678$ & $0.00394$
& $0.709$ & $0.00120$
& $0.679$ & $0.000766$
& $\mathbf{0.673}$ & $0.00137$ \\
Klein (Uniform)
& $0.653$ & $0.0108$
& $0.651$ & $0.00330$
& $0.779$ & $0.0504$
& $0.652$ & $0.0112$
& $0.654$ & $0.0104$
& $0.706$ & $0.000503$
& $0.655$ & $0.00693$
& $\mathbf{0.647}$ & $0.00151$ \\
Klein--Torus
& $0.545$ & $0.0325$
& $0.555$ & $0.0234$
& $0.832$ & $0.0164$
& $0.542$ & $0.0406$
& $0.541$ & $0.0183$
& $0.709$ & $0.000519$
& $\mathbf{0.525}$ & $0.00640$
& $0.555$ & $0.0309$ \\
1-ball (GMM)
& $\mathbf{0.523}$ & $0.00783$
& $0.564$ & $0.0177$
& $0.740$ & $0.0965$
& $0.524$ & $0.00940$
& $0.526$ & $0.0118$
& $0.604$ & $0.0139$
& $0.531$ & $0.0134$
& $0.563$ & $0.0173$ \\
1-ball (Uniform)
& $0.470$ & $0.0265$
& $0.513$ & $0.0183$
& $0.768$ & $0.171$
& $0.476$ & $0.0189$
& $0.456$ & $0.0119$
& $0.564$ & $0.0160$
& $\mathbf{0.451}$ & $0.00761$
& $0.518$ & $0.0204$ \\
2-balls (GMM)
& $0.596$ & $0.0114$
& $0.655$ & $0.00479$
& $2.28$ & $1.55$
& $0.597$ & $0.0119$
& $0.586$ & $0.00568$
& $0.687$ & $0.00468$
& $\mathbf{0.576}$ & $0.00717$
& $0.653$ & $0.00447$ \\
2-balls (Uniform)
& $0.573$ & $0.0195$
& $0.649$ & $0.00529$
& $1.51$ & $0.758$
& $0.565$ & $0.0170$
& $0.544$ & $0.0148$
& $0.681$ & $0.00810$
& $\mathbf{0.532}$ & $0.00465$
& $0.646$ & $0.00471$ \\
\bottomrule
\end{tabular}%
}
\end{table*}
}

\newcommand{\TableLossAblationUnconditionalEnergyHorizonThreeSeconds}{%
\begin{table*}
\centering
\scriptsize
\setlength{\tabcolsep}{4pt}
\renewcommand{\arraystretch}{1.1}
\caption{
Loss-function ablation measured by test-set unconditional path distribution loss $\mathcal L_{\mathrm{x}}^{\mathrm{traj}}$ ($\downarrow$).
Column labels other than Default model identify the removed objectives (their weights are set to zero).
Values are the mean $\pm$ sample standard deviation over five training seeds.
The lowest mean among finite entries in each row is bold.
Diverged denotes a configuration with nonfinite results or either test metric exceeding $10^3$ in any seed.
Horizon: $3\,\mathrm{s}$.
}
\label{tab:ablation-split-weights-x-trajectory-loss-3s}
\resizebox{\textwidth}{!}{%
\begin{tabular}{l*{8}{r@{$\,\pm\,$}l}}
\toprule
Case & \multicolumn{2}{c}{Default model} & \multicolumn{2}{c}{$\mathcal L_{\mathrm{x}}^{\mathrm{traj,c}}$} & \multicolumn{2}{c}{$\mathcal L_z$} & \multicolumn{2}{c}{$\mathcal L_x$} & \multicolumn{2}{c}{$\mathcal L_{\mathrm{x}}^{\mathrm{traj}}$} & \multicolumn{2}{c}{$\mathcal L_{\mathrm{x}}^{\mathrm{traj,c}},\ \mathcal L_{\mathrm{x}}^{\mathrm{traj}}$} & \multicolumn{2}{c}{$\mathcal L_x,\ \mathcal L_{\mathrm{x}}^{\mathrm{traj}}$} & \multicolumn{2}{c}{$\mathcal L_{\mathrm{x}}^{\mathrm{traj,c}},\ \mathcal L_x$} \\
\midrule
b-ball (GMM)
& $\mathbf{0.732}$ & $0.137$
& $0.903$ & $0.187$
& $2.25$ & $0.908$
& $1.18$ & $0.268$
& $0.867$ & $0.134$
& $1.07$ & $0.568$
& $0.954$ & $0.265$
& $0.842$ & $0.151$ \\
b-ball (Uniform)
& $0.690$ & $0.156$
& $0.646$ & $0.216$
& $1.62$ & $1.19$
& $0.696$ & $0.176$
& $0.717$ & $0.0934$
& $0.710$ & $0.121$
& $0.783$ & $0.248$
& $\mathbf{0.572}$ & $0.150$ \\
Torus (GMM)
& $0.615$ & $0.0293$
& $\mathbf{0.603}$ & $0.0168$
& $2.81$ & $1.70$
& $0.651$ & $0.0348$
& $0.695$ & $0.0517$
& $0.741$ & $0.0152$
& $0.680$ & $0.0673$
& $0.604$ & $0.0211$ \\
Torus (Uniform)
& $0.595$ & $0.0462$
& $0.551$ & $0.00996$
& $2.81$ & $1.53$
& $0.589$ & $0.0381$
& $0.586$ & $0.0239$
& $0.776$ & $0.0203$
& $0.633$ & $0.0439$
& $\mathbf{0.543}$ & $0.00864$ \\
Klein (GMM)
& $0.626$ & $0.0280$
& $\mathbf{0.586}$ & $0.0143$
& $5.89$ & $3.80$
& $0.640$ & $0.0394$
& $0.654$ & $0.0503$
& $0.733$ & $0.0108$
& $0.703$ & $0.0617$
& $0.590$ & $0.0167$ \\
Klein (Uniform)
& $0.644$ & $0.0624$
& $0.618$ & $0.0222$
& $3.90$ & $1.99$
& $0.643$ & $0.0678$
& $0.684$ & $0.0529$
& $0.841$ & $0.0198$
& $0.689$ & $0.0346$
& $\mathbf{0.608}$ & $0.0166$ \\
Klein--Torus
& $0.555$ & $0.0695$
& $0.548$ & $0.0209$
& $6.08$ & $0.926$
& $0.548$ & $0.0636$
& $0.533$ & $0.0190$
& $0.931$ & $0.0185$
& $\mathbf{0.519}$ & $0.0224$
& $0.550$ & $0.0511$ \\
1-ball (GMM)
& $0.917$ & $0.115$
& $0.938$ & $0.0890$
& $7.09$ & $3.99$
& $0.938$ & $0.0830$
& $0.973$ & $0.159$
& $0.951$ & $0.0672$
& $1.01$ & $0.191$
& $\mathbf{0.907}$ & $0.0693$ \\
1-ball (Uniform)
& $0.776$ & $0.128$
& $0.893$ & $0.125$
& $9.66$ & $6.55$
& $0.808$ & $0.147$
& $\mathbf{0.711}$ & $0.0574$
& $0.882$ & $0.0561$
& $0.749$ & $0.108$
& $0.871$ & $0.0859$ \\
2-balls (GMM)
& $1.53$ & $0.303$
& $1.55$ & $0.0771$
& $130$ & $119$
& $1.54$ & $0.0864$
& $1.53$ & $0.124$
& $1.72$ & $0.141$
& $\mathbf{1.46}$ & $0.0530$
& $1.56$ & $0.0822$ \\
2-balls (Uniform)
& $1.65$ & $0.398$
& $1.89$ & $0.360$
& $69.5$ & $60.5$
& $1.59$ & $0.245$
& $1.51$ & $0.285$
& $2.06$ & $0.487$
& $\mathbf{1.34}$ & $0.121$
& $1.85$ & $0.312$ \\
\bottomrule
\end{tabular}%
}
\end{table*}
}

\newcommand{\TableBaselineMeanHorizonFull}{%
\begin{table*}
\centering
\small
\setlength{\tabcolsep}{6pt}
\renewcommand{\arraystretch}{1.1}
\caption{
Baseline comparison of conditional (unconditional) path distribution loss across 11 benchmark problems.
Per-case means and sample standard deviations are reported in \Cref{tab:baseline-conditional-energy-full,tab:baseline-trajectory-energy-full}. $\star$ denotes using $\mathcal{L}_x$ and $\mathcal{L}_\mathrm{x}^{\mathrm{traj}}$.
Horizon: $5\,\mathrm{s}$. }
\label{tab:baseline-mean-full}
\resizebox{\textwidth}{!}{%
\begin{tabular}{l*{6}{c}}
\toprule
 & Ours & \citep{teng2026embedding} & \citep{li2020scalable} & \citep{li2020scalable}$\star$ & \citep{bartosh2025sde} & \citep{kiyohara2025neural} \\
\midrule
b-ball (GMM)
& $\mathbf{0.432}\;(\mathbf{0.823})$ & Diverged & Diverged & $3.69\;(154)$ & Diverged & $0.730\;(9.53)$ \\
b-ball (Uniform)
& $\mathbf{0.338}\;(\mathbf{0.709})$ & Diverged & Diverged & Diverged & Diverged & $0.534\;(5.40)$ \\
Torus (GMM)
& $\mathbf{0.687}\;(\mathbf{0.795})$ & Diverged & Diverged & Diverged & Diverged & Diverged \\
Torus (Uniform)
& $\mathbf{0.677}\;(\mathbf{0.794})$ & Diverged & Diverged & Diverged & Diverged & $0.955\;(16.7)$ \\
Klein (GMM)
& $\mathbf{0.688}\;(\mathbf{0.833})$ & Diverged & Diverged & Diverged & Diverged & $0.903\;(13.4)$ \\
Klein (Uniform)
& $\mathbf{0.676}\;(\mathbf{0.826})$ & Diverged & Diverged & Diverged & Diverged & $0.922\;(14.7)$ \\
Klein--Torus
& $\mathbf{0.581}\;(\mathbf{0.763})$ & Diverged & Diverged & Diverged & Diverged & $0.995\;(19.4)$ \\
1-ball (GMM)
& $\mathbf{0.482}\;(\mathbf{1.05})$ & Diverged & Diverged & Diverged & Diverged & Diverged \\
1-ball (Uniform)
& $\mathbf{0.439}\;(\mathbf{0.956})$ & Diverged & Diverged & Diverged & Diverged & Diverged \\
2-balls (GMM)
& $\mathbf{0.525}\;(\mathbf{1.89})$ & $0.839\;(4.02)$ & $1.64\;(93.4)$ & $1.84\;(132)$ & Diverged & Diverged \\
2-balls (Uniform)
& $\mathbf{0.509}\;(\mathbf{2.14})$ & $0.809\;(4.07)$ & Diverged & $3.01\;(277)$ & Diverged & Diverged \\
\bottomrule
\end{tabular}%
}
\vspace{-3mm}
\end{table*}
}

\newcommand{\TableAblationSummaryHorizonFull}{%
\begin{table}
\centering
\footnotesize
\setlength{\tabcolsep}{4pt}
\renewcommand{\arraystretch}{1.15}
\caption{
Ablation summary across 11 benchmark settings on the conditional (unconditional) trajectory distribution loss.
Mean change is the average of per-setting percentage changes relative to the default model;
positive values indicate degradation.
Worse cases count the settings with a higher mean test loss than the default model
(out of 11), using the means over five training seeds.
MLP replaces the diffusion sampler with a deterministic MLP encoder;
loss labels identify the removed objectives.
Per-case latent-dimension and loss-ablation results are reported in \Cref{tab:ablation-split-latent-conditional-energy-loss-full,tab:ablation-split-latent-x-trajectory-loss-full,tab:ablation-split-weights-conditional-energy-loss-full,tab:ablation-split-weights-x-trajectory-loss-full}.
Per-case conditional and unconditional results for the MLP encoder ablation are reported in \Cref{tab:ablation-split-sampler-conditional-energy-loss-full}.
Horizon: $5\,\mathrm{s}$.
}
\label{tab:ablation-summary-full}
\resizebox{\textwidth}{!}{%
\begin{tabular}{l*{11}{c}}
\toprule
Setting
& $d_z=d_x$
& $d_z=2d_x$
& $d_z=3d_x$
& MLP
& $\mathcal L_{\mathrm{x}}^{\mathrm{traj,c}}$
& $\mathcal L_z$
& $\mathcal L_x$
& $\mathcal L_{\mathrm{x}}^{\mathrm{traj}}$
& $\mathcal L_{\mathrm{x}}^{\mathrm{traj,c}},\ \mathcal L_{\mathrm{x}}^{\mathrm{traj}}$
& $\mathcal L_x,\ \mathcal L_{\mathrm{x}}^{\mathrm{traj}}$
& $\mathcal L_{\mathrm{x}}^{\mathrm{traj,c}},\ \mathcal L_x$ \\
\midrule
Mean change (\%)
& $\text{Diverged}\;(\text{Diverged})$
& $\text{Diverged}\;(\text{Diverged})$
& $+20.9\;(+396.2)$
& $+3.1\;(+14.2)$
& $+3.7\;(+1.6)$
& $\text{Diverged}\;(\text{Diverged})$
& $+0.8\;(+11.6)$
& $-0.4\;(+4.0)$
& $+9.9\;(+17.4)$
& $-1.0\;(+4.8)$
& $+3.2\;(-1.1)$ \\
Worse cases (/11)
& $\text{Diverged}\;(\text{Diverged})$
& $\text{Diverged}\;(\text{Diverged})$
& $7\;(5)$
& $7\;(10)$
& $8\;(4)$
& $\text{Diverged}\;(\text{Diverged})$
& $7\;(7)$
& $6\;(6)$
& $11\;(11)$
& $6\;(8)$
& $8\;(4)$ \\
\bottomrule
\end{tabular}%
}
\vspace{-2mm}
\end{table}
}

\newcommand{\TableBaselineConditionalEnergyHorizonFull}{%
\begin{table*}[!b]
\centering
\small
\setlength{\tabcolsep}{6pt}
\renewcommand{\arraystretch}{1.1}
\caption{
Baseline comparison measured by test-set conditional path distribution loss $\mathcal L_{\mathrm{x}}^{\mathrm{traj,c}}$ ($\downarrow$).
Horizon: $5\,\mathrm{s}$.
}
\label{tab:baseline-conditional-energy-full}
\resizebox{\textwidth}{!}{%
\begin{tabular}{l*{6}{r@{$\,\pm\,$}l}}
\toprule
Case & \multicolumn{2}{c}{Ours}
& \multicolumn{2}{c}{\citep{teng2026embedding}}
& \multicolumn{2}{c}{\citep{li2020scalable}}
& \multicolumn{2}{c}{\citep{li2020scalable}$\star$}
& \multicolumn{2}{c}{\citep{bartosh2025sde}}
& \multicolumn{2}{c}{\citep{kiyohara2025neural}} \\
\midrule
b-ball (GMM)
& $\mathbf{0.432}$ & $0.00909$
& \multicolumn{2}{c}{Diverged}
& \multicolumn{2}{c}{Diverged}
& $3.69$ & $3.66$
& \multicolumn{2}{c}{Diverged}
& $0.730$ & $0.00993$ \\
b-ball (Uniform)
& $\mathbf{0.338}$ & $0.0119$
& \multicolumn{2}{c}{Diverged}
& \multicolumn{2}{c}{Diverged}
& \multicolumn{2}{c}{Diverged}
& \multicolumn{2}{c}{Diverged}
& $0.534$ & $0.00627$ \\
Torus (GMM)
& $\mathbf{0.687}$ & $0.00109$
& \multicolumn{2}{c}{Diverged}
& \multicolumn{2}{c}{Diverged}
& \multicolumn{2}{c}{Diverged}
& \multicolumn{2}{c}{Diverged}
& \multicolumn{2}{c}{Diverged} \\
Torus (Uniform)
& $\mathbf{0.677}$ & $0.00517$
& \multicolumn{2}{c}{Diverged}
& \multicolumn{2}{c}{Diverged}
& \multicolumn{2}{c}{Diverged}
& \multicolumn{2}{c}{Diverged}
& $0.955$ & $0.0207$ \\
Klein (GMM)
& $\mathbf{0.688}$ & $0.00339$
& \multicolumn{2}{c}{Diverged}
& \multicolumn{2}{c}{Diverged}
& \multicolumn{2}{c}{Diverged}
& \multicolumn{2}{c}{Diverged}
& $0.903$ & $0.0164$ \\
Klein (Uniform)
& $\mathbf{0.676}$ & $0.00832$
& \multicolumn{2}{c}{Diverged}
& \multicolumn{2}{c}{Diverged}
& \multicolumn{2}{c}{Diverged}
& \multicolumn{2}{c}{Diverged}
& $0.922$ & $0.0205$ \\
Klein--Torus
& $\mathbf{0.581}$ & $0.0546$
& \multicolumn{2}{c}{Diverged}
& \multicolumn{2}{c}{Diverged}
& \multicolumn{2}{c}{Diverged}
& \multicolumn{2}{c}{Diverged}
& $0.995$ & $0.0444$ \\
1-ball (GMM)
& $\mathbf{0.482}$ & $0.00578$
& \multicolumn{2}{c}{Diverged}
& \multicolumn{2}{c}{Diverged}
& \multicolumn{2}{c}{Diverged}
& \multicolumn{2}{c}{Diverged}
& \multicolumn{2}{c}{Diverged} \\
1-ball (Uniform)
& $\mathbf{0.439}$ & $0.0209$
& \multicolumn{2}{c}{Diverged}
& \multicolumn{2}{c}{Diverged}
& \multicolumn{2}{c}{Diverged}
& \multicolumn{2}{c}{Diverged}
& \multicolumn{2}{c}{Diverged} \\
2-balls (GMM)
& $\mathbf{0.525}$ & $0.00689$
& $0.839$ & $0.00677$
& $1.64$ & $0.185$
& $1.84$ & $0.672$
& \multicolumn{2}{c}{Diverged}
& \multicolumn{2}{c}{Diverged} \\
2-balls (Uniform)
& $\mathbf{0.509}$ & $0.0191$
& $0.809$ & $0.0233$
& \multicolumn{2}{c}{Diverged}
& $3.01$ & $0.866$
& \multicolumn{2}{c}{Diverged}
& \multicolumn{2}{c}{Diverged} \\
\bottomrule
\end{tabular}%
}
\end{table*}
}

\newcommand{\TableBaselineTrajectoryEnergyHorizonFull}{%
\begin{table*}
\centering
\scriptsize
\setlength{\tabcolsep}{6pt}
\renewcommand{\arraystretch}{1.1}
\caption{
Baseline comparison measured by test-set unconditional path distribution loss $\mathcal L_{\mathrm{x}}^{\mathrm{traj}}$ ($\downarrow$).
Horizon: $5\,\mathrm{s}$.
}
\label{tab:baseline-trajectory-energy-full}
\resizebox{\textwidth}{!}{%
\begin{tabular}{l*{6}{r@{$\,\pm\,$}l}}
\toprule
Case & \multicolumn{2}{c}{Ours}
& \multicolumn{2}{c}{\citep{teng2026embedding}}
& \multicolumn{2}{c}{\citep{li2020scalable}}
& \multicolumn{2}{c}{\citep{li2020scalable}$\star$}
& \multicolumn{2}{c}{\citep{bartosh2025sde}}
& \multicolumn{2}{c}{\citep{kiyohara2025neural}} \\
\midrule
b-ball (GMM)
& $\mathbf{0.823}$ & $0.144$
& \multicolumn{2}{c}{Diverged}
& \multicolumn{2}{c}{Diverged}
& $154$ & $187$
& \multicolumn{2}{c}{Diverged}
& $9.53$ & $0.341$ \\
b-ball (Uniform)
& $\mathbf{0.709}$ & $0.159$
& \multicolumn{2}{c}{Diverged}
& \multicolumn{2}{c}{Diverged}
& \multicolumn{2}{c}{Diverged}
& \multicolumn{2}{c}{Diverged}
& $5.40$ & $0.107$ \\
Torus (GMM)
& $\mathbf{0.795}$ & $0.0315$
& \multicolumn{2}{c}{Diverged}
& \multicolumn{2}{c}{Diverged}
& \multicolumn{2}{c}{Diverged}
& \multicolumn{2}{c}{Diverged}
& \multicolumn{2}{c}{Diverged} \\
Torus (Uniform)
& $\mathbf{0.794}$ & $0.0622$
& \multicolumn{2}{c}{Diverged}
& \multicolumn{2}{c}{Diverged}
& \multicolumn{2}{c}{Diverged}
& \multicolumn{2}{c}{Diverged}
& $16.7$ & $1.32$ \\
Klein (GMM)
& $\mathbf{0.833}$ & $0.0324$
& \multicolumn{2}{c}{Diverged}
& \multicolumn{2}{c}{Diverged}
& \multicolumn{2}{c}{Diverged}
& \multicolumn{2}{c}{Diverged}
& $13.4$ & $1.04$ \\
Klein (Uniform)
& $\mathbf{0.826}$ & $0.0696$
& \multicolumn{2}{c}{Diverged}
& \multicolumn{2}{c}{Diverged}
& \multicolumn{2}{c}{Diverged}
& \multicolumn{2}{c}{Diverged}
& $14.7$ & $1.27$ \\
Klein--Torus
& $\mathbf{0.763}$ & $0.118$
& \multicolumn{2}{c}{Diverged}
& \multicolumn{2}{c}{Diverged}
& \multicolumn{2}{c}{Diverged}
& \multicolumn{2}{c}{Diverged}
& $19.4$ & $2.85$ \\
1-ball (GMM)
& $\mathbf{1.05}$ & $0.141$
& \multicolumn{2}{c}{Diverged}
& \multicolumn{2}{c}{Diverged}
& \multicolumn{2}{c}{Diverged}
& \multicolumn{2}{c}{Diverged}
& \multicolumn{2}{c}{Diverged} \\
1-ball (Uniform)
& $\mathbf{0.956}$ & $0.185$
& \multicolumn{2}{c}{Diverged}
& \multicolumn{2}{c}{Diverged}
& \multicolumn{2}{c}{Diverged}
& \multicolumn{2}{c}{Diverged}
& \multicolumn{2}{c}{Diverged} \\
2-balls (GMM)
& $\mathbf{1.89}$ & $0.374$
& $4.02$ & $0.981$
& $93.4$ & $23.6$
& $132$ & $83.5$
& \multicolumn{2}{c}{Diverged}
& \multicolumn{2}{c}{Diverged} \\
2-balls (Uniform)
& $\mathbf{2.14}$ & $0.693$
& $4.07$ & $1.51$
& \multicolumn{2}{c}{Diverged}
& $277$ & $96.2$
& \multicolumn{2}{c}{Diverged}
& \multicolumn{2}{c}{Diverged} \\
\bottomrule
\end{tabular}%
}
\end{table*}
}

\newcommand{\TableLatentConditionalEnergyHorizonFull}{%
\begin{table*}
\centering
\scriptsize
\setlength{\tabcolsep}{4pt}
\renewcommand{\arraystretch}{1.1}
\caption{
Latent-dimension ablation measured by test-set conditional path distribution loss $\mathcal L_{\mathrm{x}}^{\mathrm{traj,c}}$ ($\downarrow$).
Horizon: $5\,\mathrm{s}$.
}
\label{tab:ablation-split-latent-conditional-energy-loss-full}
\begin{tabular}{l*{4}{r@{$\,\pm\,$}l}}
\toprule
 & \multicolumn{2}{c}{$d_z=d_x$} & \multicolumn{2}{c}{$d_z=2d_x$} & \multicolumn{2}{c}{$d_z=3d_x$} & \multicolumn{2}{c}{$d_z=4d_x$ (default)} \\
\midrule
b-ball (GMM)
& $0.469$ & $0.0170$
& $0.444$ & $0.0239$
& $0.465$ & $0.0384$
& $\mathbf{0.432}$ & $0.00909$ \\
b-ball (Uniform)
& \multicolumn{2}{c}{Diverged}
& $0.381$ & $0.0576$
& $0.343$ & $0.0192$
& $\mathbf{0.338}$ & $0.0119$ \\
Torus (GMM)
& $0.771$ & $0.153$
& $0.689$ & $0.00293$
& $0.689$ & $0.00180$
& $\mathbf{0.687}$ & $0.00109$ \\
Torus (Uniform)
& $0.705$ & $0.00415$
& $0.674$ & $0.00480$
& $\mathbf{0.671}$ & $0.00285$
& $0.677$ & $0.00517$ \\
Klein (GMM)
& $0.707$ & $0.00579$
& $0.693$ & $0.00565$
& $0.688$ & $0.000717$
& $\mathbf{0.688}$ & $0.00339$ \\
Klein (Uniform)
& $0.694$ & $0.00823$
& $\mathbf{0.675}$ & $0.00388$
& $0.677$ & $0.00225$
& $0.676$ & $0.00832$ \\
Klein--Torus
& $0.705$ & $0.00816$
& $0.628$ & $0.0743$
& $\mathbf{0.574}$ & $0.0401$
& $0.581$ & $0.0546$ \\
1-ball (GMM)
& $0.501$ & $0.0115$
& $\mathbf{0.479}$ & $0.00566$
& $0.482$ & $0.00328$
& $0.482$ & $0.00578$ \\
1-ball (Uniform)
& $0.441$ & $0.0105$
& $\mathbf{0.423}$ & $0.00767$
& $0.427$ & $0.00571$
& $0.439$ & $0.0209$ \\
2-balls (GMM)
& $0.532$ & $0.00539$
& \multicolumn{2}{c}{Diverged}
& $1.71$ & $1.71$
& $\mathbf{0.525}$ & $0.00689$ \\
2-balls (Uniform)
& $0.515$ & $0.0113$
& $0.516$ & $0.0363$
& $\mathbf{0.505}$ & $0.0156$
& $0.509$ & $0.0191$ \\
\bottomrule
\end{tabular}%
\end{table*}
}

\newcommand{\TableLatentTrajectoryEnergyHorizonFull}{%
\begin{table*}
\centering
\scriptsize
\setlength{\tabcolsep}{4pt}
\renewcommand{\arraystretch}{1.1}
\caption{
Latent-dimension ablation measured by test-set unconditional path distribution loss $\mathcal L_{\mathrm{x}}^{\mathrm{traj}}$ ($\downarrow$).
Horizon: $5\,\mathrm{s}$.
}
\label{tab:ablation-split-latent-x-trajectory-loss-full}
\begin{tabular}{l*{4}{r@{$\,\pm\,$}l}}
\toprule
 & \multicolumn{2}{c}{$d_z=d_x$} & \multicolumn{2}{c}{$d_z=2d_x$} & \multicolumn{2}{c}{$d_z=3d_x$} & \multicolumn{2}{c}{$d_z=4d_x$ (default)} \\
\midrule
b-ball (GMM)
& $1.63$ & $0.465$
& $1.16$ & $0.401$
& $1.53$ & $0.797$
& $\mathbf{0.823}$ & $0.144$ \\
b-ball (Uniform)
& \multicolumn{2}{c}{Diverged}
& $1.17$ & $0.816$
& $\mathbf{0.654}$ & $0.254$
& $0.709$ & $0.159$ \\
Torus (GMM)
& $6.81$ & $13.2$
& $0.801$ & $0.0320$
& $0.825$ & $0.0328$
& $\mathbf{0.795}$ & $0.0315$ \\
Torus (Uniform)
& $0.962$ & $0.0922$
& $0.769$ & $0.0617$
& $\mathbf{0.761}$ & $0.0301$
& $0.794$ & $0.0622$ \\
Klein (GMM)
& $1.13$ & $0.199$
& $0.890$ & $0.0836$
& $\mathbf{0.830}$ & $0.0552$
& $0.833$ & $0.0324$ \\
Klein (Uniform)
& $0.959$ & $0.0524$
& $\mathbf{0.792}$ & $0.0185$
& $0.807$ & $0.0245$
& $0.826$ & $0.0696$ \\
Klein--Torus
& $1.24$ & $0.199$
& $0.840$ & $0.120$
& $\mathbf{0.725}$ & $0.0594$
& $0.763$ & $0.118$ \\
1-ball (GMM)
& $1.27$ & $0.124$
& $\mathbf{0.977}$ & $0.0815$
& $1.10$ & $0.0787$
& $1.05$ & $0.141$ \\
1-ball (Uniform)
& $1.01$ & $0.135$
& $\mathbf{0.853}$ & $0.0873$
& $0.903$ & $0.0532$
& $0.956$ & $0.185$ \\
2-balls (GMM)
& $2.08$ & $0.257$
& \multicolumn{2}{c}{Diverged}
& $82.6$ & $112$
& $\mathbf{1.89}$ & $0.374$ \\
2-balls (Uniform)
& $\mathbf{1.99}$ & $0.232$
& $4.68$ & $5.53$
& $2.39$ & $0.658$
& $2.14$ & $0.693$ \\
\bottomrule
\end{tabular}%
\end{table*}
}

\newcommand{\TableSamplerAblationHorizonFull}{%
\begin{table*}
\centering
\scriptsize
\setlength{\tabcolsep}{4pt}
\renewcommand{\arraystretch}{1.1}
\caption{
Sampler ablation measured by test-set conditional and unconditional path distribution losses ($\downarrow$).
Horizon: $5\,\mathrm{s}$.
}
\label{tab:ablation-split-sampler-conditional-energy-loss-full}
\label{tab:ablation-split-sampler-x-trajectory-loss-full}
\begin{tabular}{l*{4}{r@{$\,\pm\,$}l}}
\toprule
 & \multicolumn{4}{c}{DDIM} & \multicolumn{4}{c}{MLP} \\
\cmidrule(lr){2-5}\cmidrule(lr){6-9}
Case & \multicolumn{2}{c}{Conditional} & \multicolumn{2}{c}{Unconditional}
& \multicolumn{2}{c}{Conditional} & \multicolumn{2}{c}{Unconditional} \\
\midrule
b-ball (GMM)
& $\mathbf{0.432}$ & $0.00909$
& $\mathbf{0.823}$ & $0.144$
& $0.469$ & $0.0210$
& $1.48$ & $0.567$ \\
b-ball (Uniform)
& $\mathbf{0.338}$ & $0.0119$
& $\mathbf{0.709}$ & $0.159$
& $0.370$ & $0.0139$
& $0.978$ & $0.222$ \\
Torus (GMM)
& $\mathbf{0.687}$ & $0.00109$
& $\mathbf{0.795}$ & $0.0315$
& $0.689$ & $0.00147$
& $0.815$ & $0.0344$ \\
Torus (Uniform)
& $0.677$ & $0.00517$
& $\mathbf{0.794}$ & $0.0622$
& $\mathbf{0.673}$ & $0.00316$
& $0.815$ & $0.0550$ \\
Klein (GMM)
& $\mathbf{0.688}$ & $0.00339$
& $\mathbf{0.833}$ & $0.0324$
& $0.690$ & $0.00186$
& $0.840$ & $0.0343$ \\
Klein (Uniform)
& $0.676$ & $0.00832$
& $\mathbf{0.826}$ & $0.0696$
& $\mathbf{0.675}$ & $0.00235$
& $0.861$ & $0.0246$ \\
Klein--Torus
& $\mathbf{0.581}$ & $0.0546$
& $\mathbf{0.763}$ & $0.118$
& $0.642$ & $0.0607$
& $0.886$ & $0.146$ \\
1-ball (GMM)
& $\mathbf{0.482}$ & $0.00578$
& $\mathbf{1.05}$ & $0.141$
& $0.506$ & $0.0111$
& $1.18$ & $0.0986$ \\
1-ball (Uniform)
& $\mathbf{0.439}$ & $0.0209$
& $\mathbf{0.956}$ & $0.185$
& $0.453$ & $0.00949$
& $1.01$ & $0.0857$ \\
2-balls (GMM)
& $0.525$ & $0.00689$
& $\mathbf{1.89}$ & $0.374$
& $\mathbf{0.524}$ & $0.00922$
& $2.00$ & $0.432$ \\
2-balls (Uniform)
& $0.509$ & $0.0191$
& $2.14$ & $0.693$
& $\mathbf{0.500}$ & $0.00990$
& $\mathbf{1.90}$ & $0.260$ \\
\bottomrule
\end{tabular}%
\end{table*}
}

\newcommand{\TableLossAblationConditionalEnergyHorizonFull}{%
\begin{table*}
\centering
\scriptsize
\setlength{\tabcolsep}{4pt}
\renewcommand{\arraystretch}{1.1}
\caption{
Loss-function ablation measured by test-set conditional path distribution loss $\mathcal L_{\mathrm{x}}^{\mathrm{traj,c}}$ ($\downarrow$).
Column labels other than Default model identify the removed objectives (their weights are set to zero).
Values are the mean $\pm$ sample standard deviation over five training seeds.
The lowest mean among finite entries in each row is bold.
Diverged denotes a configuration with nonfinite results or either test metric exceeding $10^3$ in any seed.
Horizon: $5\,\mathrm{s}$.
}
\label{tab:ablation-split-weights-conditional-energy-loss-full}
\resizebox{\textwidth}{!}{%
\begin{tabular}{l*{8}{r@{$\,\pm\,$}l}}
\toprule
Case & \multicolumn{2}{c}{Default model} & \multicolumn{2}{c}{$\mathcal L_{\mathrm{x}}^{\mathrm{traj,c}}$} & \multicolumn{2}{c}{$\mathcal L_z$} & \multicolumn{2}{c}{$\mathcal L_x$} & \multicolumn{2}{c}{$\mathcal L_{\mathrm{x}}^{\mathrm{traj}}$} & \multicolumn{2}{c}{$\mathcal L_{\mathrm{x}}^{\mathrm{traj,c}},\ \mathcal L_{\mathrm{x}}^{\mathrm{traj}}$} & \multicolumn{2}{c}{$\mathcal L_x,\ \mathcal L_{\mathrm{x}}^{\mathrm{traj}}$} & \multicolumn{2}{c}{$\mathcal L_{\mathrm{x}}^{\mathrm{traj,c}},\ \mathcal L_x$} \\
\midrule
b-ball (GMM)
& $\mathbf{0.432}$ & $0.00909$
& $0.446$ & $0.00769$
& $0.512$ & $0.0405$
& $0.456$ & $0.0193$
& $0.442$ & $0.0118$
& $0.454$ & $0.0190$
& $0.444$ & $0.0144$
& $0.445$ & $0.00619$ \\
b-ball (Uniform)
& $\mathbf{0.338}$ & $0.0119$
& $0.348$ & $0.0204$
& $0.403$ & $0.0524$
& $0.342$ & $0.0167$
& $0.342$ & $0.0113$
& $0.356$ & $0.00949$
& $0.342$ & $0.00911$
& $0.339$ & $0.00832$ \\
Torus (GMM)
& $\mathbf{0.687}$ & $0.00109$
& $0.688$ & $0.00111$
& $0.755$ & $0.0550$
& $0.689$ & $0.00123$
& $0.691$ & $0.000986$
& $0.709$ & $0.000482$
& $0.690$ & $0.00178$
& $0.688$ & $0.000696$ \\
Torus (Uniform)
& $0.677$ & $0.00517$
& $0.673$ & $0.00130$
& $0.748$ & $0.0434$
& $\mathbf{0.670}$ & $0.00201$
& $0.671$ & $0.00221$
& $0.708$ & $0.000951$
& $0.675$ & $0.00248$
& $0.670$ & $0.00107$ \\
Klein (GMM)
& $0.688$ & $0.00339$
& $0.687$ & $0.00281$
& $0.897$ & $0.208$
& $0.689$ & $0.00249$
& $0.690$ & $0.00326$
& $0.707$ & $0.000818$
& $0.690$ & $0.000768$
& $\mathbf{0.685}$ & $0.000721$ \\
Klein (Uniform)
& $0.676$ & $0.00832$
& $0.674$ & $0.00120$
& $0.795$ & $0.0644$
& $0.675$ & $0.00772$
& $0.677$ & $0.00610$
& $0.707$ & $0.00122$
& $0.678$ & $0.00651$
& $\mathbf{0.671}$ & $0.000704$ \\
Klein--Torus
& $0.581$ & $0.0546$
& $0.592$ & $0.0390$
& $0.868$ & $0.0132$
& $0.575$ & $0.0532$
& $0.581$ & $0.0275$
& $0.708$ & $0.000533$
& $\mathbf{0.549}$ & $0.00964$
& $0.590$ & $0.0515$ \\
1-ball (GMM)
& $\mathbf{0.482}$ & $0.00578$
& $0.513$ & $0.0108$
& $0.987$ & $0.674$
& $0.484$ & $0.00540$
& $0.487$ & $0.00730$
& $0.546$ & $0.0138$
& $0.494$ & $0.0116$
& $0.513$ & $0.0112$ \\
1-ball (Uniform)
& $0.439$ & $0.0209$
& $0.474$ & $0.0137$
& $1.08$ & $0.724$
& $0.446$ & $0.0156$
& $0.429$ & $0.0113$
& $0.512$ & $0.0126$
& $\mathbf{0.425}$ & $0.00660$
& $0.473$ & $0.0165$ \\
2-balls (GMM)
& $0.525$ & $0.00689$
& $0.568$ & $0.00330$
& \multicolumn{2}{c}{Diverged}
& $0.536$ & $0.0243$
& $0.516$ & $0.00392$
& $0.595$ & $0.00503$
& $\mathbf{0.511}$ & $0.00610$
& $0.567$ & $0.00337$ \\
2-balls (Uniform)
& $0.509$ & $0.0191$
& $0.568$ & $0.00777$
& $4.30$ & $3.28$
& $0.507$ & $0.0202$
& $0.486$ & $0.0160$
& $0.601$ & $0.00637$
& $\mathbf{0.474}$ & $0.00419$
& $0.563$ & $0.00686$ \\
\bottomrule
\end{tabular}%
}
\end{table*}
}

\newcommand{\TableLossAblationUnconditionalEnergyHorizonFull}{%
\begin{table*}
\centering
\scriptsize
\setlength{\tabcolsep}{4pt}
\renewcommand{\arraystretch}{1.1}
\caption{
Loss-function ablation measured by test-set unconditional path distribution loss $\mathcal L_{\mathrm{x}}^{\mathrm{traj}}$ ($\downarrow$).
Column labels other than Default model identify the removed objectives (their weights are set to zero).
Values are the mean $\pm$ sample standard deviation over five training seeds.
The lowest mean among finite entries in each row is bold.
Diverged denotes a configuration with nonfinite results or either test metric exceeding $10^3$ in any seed.
Horizon: $5\,\mathrm{s}$.
}
\label{tab:ablation-split-weights-x-trajectory-loss-full}
\resizebox{\textwidth}{!}{%
\begin{tabular}{l*{8}{r@{$\,\pm\,$}l}}
\toprule
Case & \multicolumn{2}{c}{Default model} & \multicolumn{2}{c}{$\mathcal L_{\mathrm{x}}^{\mathrm{traj,c}}$} & \multicolumn{2}{c}{$\mathcal L_z$} & \multicolumn{2}{c}{$\mathcal L_x$} & \multicolumn{2}{c}{$\mathcal L_{\mathrm{x}}^{\mathrm{traj}}$} & \multicolumn{2}{c}{$\mathcal L_{\mathrm{x}}^{\mathrm{traj,c}},\ \mathcal L_{\mathrm{x}}^{\mathrm{traj}}$} & \multicolumn{2}{c}{$\mathcal L_x,\ \mathcal L_{\mathrm{x}}^{\mathrm{traj}}$} & \multicolumn{2}{c}{$\mathcal L_{\mathrm{x}}^{\mathrm{traj,c}},\ \mathcal L_x$} \\
\midrule
b-ball (GMM)
& $\mathbf{0.823}$ & $0.144$
& $0.990$ & $0.219$
& $2.59$ & $1.06$
& $1.29$ & $0.188$
& $1.09$ & $0.295$
& $1.19$ & $0.538$
& $1.18$ & $0.347$
& $0.931$ & $0.191$ \\
b-ball (Uniform)
& $0.709$ & $0.159$
& $0.653$ & $0.207$
& $2.05$ & $1.90$
& $0.699$ & $0.177$
& $0.738$ & $0.107$
& $0.711$ & $0.117$
& $0.793$ & $0.232$
& $\mathbf{0.593}$ & $0.155$ \\
Torus (GMM)
& $0.795$ & $0.0315$
& $\mathbf{0.770}$ & $0.0216$
& $4.45$ & $2.87$
& $0.835$ & $0.0448$
& $0.883$ & $0.0634$
& $0.899$ & $0.0177$
& $0.869$ & $0.0769$
& $0.782$ & $0.0221$ \\
Torus (Uniform)
& $0.794$ & $0.0622$
& $0.732$ & $0.00632$
& $4.08$ & $2.07$
& $0.789$ & $0.0612$
& $0.778$ & $0.0321$
& $0.924$ & $0.0233$
& $0.826$ & $0.0568$
& $\mathbf{0.726}$ & $0.00525$ \\
Klein (GMM)
& $0.833$ & $0.0324$
& $\mathbf{0.771}$ & $0.0142$
& $12.7$ & $13.8$
& $0.828$ & $0.0536$
& $0.844$ & $0.0674$
& $0.900$ & $0.0208$
& $0.898$ & $0.0789$
& $0.780$ & $0.0140$ \\
Klein (Uniform)
& $0.826$ & $0.0696$
& $0.789$ & $0.0250$
& $5.78$ & $3.49$
& $0.829$ & $0.0866$
& $0.877$ & $0.0660$
& $0.991$ & $0.0186$
& $0.872$ & $0.0374$
& $\mathbf{0.786}$ & $0.0125$ \\
Klein--Torus
& $0.763$ & $0.118$
& $0.741$ & $0.0599$
& $9.66$ & $0.945$
& $0.747$ & $0.111$
& $0.738$ & $0.0386$
& $1.15$ & $0.0293$
& $\mathbf{0.696}$ & $0.0480$
& $0.756$ & $0.0933$ \\
1-ball (GMM)
& $\mathbf{1.05}$ & $0.141$
& $1.11$ & $0.122$
& $27.9$ & $39.2$
& $1.09$ & $0.0576$
& $1.15$ & $0.169$
& $1.09$ & $0.108$
& $1.15$ & $0.168$
& $1.09$ & $0.135$ \\
1-ball (Uniform)
& $0.956$ & $0.185$
& $1.05$ & $0.147$
& $35.4$ & $43.4$
& $0.976$ & $0.110$
& $\mathbf{0.900}$ & $0.0876$
& $1.01$ & $0.0667$
& $0.957$ & $0.109$
& $0.986$ & $0.0999$ \\
2-balls (GMM)
& $1.89$ & $0.374$
& $1.79$ & $0.0611$
& \multicolumn{2}{c}{Diverged}
& $2.79$ & $1.58$
& $1.83$ & $0.129$
& $1.96$ & $0.204$
& $\mathbf{1.74}$ & $0.157$
& $1.79$ & $0.0994$ \\
2-balls (Uniform)
& $2.14$ & $0.693$
& $2.57$ & $0.852$
& $415$ & $370$
& $2.50$ & $1.17$
& $2.04$ & $0.693$
& $2.65$ & $0.861$
& $\mathbf{1.68}$ & $0.297$
& $2.42$ & $0.773$ \\
\bottomrule
\end{tabular}%
}
\end{table*}
}

\maketitle

\newcommand{\hang}[1]{[{\textbf{\textcolor{red}{Hang: #1}}}]}

\definecolor{myblue}{RGB}{0,90,180} 
\hypersetup{linkcolor=myblue, citecolor=myblue}

\begin{abstract}
A stochastic hybrid system (SHS) is governed by a stochastic differential equation (SDE) describing the continuous dynamics and a Markov reset kernel triggered on the guard surface. Its probability evolution can be described by a hybrid Fokker–Planck (HFP) equation with a partial differential term corresponding to the SDE and an integral term arising from the reset kernel.
This work shows that such an SHS can be approximated by an SDE in a higher-dimensional latent space where the sample paths are continuous.
The key to this result is to encode different branches of the reset kernel using auxiliary variables, transforming the resets into deterministic ones that enable topological gluing. By the embedding theorem, the glued manifold can then be embedded into a higher-dimensional Euclidean space. 
We show that the probability evolution on the embedded image no longer requires explicit reset terms in the HFP equation. Building on this theorem, we design a loss that matches the evolving state distributions, enabling a single latent SDE to recover the probability evolution of the SHS without mode labeling, trajectory segmentation, or event-based simulations.

\end{abstract}

\section{Introduction}
\begin{figure}[b]
    \centering
    \includegraphics[width=1\linewidth]{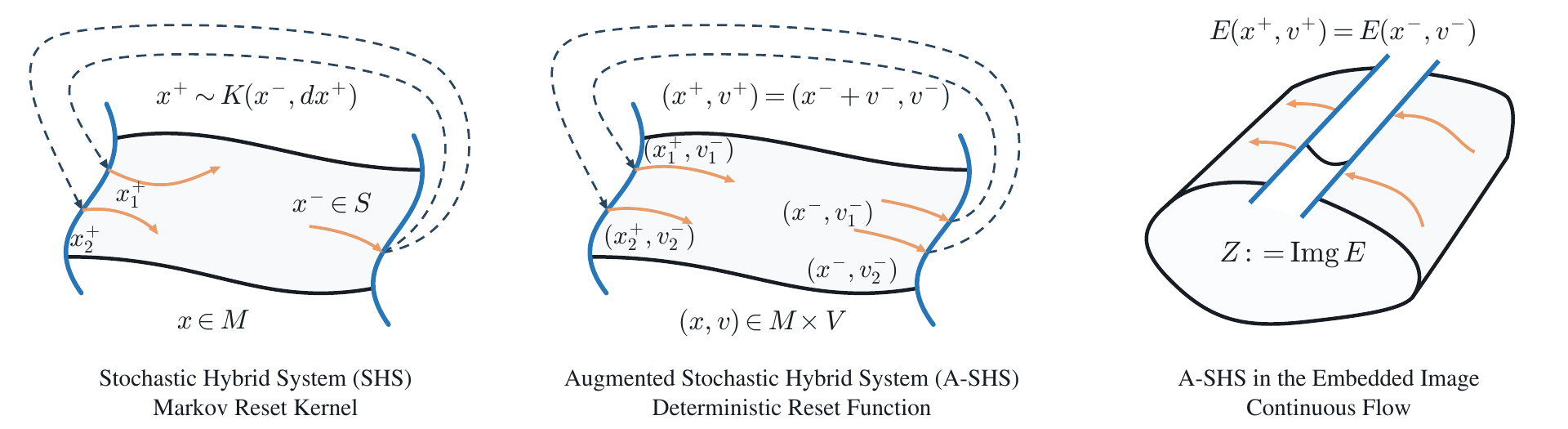}
    \caption{We approximate SHSs with a Markov reset kernel by an A-SHS with a deterministic reset function that allows the topological gluing to make the hybrid flow continuous.  }
    \label{fig:teaser}
\end{figure}

Hybrid systems are a powerful abstraction for dynamical systems such as robotics~\citep{posa2014direct, westervelt2003hybrid} and autonomous driving~\citep{10285579}. To model the stochasticity, we can consider stochastic hybrid systems~(SHSs)~\citep{shs-jianghai, Bujorianu2006}. As SHSs are governed by compositions of continuous-time dynamics and discrete events, learning both components concurrently from time-series data is challenging.

To address this challenge, the recent work~\citep{teng2026embedding, teng2026chyll} applies the geometric hybrid system theory~\citep{simic2005towards, hirsch1976differential} to represent a hybrid system by continuous dynamics on a manifold. In~\citep{teng2026embedding, teng2026chyll}, the manifold representation admits an embedding in a higher-dimensional Euclidean space, which can be learned by a continuous latent ODE~\citep{chen2018neural}. However, the results in~\citep{simic2005towards} assume deterministic resets, which cannot be applied to SHSs where the reset kernel is a probability distribution. In this work, we extend the geometric framework~\citep{teng2026embedding, teng2026chyll} to SHSs by representing different reset branches with auxiliary variables, which enables us to represent SHSs by a continuous SDE on a manifold (see \Cref{fig:teaser}). The main contributions are summarized as:
\begin{enumerate}
    \item We present an augmented-state SHS (A-SHS) with a deterministic reset to approximate an SHS with a Markov reset kernel. The A-SHS is then embedded into a higher-dimensional Euclidean space where the sample paths become continuous. 
    \item We derive the hybrid Fokker-Planck (HFP) equation of the A-SHS and show that the source and sink terms in the HFP equation vanish in the embedded image. This result enables us to predict the probability distribution of A-SHS by a continuous SDE without identifying the guard surface after the embedding. 
    \item Based on the theoretical result, we can accuretely predict the distribution evolution of SHS from time-series data vy a latent SDE without any trajectory segmentation, mode labeling, or event-based simulations.
\end{enumerate}

\section{Related Works}
\textbf{Hybrid System Theory:} Hybrid automata~\citep{lygeros2003dynamical, goebel2009hybrid} model continuous-time dynamics within modes and state-triggered transitions and resets between them. Stochastic hybrid systems~\citep{shs-jianghai} incorporate stochasticity into both continuous dynamics and discrete resets. Specifically, \citep{shs-jianghai} models continuous evolution as a stochastic differential equation (SDE) and resets as Markov kernels triggered on lower-dimensional guard surfaces, while \citep{Bujorianu2006} further allows spontaneous jumps within the state-space interior. Hybrid inclusions instead use set-valued flow and reset maps to represent multiple admissible reset branches~\citep{aubin2002impulse, goebel2006solutions}. Whereas SHSs assign probability laws to reset outcomes, hybrid inclusions specify only the admissible outcomes. To consider the state distribution, the hybrid Fokker--Planck (HFP) equations include additional terms that account for reset-induced sources and sinks~\citep{wang2020spectral, wang2022uncertainty}. Although difficult to solve, they provide a powerful tool for studying SHS distribution evolution.

\textbf{Learning Hybrid Systems:} Learning hybrid systems from time-series data requires jointly identifying modes, guard surfaces, and continuous dynamics, making the problem inherently combinatorial. \citep{roll2004identification} uses mixed-integer programming (MIP) to identify piecewise-affine (PWA) systems and reformulates the problem as a linear complementarity problem to avoid MIP. However, PWA systems are a restricted class of hybrid automata (HA) with identity resets. To learn general SHSs, \citep{poli2021neural} uses normalizing flows for event transitions and Neural ODEs~\citep{chen2018neural} for continuous dynamics. Similarly, \citep{liu2025discrete} learns discrete-time HA using a mode classifier and state-space model. Both methods require heuristic trajectory segmentation, which is sensitive to stochastic noise, while \citep{ochoaneural} learns vector fields and reset functions assuming known segmentations. Alternatively, \citep{teng2026embedding,teng2026chyll} builds on the theory in~\citep{simic2005towards} to represent HA as Neural ODEs in higher-dimensional latent spaces. 

\textbf{Learning Stochastic Differential Equations:} Recent work learns stochastic dynamics at different statistical levels. \citep{kidger2021neural} trains a Neural SDE generator against a Neural CDE critic using a Wasserstein path-level loss. Most closely related, \citep{bartosh2025sde} parameterizes posterior marginals and trains latent SDEs without simulation using an ELBO with a path-space KL loss. \citep{neklyudov2023action} recovers a canonical probability flow from unpaired temporal marginals through a variational action objective, while \citep{kiyohara2025neural} learns conditional normalizing-flow transition kernels using likelihood and Chapman--Kolmogorov consistency. \citep{seifner2025context} pretrains a transformer through supervised drift--diffusion regression, whereas \citep{ilersich2025learning} learns coupled macro--microscale latent SDEs using an ELBO and Product-of-Experts likelihood. Unlike these methods, the proposed latent SDE framework represents stochastic resets without explicit event-based simulation.

\section{Preliminaries \& Problem Formulation}
\subsection{Stochastic Hybrid Systems}
We refer to~\citep {shs-jianghai, wang2020spectral} for a preliminary overview of SHSs. Consider a SHS with state-triggered Markov reset kernel as illustrated in~\Cref{fig:teaser}: 
\begin{equation}
\label{eq:stochastic-hybrid-sys}
\begin{aligned}
        dx &= f(t, x)dt+g(t, x)dw,\ x\notin S, \\
        x^{+} &\sim K(x^{-},dx^{+}),\ x^{-}\in S ,
\end{aligned}\tag{SHS}
\end{equation}
where \(w\) is an $n-\dim$ Wiener process, $x$ evolves in the domain $M \subset \mathbb{R}^n$ with smooth boundary and \(S\subset \partial{M}\) is the $(n-1)-\dim$ guard surface. When $x\notin S$, the system is governed by the drift $f:\mathbb{R}\times M\rightarrow M$ and the diffusion matrix $g:\mathbb{R}\times M\rightarrow \mathbb{R}^{n \times n}$.
The reset is described by a Markov kernel \(K:S\times\mathcal B({M})\to[0,1]\), where $\mathcal{B}({M})$ is all Borel sets in ${M}$. \(K(x^{-},A)\) denotes the probability that the post-reset state \(x^+\) lies in a measurable set \(A\subset {M}\) conditioned on \(x^{-} \in S\). Assume \(x\) admits a density \(p(t,x)\) on \(M\).
Let \(D(t,x)=\frac{1}{2}g(t,x)g(t,x)^\top\). The probability flux associated with the continuous diffusion is:
\begin{equation}
    J(t,x) = f(t,x)p(t,x) - \nabla\cdot(D(t,x)p(t,x))
\end{equation}
Let \(n_S(x)\) be the chosen normal direction on \(S\), and define the guard flux density by 
\begin{equation}
\label{eq:flux-shs}
\alpha(t,x)=J(t,x)\cdot n_S(x),\ \forall x\in S.
\end{equation}
Consider the induced density of $K(x^-, dx)$, i.e., $\kappa(x^-, x^+)$ as defined by $K(x^-,A)=:\int_A \kappa(x^-,x^+)dx^+, A\subset {M}.$ The hybrid Fokker--Planck (HFP) equation with reset is:
\begin{equation}
\label{eq:hybrid-fp}
\begin{aligned}
\frac{\partial p(t,x)}{\partial t}
=
&-\nabla\cdot J(t,x)+
\int_S
{\color{blue}\kappa(\xi,x)}\alpha(t,\xi)d\sigma_S(\xi) -
\int_S
{\color{red}\delta(x-\xi)}\alpha(t,\xi)d\sigma_S(\xi).
\end{aligned}\tag{HFP}
\end{equation}
Here \(d\sigma_S\) is the surface measure on $S$. The second term shows that the probability mass {\color{blue}\emph{leaving}} the guard point \(\xi\in S\) is redistributed to post-reset states by \(x^+\sim\kappa(\xi,x^+)\). The last term suggests that the outgoing mass is {\color{red}\emph{removed}} from the pre-reset point \(\xi \in S\). 

\subsection{Geometric Representation of Hybrid Systems}

When the stochasticity of~\ref{eq:stochastic-hybrid-sys} vanishes, the  Markov kernel $K(x^-, dx^+)$ becomes a classical reset function $R(\cdot)$ and \ref{eq:stochastic-hybrid-sys} reduces to a hybrid automaton: 
\begin{equation}
\label{eq:automaton}
\begin{aligned}
    \dot{x} &=f(t,x),\ \ \ x \notin S, \\
    x^+ &= R(x^-),\ x^- \in S.
\end{aligned}\tag{HA}
\end{equation}
Assume $R$ is a diffeomorphism and $S \in \partial M$ satisfies the regularity condition specified in~\citep{simic2005towards}, we have the following result to make \ref{eq:automaton} continuous:
\begin{theorem}
\label{thm:hybrifold}
    Let $\sim$ be the equivalence relation by $x \sim R(x), \forall x \in S$. We collapse the equivalence class by $\sim$ to a point to obtain the quotient manifold $M / \sim$.
By~\cite{simic2005towards}, $M / \sim$ is a $n-\dim$ manifold
with boundary, and both $M/\sim$ and its boundary are piecewise smooth. 
\end{theorem}
\begin{theorem}[Whitney Embedding Theorem \citep{hirsch1976differential}]
\label{theorem:whitney}
 Any $C^r$-manifold $(r\ge 1)$ of dimension $n$ can be embedded into $\mathbb{R}^{2n}$. 
\end{theorem} 
\begin{remark}
    By \Cref{theorem:whitney}, $M/\sim$ admits an embedding that is globally piecewise smooth. Therefore, the hybrid flow of \ref{eq:automaton} can be represented by ODEs with continuous flow~\citep{teng2026embedding, teng2026chyll}.  However, \Cref{thm:hybrifold} cannot be applied to \ref{eq:stochastic-hybrid-sys} as the reset kernel is not a one-to-one mapping.
\end{remark}
\subsection{Problem Formulation}
\begin{problem}
Given $N$ trajectories $\mathcal{X}:=\{(t_{0:T},x_{0:T}^{(k)})\}_{k=1}^N$ sampled from~\eqref{eq:stochastic-hybrid-sys}, learn its conditional path distribution $p(x_{1:T}\mid x_0)$ at prescribed observation times $t_{0:T}$ and initial state $x_0$. 
\end{problem}
We note that this problem is extremely challenging as the flow of \ref{eq:stochastic-hybrid-sys} is discontinuous and stochastic, the guard surface $S$ is unknown, and the reset is stochastic and set-valued.

\section{Continuous Representation of Stochastic Hybrid System}
In this section, we extend the geometric framework to \ref{eq:stochastic-hybrid-sys}. 

\subsection{Augmented Stochastic Hybrid Systems}
We define an augmented-state stochastic hybrid system (A-SHS) where different reset of $x^+\sim K(x^-,dx^+)$ are represented by an auxiliary variable $v \in V \subset \mathbb{R}^n$. For simplification, let the augmented state be $y:=
\begin{bmatrix}
x^\top,\
v^\top
\end{bmatrix}^\top
\in M\times V\subseteq\mathbb R^{2n}$ and we have: 
\begin{equation}
\begin{aligned}
dy &= \overline{f}(t,y)dt+\overline{g}(t,y)d\overline{w},
\quad (x,v) \notin S \times V,\\
x^+ &= x^-+v^-,
\quad
v^+=v^-,
\quad (x^-, v^-)\in S \times V.
\end{aligned}
\label{eq:aug-stochastic-hybrid-sys}
\tag{A-SHS}
\end{equation}
with $\overline{f}=\begin{bmatrix}
    {f}_x^\top & {f}_v^\top
\end{bmatrix}^\top: \mathbb{R} \times (M \times V) \rightarrow \mathbb{R}^{2n} $ the drift, $ \overline{g}=\begin{bmatrix}
    {g}_x^\top & 
    {g}_v^\top
\end{bmatrix}: \mathbb{R} \times (M \times V) \rightarrow \mathbb{R}^{2n \times 2n} $ the diffusion, and $\overline{w}$ the Wiener process of dimension $2n$. 
The reset map in the augmented state thus can be represented by $R(x^-,v^-)
=
(x^-+v^-,v^-).
$\footnote{It is possible to consider other diffeomorphism for the augmented systems.}
Under the Euclidean addition, $R$ is injective and is a diffeomorphism onto its image, since $(x^+,v^+)$ uniquely determines $(x^-,v^-)=(x^+-v^+,v^+)$. 

 We denote the density of $y$ in \ref{eq:aug-stochastic-hybrid-sys} as $\overline{p}(t,y)=\overline{p}(t,x,v)$. Then we have the $\overline{D}(t,y) 
=
\frac{1}{2}
\overline{g}(t,y)\overline{g}(t,y)^\top
\in\mathbb R^{2n\times 2n}$ and the current of $y$:
\begin{equation}
\label{eq:aug-flux}
    \overline{J}(t,y)=\overline{f}(t,y)\overline{p}(t,y) - \nabla\cdot(\overline{D}(t,x)\overline{p}(t,y))
\end{equation}
Given the normal vector of $S\times V$, i.e, $\overline{n}(x,v)=(n_S(x),0_n)$, we have the outgoing flux density as:
\begin{equation}
\label{eq:aug-guard-flux-density}
\overline{\alpha}(t,x,v)
=
\overline{J}(t,x,v)\cdot \overline{n}(x,v),
\
(x,v)\in S\times V.
\end{equation}
Therefore, the augmented hybrid Fokker--Planck equation is the measure-valued equation
\begin{equation}
\label{eq:aug-hybrid-fp}
\textstyle \frac{\partial \overline{p}(t,y)}{\partial t} = -\nabla \cdot \overline{J}(t,y) +
\int_S\int_{ V}
\Big[\delta(y-\begin{bmatrix}
    x'+v' \\ v'
\end{bmatrix})-\delta(y-\begin{bmatrix}
    x' \\ v'
\end{bmatrix})\Big]
\overline{\alpha}(t,x,v)
dv'\ d\sigma_S(x').
\tag{A-HFP}
\end{equation}
The last term indicates the mass injection and removal by the deterministic reset $(x^+,v^+)=(x^-+v^-,v^-)$. Compared to \eqref{eq:hybrid-fp}, the source and sink terms only involve Dirac functions. Ideally, \eqref{eq:aug-stochastic-hybrid-sys} will distinguish different reset branch by the auxiliary variable $v$, as shown in~\Cref{fig:teaser}. 

\subsection{Embedding \ref{eq:aug-stochastic-hybrid-sys} into Higher-Dimensional Space}

We now leverage the Whitney Embedding Theorem to eliminate the discontinuity of \ref{eq:aug-stochastic-hybrid-sys}. Define the equivalence relation on \(M\times V\) by $        (x,v)\sim (x+v,v),\ (x,v)\in S\times V .$
Let $\pi: M\times V \rightarrow (M\times V)/\sim $ be the quotient map\footnote{The map that identifies all points in each equivalence class as a single point.}. Then we have the following theorem: 

\begin{lemma}
\label{lemma:E4n}
There exists a piecewise smooth function $E: M\times V \rightarrow \mathbb{R}^{m \ge 4n}$,
such that $E(x,v)=E(x+v,v),\ (x,v)\in S\times V, $ and the map $\overline{E}$ defined by $\overline{E}\circ \pi:=E$ is an embedding on $(M\times V)/\sim$. 
\end{lemma}
\begin{proof}
    By \citep{simic2005towards}, the quotient manifold $(M\times V)/\sim$ is an $(n+n)-$dimensional piecewise continuous manifold. By the Whitney Embedding theorem, $(M\times V)/\sim$ can be embedded into $\mathbb{R}^{m}$ whenever $m\ge 2(n+n)=4n$. Let the embedding be $\overline{E}$. Since \(\pi\) identifies \((x,v)\sim(x+v,v)\), the composed map \(E=\overline{E}\circ\pi\) satisfies $E(x,v)=E(x+v,v),\ (x,v)\in S\times V.$
\end{proof}
By \Cref{lemma:E4n}, we can eliminate the reset of \ref{eq:aug-stochastic-hybrid-sys} in the embedded image as shown in~\Cref{fig:teaser}:
\begin{theorem}[Embedding of A-SHS]
\label{thm:glued-manifold-fp}
Let $ z=E(x,v) \in Z := \operatorname{Img} E$, with $E$ the embedding defined in~\Cref{lemma:E4n}. We have the pushforward dynamics of $z$ by a piecewise smooth SDE:
\begin{equation}
\label{eq:glued-sde}
    \begin{aligned}
        dz &= F(t, z)dt + G(t,z)dw,\ \forall z \in Z. 
    \end{aligned}
\end{equation}
Let $D_z(t,z) := \frac{1}{2}G(t,z)G(t,z)^\top$, the density $p_z(t,z)$ is governed by the Fokker--Planck equation:
\begin{equation}
\textstyle
\label{eq:glued-fp}
        \frac{\partial p_z(t,z)}{\partial t}
        =
        -\nabla \cdot
        \left(
        F(t,z)p_z(t,z)
        \right)
        +
        \nabla \cdot \nabla \cdot
        \left(
        D_z(t,z)p_z(t,z)
        \right),
        \
        z\in Z.
\end{equation}
\end{theorem}
\begin{proof}
We first show the pushforward dynamics of \ref{eq:aug-stochastic-hybrid-sys}. By $E(x,v)=E(x+v,v), (x,v)\in S\times V$, the discrete reset vanishes at the guard surface $S\times V$ by:
\begin{equation}
            z^+= E(x^+,v^+) = E(x^-+v^-,v^-) = E(x^-,v^-)=z^-,\ \forall (x^-,v^-) \in S\times V.
\end{equation}
Then we have the pushforward dynamics by the Taylor expansion of It\^{o} process~\citep{oksendal2003stochastic}:
\begin{equation}
\textstyle
    F(t,z):= \frac{\partial E}{\partial y}\overline{f}(t,x,v) + \frac{1}{2}\sum\nolimits_{i,j=1}^{2n} \overline{D}_{ij}(t,x,v)\frac{\partial^2 E}{\partial y_i \partial y_j},\quad \ G(t,z):=\frac{\partial E}{\partial y}\overline{g}(t,x,v).
\end{equation}
We note that $z$ is an embedding of $(x,v)$ that is injective and has a unique inverse. Therefore, $F$ and $G$ can be uniquely represented by $t$ and $z$. Thus we have \eqref{eq:glued-sde} as a function of $z$.

It remains to check the reset contribution. For the reset source term, integrating against a smooth test function on compact support, i.e.,  $\phi(z) :Z\rightarrow \mathbb{R}, \phi(z) \in C_c^\infty(Z)$ yields:
\begin{equation}
\textstyle
    \int_{M \times V} \phi(E(y')) \left( \int_S\int_V \delta\left(y'-(x+v,v)\right) \overline{\alpha}(t,x,v) \,dv\,d\sigma_S(x) \right) dy'.
\end{equation}
By the property of the Dirac measure, we move the test function to nested integration to reduce the above term to $    \int_S\int_{ V} \phi\left(E(s+v,v)\right) \overline{\alpha}(t,s,v) \,dv\,d\sigma_S(s)$.
For any smooth test function $\phi: Z \rightarrow \mathbb{R}$, the net reset source-sink contribution to the forward equation in the embedded image is $I_R[\phi]=\int_S\int_{ V}
\Big(
\phi\left(E(x+v,v)\right)
-
\phi\left(E(x,v)\right)
\Big)
\overline{\alpha}(t,x,v)dvd\sigma_S(x).
$

As $ E(x+v,v) = E(x,v), \forall (x,v) \in S\times V$, this integrand vanishes point-wise, yielding $I_R[\phi]=0$. Therefore, the pushed-forward dynamics on \(z\) can be governed entirely by \eqref{eq:glued-fp}.
\end{proof}


\subsection{A Construction of \ref{eq:aug-stochastic-hybrid-sys}}
\label{sec:cons-aug-shs}
\begin{figure}
    \centering
    \includegraphics[width=0.3\linewidth]{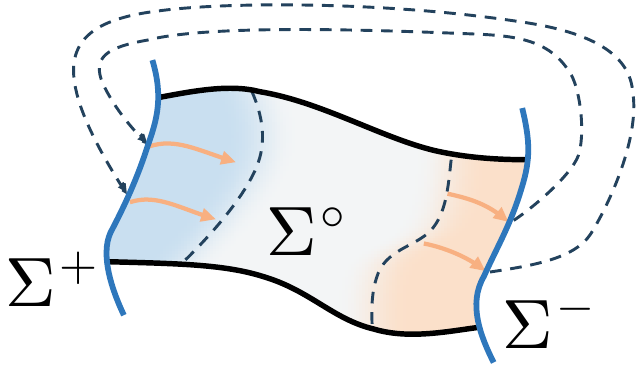}
    \caption{Target distribution. The dashed lines indicate the boundary of the partition of unity. }
    \label{fig:construction}
\end{figure}
\ref{eq:aug-stochastic-hybrid-sys} 
provides a pathway to replace the discrete reset with a continuous SDE. Now we construct such an \ref{eq:aug-stochastic-hybrid-sys} to approximate \ref{eq:stochastic-hybrid-sys}. 
We define a residual reset kernel $Q:S\times\mathcal B( V)\to[0,1]$:
\begin{equation}
\label{eq:residual-kernel}
K(x^-,A)
=
Q\left(
x^-,
\{v\in V\ |\ x^-+v\in A\}
\right),
\ A\in\mathcal B(M).
\end{equation}
Equivalently, if $v^-\sim Q(x^-,dv)$, then $x^+=x^-+v^-\sim K(x^-,dx^+)$. Similar to $\kappa(x^-, x^+)$, we define the density of $Q$ by $\kappa_Q(x^-,v)$. Now we show that we can construct a target distribution whose marginal density in the $x-$direction is equivalent to the solution of \ref{eq:stochastic-hybrid-sys}.

\begin{construction}[Target Density]
\label{cons:partition-mixing-ashs}
Define the pre-reset set $\Sigma^-:=S\times V$, its post-reset image $\Sigma^+$, and the interior
$\Sigma^\circ:=(M\times V)\setminus\overline{(\Sigma^-\cup\Sigma^+)}$. 
Let $\chi^-(x)+\chi^+(x)+\chi^\circ(x)=1$ be a smooth partition of unity\footnote{Weight functions that allows local constructions to be blended smoothly into a global one. In \Cref{cons:partition-mixing-ashs}, $v$ has the identical distribution as \eqref{eq:residual-kernel} on the guard surface and the pushforward on its image. } associated with $\Sigma^-$, $\Sigma^+$, and $\Sigma^\circ$, and prescribe the conditional density of $v$:
\begin{equation}
p^\star(v\mid x)
=
\chi^-(x)\kappa_Q(x,v)
+
\chi^+(x)(R_\#\kappa_Q)(x,v)
+
\chi^\circ(x)\mathcal N(v;\mu, \delta^2I).
\end{equation}
where $R_\# \kappa_Q$ is the pushforward density of $\kappa_Q$ under the reset $R$. Let the solution of \ref{eq:hybrid-fp} be $p^\star(t,x)$, we have the target density that preserves the marginal density of $x$ for some $T < \infty$:
\begin{equation}
    \overline{p}^\star(t,x,v)=p^\star(t,x)p^\star(v|x),\ t \in [0, T], \tag{Target Density}
\end{equation}
and the target density $J^\star:=[J^{\star\top}_x,\ J^{\star\top}_v]^\top$ given
by~\eqref{eq:aug-flux}. The target distribution is shown in~\Cref{fig:construction}. 
\end{construction}


Now we design \ref{eq:aug-stochastic-hybrid-sys} whose evolution of density will approximate $\overline{p}^\star(t,x,v)$. 
\begin{construction}[Approximate \ref{eq:aug-stochastic-hybrid-sys}]
\label{cons:approximate-ashs}
Choose ${f}_x=f$ and ${g}_{x}=g$ to make the $x$-dynamics equivalent to \ref{eq:stochastic-hybrid-sys} in the interior. Given the $x$-current $J_x^\star= f_x\overline{p}^\star-\nabla_x\cdot({D}_{xx}\overline{p}^\star) - \nabla_v\cdot({D}_{xv}\overline{p}^\star)$, we solve for ${D}_{xv}$. 
For any $\eta>0$, assume $D$ invertible, choose
${D}_{vv}
=
{D}_{xv}^\top D^{-1}{D}_{xv}
+
\eta I,
$
and solve 
$\frac{\partial p^\star}{\partial t} = -\nabla \cdot J^\star$ 
for the $v$-dynamics. The resulting~\ref{eq:aug-stochastic-hybrid-sys} dynamics are given by:
\begin{equation}
\textstyle
\overline{f}
=
\begin{bmatrix}
f \\
\frac{
 J_v^\star 
+
\nabla_x\cdot
\left(
{D}_{xv}^\top \overline p^\star
\right)
+
\nabla_v\cdot
\left(
{D}_{vv} \overline p^\star
\right)
}{
\overline p^\star
}
\end{bmatrix},
\quad
\overline{D}
=
\frac{1}{2}\overline{g}\overline{g}^\top,
\quad
\overline{g}
:=
\begin{bmatrix}
g & 0 \\
g_{xv} & g_v
\end{bmatrix}.
\end{equation}
Here, $\overline{g}$ is obtained from $\overline{D}$ by choosing
$g_{xv}={D}_{xv}^\top D^{-1}g$ and
$g_v=\sqrt{2\eta}I$.
Together with the deterministic reset $R(x,v)=(x+v,v)$, these coefficients define an approximate \ref{eq:aug-stochastic-hybrid-sys} whose target density is
$\overline p^\star(t,x,v)$. The construction of ${D}_{xv}$ is deferred to Appendix~\ref{appx:a-shs-cons}.
\end{construction}
\begin{remark}[Interior mixing and exponential forgetting]
In the interior where $\chi^\circ=1$, we have
$p^\star(v\mid x)=\mathcal N(v;\mu,\delta^2I)$,
${D}_{xv}=0$, and $\overline J_v^\star=0$. Hence, we can choose ${f}_v=-\lambda(v-\mu),
\
{D}_{vv}=\lambda\delta^2I,
$
and the $v-$dynamics reduce to the Ornstein--Uhlenbeck (OU) process, i.e., $$dv
=
-\lambda(v-\mu)\,dt
+
\sqrt{2\lambda}\,\delta\,dw_v.$$
If the trajectory remains in the interior for $\tau$, then $v(t+\tau)
=
\mu
+
e^{-\lambda\tau}(v(t)-\mu)
+
\delta\sqrt{1-e^{-2\lambda\tau}}\,\xi,
\
\xi\sim\mathcal N(0,I).
$ Consequently, the dependence of the receding reset $v(t)$ is forgotten exponentially at rate $\lambda$ and converges to Gaussian.
\end{remark}
As shown in~\Cref{fig:construction}, \Cref{cons:approximate-ashs} first generates the desired density on $\Sigma^-$ and then mixes it to a Gaussian in $\Sigma^\circ$ before the next reset. The details of both constructions are referred in Appendix~\ref{appx:a-shs-cons}.

\vspace{-3mm}
\section{Main Algorithms}
\label{sec:alg}

\vspace{-2mm}
\begin{wrapfigure}{R}{0.4\textwidth}
\begin{minipage}{\linewidth}
\vspace{-10mm}
\begin{algorithm}[H]
\caption{} 
\label{alg:main}
\begin{algorithmic}
\footnotesize
\Require Dataset $\mathcal X$; sampler schedule $\{\beta_\ell\}_{\ell=1}^L$; batch size $B$; rollouts per initial state $M$; training steps $J$; learning rate $\eta$.

\For{$j=1,\ldots,J$}
    \State Sample $\{x_{0:T}^{(k)}\}_{k=1}^B$ from $\mathcal X$ at $t_{0:T}$
    \State $z_i^{(k,r)}\sim p_\theta(\cdot\mid x_i^{(k)})$ \Comment{Encode}
    \State $\hat z_0^{(k,r)}\gets z_0^{(k,r)}$ \Comment{Initialize}
    \State $\hat z_{0:T}^{(k,r)}\xleftarrow{\eqref{eq:SDE-EM}}\hat z_0^{(k,r)}$ \Comment{Rollout}
    \State $\hat x_i^{(k,r)}\gets D_\theta(\hat z_i^{(k,r)})$ \Comment{Decode}
    \State Compute $\mathcal L(\theta)$ using~\eqref{eq:training-loss} \Comment{Loss}
    \State $\theta\gets\theta-\eta\nabla_\theta\mathcal L(\theta)$ \Comment{Update}
\EndFor
\State \Return $p_\theta(z\mid x)$, $F_\theta$, $G_\theta$, and $D_\theta$
\end{algorithmic}
\end{algorithm}
\end{minipage}
\vspace{-4mm}
\end{wrapfigure}
Building on \Cref{thm:glued-manifold-fp}, we learn~\eqref{eq:stochastic-hybrid-sys} using a stochastic encoder, a latent SDE, and a deterministic decoder.
For clarity, we describe these components using a single observed trajectory $\mathrm{x}=x_{0:T}$ at observation times $t_{0:T}$ and a single predicted trajectory $\hat{\mathrm{x}}=\hat x_{0:T}$.

\textbf{Stochastic Encoder:} The reset kernel $x^+ \sim \kappa(x^-,x^+)$ assigns a distribution of post-reset states to each state $x^-$, inducing a distribution of reset displacements $v=x^+-x^-$.
By \Cref{thm:glued-manifold-fp}, the augmented state $(x,v)$ can be represented by a latent state $z=E(x,v)$ by \Cref{lemma:E4n}.
We therefore model the encoder as $z\sim p_\theta(z\mid x)$ and sample $z_i\sim p_\theta(z\mid x_i)$ at $t_i$.

In practice, $p_\theta(z\mid x)$ is represented by a conditional diffusion sampler.
Starting from $\zeta_L\sim\mathcal N(0,I)$, we apply $L$ differentiable DDIM~\citep{song2021denoising} updates conditioned on $x$:
$\zeta_{\ell-1}
=
\sqrt{\bar\alpha_{\ell-1}}
\frac{\zeta_\ell-\sqrt{1-\bar\alpha_\ell}\,\epsilon_\theta(\zeta_\ell,\ell,x)}
{\sqrt{\bar\alpha_\ell}}
+
\sqrt{1-\bar\alpha_{\ell-1}}\,\epsilon_\theta(\zeta_\ell,\ell,x),
\ \ell=L,\ldots,1,$
where $\bar\alpha_\ell=\prod_{s=1}^{\ell}(1-\beta_s)$ and $\bar\alpha_0=1$.
The output is $z=\zeta_0$, and the noise predictor $\epsilon_\theta$ is trained end-to-end without a separate denoising loss.

\textbf{Latent SDE:} Let $d_z=\dim Z$.
We represent the latent drift and diffusion by $F_\theta:\mathbb R\times Z\rightarrow\mathbb R^{d_z}$ and $G_\theta:\mathbb R\times Z\rightarrow\mathbb R^{d_z\times 2n}$.
Given $\hat z_0=z_0$, we generate $\hat{\mathrm{z}}=\hat z_{0:T}$ using the Euler--Maruyama scheme:
\begin{equation}
\label{eq:SDE-EM}
    \hat z_{i+1}=\hat z_i+F_\theta(t_i,\hat z_i)\Delta t_i+G_\theta(t_i,\hat z_i)\Delta w_i,
    \ \Delta w_i\sim\mathcal N(0,\Delta t_i I_{2n}),
\end{equation}
where $\Delta t_i=t_{i+1}-t_i$, $i=0,\ldots,T-1$, and the Brownian increments are independent.

\textbf{Deterministic Decoder:} The inverse $E^{-1}$ recovers an equivalence class in $(M\times V)/\sim$.
Choosing its post-reset representative at the glued seam defines a single-valued state decoder, which we approximate by $D_\theta:Z\rightarrow M$.
The underlying decoder can be discontinuous at the seam: its one-sided limits recover $x^-$ and $x^+$ while the latent trajectory remains continuous.
We obtain the predicted trajectory $\hat{\mathrm{x}}=\hat x_{0:T}$ by applying $\hat x_i=D_\theta(\hat z_i)$ at each observation time.

\textbf{Cost Function Design:}
At time $t_i$, let $p_{u,i}$ and $\hat p_{u,i}$ denote the data and predicted distributions for $u\in\{x,z\}$, with $p_{z,i}$ induced by encoding $x_i$.
Let $p_{\mathrm{x}}$ denote the data path distribution and $\hat p_{\mathrm{x}}$ the predicted path distribution when $x_0$ is drawn from its data distribution.
For two distributions $p_A$ and $p_B$ with finite first moments, their squared energy distance is denoted by $\mathcal E^2(p_A,p_B)$\footnote{$\mathcal E^2(p_A,p_B)=2\mathbb E\|a-b\|_2-\mathbb E\|a-a'\|_2-\mathbb E\|b-b'\|_2$, where $a,a'\sim p_A$ and $b,b'\sim p_B$ are mutually independent.}.

Inspired by \Cref{thm:glued-manifold-fp}, we define the \emph{state distribution loss}
$\mathcal L_u
=\frac{1}{T+1}\sum_{i=0}^T
\mathcal E^2(p_{u,i},\hat p_{u,i})$,
$u\in\{x,z\}$, to match the distribution evolution at each observation time.
We also define the \emph{unconditional path distribution loss}
$\mathcal L_{\mathrm{x}}^{\mathrm{traj}}
=\mathcal E^2(p_{\mathrm{x}},\hat p_{\mathrm{x}})$
to match the joint distribution across observation times.
For the observed trajectory $\mathrm{x}$ with initial state $x_0$, the \emph{conditional path distribution loss} is
$\mathcal L_{\mathrm{x}}^{\mathrm{traj,c}}
=\frac{1}{2\sqrt{(T+1)d_x}}
\mathcal E^2(\delta_{\mathrm{x}},\hat p_{\mathrm{x}\mid x_0})$,
where $d_x$ is the state dimension and $\hat p_{\mathrm{x}\mid x_0}$ is the conditional path distribution induced by the stochastic encoder, latent SDE, and decoder.

With $w$ denoting the weights and the loss $\mathcal L_g$\footnote{$\mathcal L_g=\frac{1}{T+1}\sum_{i=0}^T\|G_\theta(t_i,\hat z_i)\|_F^2$.} to regularize diffusion, the overall objective is
\begin{equation}
\textstyle
    \label{eq:training-loss}
    \mathcal L(\theta)
    =
    w_z\mathcal L_z
    +
    w_x\mathcal L_x
    +
    w_{\mathrm{x}}^{\mathrm{traj}}\mathcal L_{\mathrm{x}}^{\mathrm{traj}}
    +
    w_{\mathrm{x}}^{\mathrm{traj,c}}\mathcal L_{\mathrm{x}}^{\mathrm{traj,c}}
    +
    w_g\mathcal L_g.
\end{equation}
For multiple observed and independently generated trajectories, we estimate the distributional losses from empirical samples and average the conditional loss and diffusion regularizer over the corresponding trajectories.
We summarize the training scheme in \Cref{alg:main}, where we use the superscript $(k,r)$ to denote the $r$-th encoder sample corresponding to the $k$-th data trajectory, with batch size $B$ and $M$ encoder samples per state.
\section{Numerical Experiments}
\label{sec:exp}

\begin{figure}
    \centering
    \includegraphics[
        width=\linewidth,
        height=0.32\textheight,
        keepaspectratio
    ]{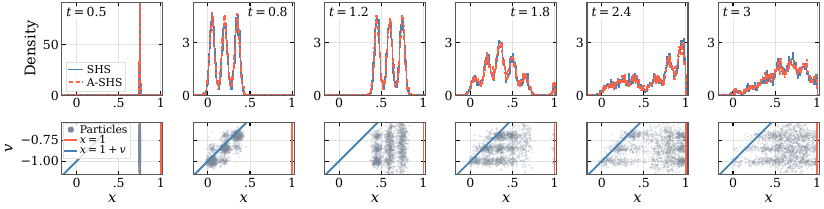}
    \caption{The density evolution of A-SHS for \eqref{eq:shs}. We find that the marginal density of A-SHS in $x$ matches that of \eqref{eq:shs}. As shown in the $x-v$ plot, the density is mixed by the OU process first and then redistributed to the GMM to approximate the reset by Construction~\ref{cons:partition-mixing-ashs} and~\ref{cons:approximate-ashs}.}
    \label{fig:xvpcl}
    \vspace{3mm}
    \includegraphics[
        width=\linewidth,
        height=0.32\textheight,
        keepaspectratio
    ]{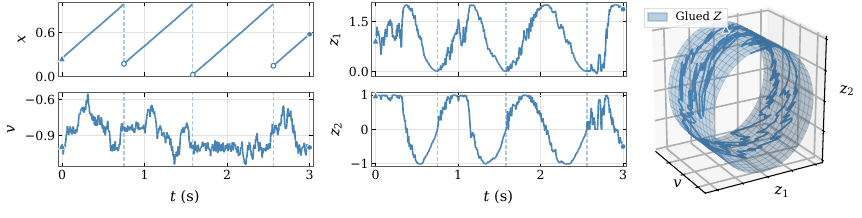}
    \caption{The time evolution of the A-SHS and the embedded trajectory in 3D space. We can see that $x$ is reset to a different state after it hits the guard at $x=1$. The $v$ replaces the role of the reset function with continuous paths. The augmented path can be embedded into a higher-dimensional space with variable $z$ to mitigate the discontinuity.}
    \label{fig:xvtraj}
\end{figure}
\subsection{Construction of A-SHS}
We illustrate the construction in~\Cref{sec:cons-aug-shs} using a $1-\dim$ hybrid system defined on $(0, 1]$ with constant drift and a Gaussian mixture reset with three peaks $\mathcal{C}:=\{0, 0.15, 0.3\}$:
\begin{equation}
\label{eq:shs}
\textstyle
dx = dt + \sigma\,dw,\quad x<1,\qquad
x^+ \sim \frac{1}{3}\sum_{x_c\in\mathcal{C}}
\operatorname{TN}_{(-\infty,1)}(x_c,\eta^2),
\quad x^-=1.
\end{equation}
Here, $\operatorname{TN}_{(-\infty,1)}(x_c,\eta^2)$ is the truncated normal distribution that denotes
$\mathcal{N}(x_c,\eta^2)$ conditioned on $x<1$, ensuring that $x^+ < 1$ almost surely.
In the extended state space $X \times V$, the guard surface and its reset image are given by $S:=\{(x,v)\mid x=1,\ v\in V\}$ and $R(S)=\{(x,v)\mid x=v+1\}$, respectively. We use Monte Carlo simulations to estimate the density evolution of~\eqref{eq:shs} and use the rollout to construct a time-varying A-SHS. As shown in~\Cref{fig:xvpcl}, the partitioned OU process enables the A-SHS to closely reproduce the density evolution of ~\eqref{eq:shs}. To verify \Cref{thm:glued-manifold-fp}, we consider the $3-\dim$ embedding $ E(x,v)=\bigl(v,1-\cos\theta,\sin\theta\bigr),
$ with $\theta=2\pi\frac{x-1-v}{-v}$. As illustrated in~\Cref{fig:xvtraj}, \eqref{eq:shs} resets $x^-=1$ to different $x^+ \in \mathcal{C}$, while the A-SHS has deterministic resets on $X\times V$ and can be embedded into a higher-dimensional space with continuous sample paths.

\subsection{Learning Stochastic Hybrid Systems}
We now apply \Cref{alg:main} to learn SHS from time-series data. We consider some classic examples in hybrid system theory with different topologies, including the bouncing ball, the Klein bottle, and the torus systems. For a deterministic system, the Torus system can be embedded into 3D Euclidean space, while the Klein bottle system can only be embedded into 4D space. We consider these systems with different reset kernels, such as a GMM distribution or a uniform distribution. We also consider a switching system that randomly chooses the reset kernel of the deterministic torus system or the Klein bottle system. Then we consider the $k$-bouncing ball system in the $x-y$ plane with dimension equal to $4k$ (velocity and position in horizontal and vertical directions). We consider $0.32s$ training horizon. All comparisons and ablations are evaluated on the same normalized test data with horizon $1s$, $3s$ or $5s$. 
By default, we consider the latent dimension of our method to be $d_z = 2(d_x + d_x)$ as inspired by \Cref{lemma:E4n}. We present the result with horizon $1s$, and the remaining cases are shown in Appendix~\ref{appx:baselines} and Appendix~\ref{appx:ablations}. We note that the baselines exhibit worse results with a horizon longer than $1s$, while our method has consistent accuracy.    

\textbf{Comparison with baselines:}
We compare our method with \citep{teng2026embedding}, a latent ODE framework for \emph{deterministic} hybrid system learning; \citep{li2020scalable} that learns latent SDEs through variational inference; \citep{bartosh2025sde} that solves a simulation-free variational objective; and \citep{kiyohara2025neural} that learns conditional transition kernels with Chapman--Kolmogorov consistency. We also evaluate \citep{li2020scalable} augmented with additional losses, i.e., $\mathcal L_x$ and $\mathcal L_{\mathrm{x}}^{\mathrm{traj}}$.

As shown in \Cref{tab:baseline-mean-1s}, our method achieves the lowest mean conditional and unconditional path distribution losses across all 11 benchmark settings. The relative performance of the competing approaches varies across datasets and metrics. Adding $\mathcal{L}_x$ and $\mathcal{L}_{\mathrm x}^{\mathrm{traj}}$ to \citep{li2020scalable} improves performance on some datasets, but does not consistently improve the results or close the gap to our method. 
\citep{kiyohara2025neural} and \citep{bartosh2025sde} exhibit larger errors or divergence under the evaluated configurations, with particularly large errors in the unconditional path distribution loss. As \citep{teng2026embedding} omits the stochastic effect, it triggers the divergence criterion on most cases. The predicted distribution evolution of the bouncing ball system is illustrated in \Cref{fig:bb-gmm}, showing that our method correctly captures the distribution evolution. As shown in \Cref{fig:bb-gmm-rollout}, our method correctly captures the path distribution by predicting bounces with different restitution coefficients. More visualizations and comparisons are illustrated in Appendix~\ref{appx:baselines} and~\ref{appx:visualizations}. 
\begin{figure}
    \centering
    \includegraphics[width=1\linewidth]{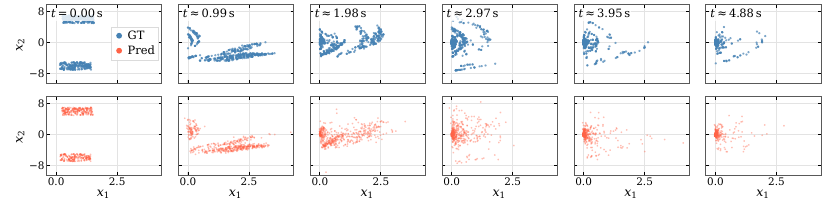}
    \caption{The predicted probability evolution of bouncing ball systems under GMM reset. Our method correctly captures the long-term distribution with a training horizon of $0.32s$. }
    \label{fig:bb-gmm}
    \vspace{-3mm}
\end{figure}
\begin{figure}
    \centering
    \includegraphics[width=1\linewidth]{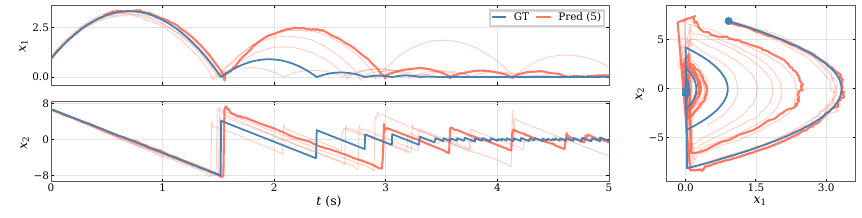}
    \caption{We plotted five predicted trajectories of the bouncing ball system under GMM reset conditioned on the same initial condition. The reset function independently chooses restitution by a GMM centered at $0.5$ or $0.9$ at each collision. Our method correctly captures this pattern.}
    \label{fig:bb-gmm-rollout}
    \vspace{-4mm}
\end{figure}

\TableBaselineMeanHorizonOneSecond
\TableAblationSummaryHorizonOneSecond

\textbf{Ablations:} We study the effects of 1) latent dimension, 2) loss function design, and 3) encoder randomness.
\Cref{tab:ablation-summary-1s} summarizes the changes in conditional and unconditional path distribution losses relative to the default model.
Reducing the latent dimension generally degrades performance, although $d_z=3d_x$ performs comparably to the default $d_z=4d_x$ at the short horizon.
At longer horizons, smaller latent dimensions can lead to substantial degradation or divergence, including degradation for $d_z=3d_x$ on some benchmarks, as shown in \Cref{tab:ablation-summary-3s,tab:ablation-summary-full}.
Thus, the longer-horizon results reinforce the benefit of a larger latent space and indicate that comparable short-horizon performance does not ensure comparable long-horizon behavior.

For loss function design, removing either path distribution loss alone has a relatively modest effect on average, whereas removing both consistently degrades both metrics.
Removing $\mathcal L_z$ also consistently degrades performance, with larger effects at longer horizons and divergence on one benchmark over the full trajectory.
These trends are consistent across horizons, as reported in \Cref{tab:ablation-summary-3s,tab:ablation-summary-full}.
The benefit of latent-space supervision is consistent with the findings of \citet{teng2026embedding}.

Finally, replacing the stochastic encoder with a deterministic MLP moderately degrades performance on average, particularly for the unconditional path distribution loss.
This trend persists at longer horizons, as shown in \Cref{tab:ablation-split-sampler-conditional-energy-loss-1s,tab:ablation-split-sampler-conditional-energy-loss-3s,tab:ablation-split-sampler-conditional-energy-loss-full}.
These results support the benefit of encoder randomness, consistent with the one-to-many mapping from state variables to latent variables.
However, the moderate degradation suggests that a deterministic encoder remains viable, as the latent SDE~\eqref{eq:SDE-EM} can also provide randomness. The details of the ablations and the cases on longer horizons are shown in Appendix~\ref{appx:ablations}.

\vspace{-6pt}
\section{Discussions}
\vspace{-6pt}

\textbf{Latent Dimension}: 
Whitney Embedding Theorem provides a general dimension bound to embed the $2n$-dimensional glued manifold $(X\times V)/\sim$ into $\mathbb{R}^{4n}$. The structure of the reset $(x^+,v^+)=(x^-+v^-,v^-)$ may improve this bound: for fixed $v$, identifying $x\sim x+v$ on the guard $S$ yields an $n$-dimensional space, embeddable into $\mathbb{R}^{2n}$. If these embeddings vary smoothly with $v$ and form a global embedding, retaining the $n$ coordinates of $v$ can give a $3n$ bound. Whether the $v$-wise manifolds vary smoothly remains open for future research.

\textbf{Loss Functions:} 
As each sampled path is stochastic, we argue that there is no need to match a specific path. As discussed in~\cite{li2020scalable}, overfitting to a single path will result in the diffusion term converging to zero. In our method, we instead match the path distribution, considering the latent SDE as a sampler to generate the paths. As indicated in the comparative study, the loss function enables us to accurately recover the path distributions beyond the training horizon. 

\textbf{Memory of Latent Variables:} 
Latent variables can encode both reset events and past-state memory. When the closures of the guard and its images are disjoint, pure OU dynamics in an interior region allow arbitrarily fast exponential memory decay by increasing the mean-reversion parameter. However, this only approximately erases memory, without guaranteeing complete removal in finite time. Fully recovering the original system's Markov property remains a direction for future research. 

\textbf{Future Applications:} The proposed methods can be applied to robotics control, such as ~\citep{teng2022input, teng2024convex, 10301632, teng2021toward, liu2025discrete, teng2024generalized, Teng-RSS-23, ghaffari2022progress, teng2022lie, teng2022error, teng2021legged, yu2023fully, he2024legged, teng2024gmkf, he2024legged,chang2026survey,teng2026max,liu2025ego, he2025invariant, iwasaki2025learning, liu2026mepoly, dong2025online, li2026stein}

\vspace{-6pt}
\section{Conclusion}
\vspace{-6pt}
In this work, we show that a continuous latent SDE can accurately approximate a stochastic hybrid system with a Markov reset kernel. The key to this result is to represent different reset branches of the reset kernel by auxiliary variables, so that the augmented hybrid system admits a deterministic reset function that enables the topological gluing. In the embedded image, the glued system admits continuous sample paths, and its state distribution can be modeled by a Fokker-Planck equation without the source and sink terms due to the resets. Based on this theorem, we can accurately learn the flow of hybrid systems from time-series data.  



\section*{Reproducibility Statement}
We provide the assumptions and derivations underlying our theoretical results in the main text, with additional details and complete derivations in the appendix. The learning procedure and objective are described in Section \ref{sec:alg}, and the experimental settings, evaluation protocols, and benchmark construction are documented in Section \ref{sec:exp} and the appendix. These details are intended to facilitate reproduction of both our theoretical and empirical results. We plan to release the implementation and experiment code upon publication.

\section*{AI Disclosure}
In this work, we used generative AI tools to assist with mathematical proofs, revise and execute code for method implementation and experiments, and improve the clarity and readability of the manuscript. We did not use generative AI tools for other tasks requiring disclosure under the conference policy. The authors carefully reviewed and verified all AI-assisted proofs for mathematical correctness, reviewed and tested AI-assisted code, and checked the resulting experimental outputs. All AI-assisted writing was reviewed and revised by the authors. We take full responsibility for the final content of this work, including all text, mathematical claims, proofs, code, and experimental results produced with the assistance of generative AI.

\bibliography{iclr2025_conference}
\bibliographystyle{iclr2025_conference}



\appendix
\section*{Appendix}
\section{A Time-Dependent Construction of the Augmented Dynamics}
\label{appx:a-shs-cons}
We now give the details of the construction in Section~\ref{sec:cons-aug-shs}. Let $p^\star(t,x)$ be the solution of~\ref{eq:hybrid-fp}, and write the augmented drift and diffusion tensor as
\begin{equation}
\label{eq:constructed-aug-diffusion-matrix}
\overline{f}(t,y)
=
\begin{bmatrix}
f_x(t,x,v) \\
f_v(t,x,v)
\end{bmatrix},
\quad
\overline{D}(t,y)
=
\begin{bmatrix}
D_{xx}(t,x,v) & D_{xv}(t,x,v) \\
D_{xv}(t,x,v)^\top & D_{vv}(t,x,v)
\end{bmatrix}.
\end{equation}
Here, $y=[x^\top,\ v^\top]^\top\in M\times V\subseteq\mathbb R^{2n}$, each block of $\overline{D}$ is in $\mathbb R^{n\times n}$, and $\overline{D}=\frac{1}{2}\overline{g}\overline{g}^\top$. We use $J(t,x)$ and $\alpha(t,s)$ for the probability current and guard flux associated with $p^\star(t,x)$:
\begin{equation}
J(t,x)
=
f(t,x)p^\star(t,x)
-\nabla_x\cdot\left(D(t,x)p^\star(t,x)\right),
\quad
\alpha(t,s)=J(t,s)\cdot n_S(s),\ s\in S.
\end{equation}
We assume the densities and coefficients have the regularity required below, and that $D(t,x)$ is positive definite on the region of construction.

\subsection{Construction of the Conditional Density}

We use the conditional density in Construction~\ref{cons:partition-mixing-ashs}. Recall the pre-reset set $\Sigma^-:=S\times V$, its post-reset image $\Sigma^+:=R(\Sigma^-)$, and the interior $\Sigma^\circ:=(M\times V)\setminus\overline{(\Sigma^-\cup\Sigma^+)}$. With the smooth partition of unity $\chi^-(x)+\chi^+(x)+\chi^\circ(x)=1$, the target conditional density is
\begin{equation}
\label{eq:conditional-density-extension}
p^\star(v\mid x)
=
\chi^-(x)\kappa_Q(x,v)
+\chi^+(x)(R_\#\kappa_Q)(x,v)
+\chi^\circ(x)\mathcal N(v;\mu,\delta^2I).
\end{equation}
The local densities are understood as smooth normalized extensions on their corresponding neighborhoods. We assume these extensions and the partition of unity are compatible with the prescribed guard trace. Since the weights are nonnegative, independent of $v$, and sum to one, we have $\int_V p^\star(v\mid x)dv=1$. In particular,
\begin{equation}
\label{eq:conditional-density-guard-trace}
p^\star(v\mid s)=\kappa_Q(s,v),\quad (s,v)\in\Sigma^-.
\end{equation}
In the interior where $\chi^\circ=1$, the conditional density reduces to $p^\star(v\mid x)=\mathcal N(v;\mu,\delta^2I)$. The Gaussian is normalized on $V=\mathbb R^n$; for a bounded auxiliary domain, it is replaced by its normalized restriction to $V$.

\subsection{Construction of the Target Augmented Density}

Define the target augmented density by
\begin{equation}
\label{eq:constructed-aug-density}
\overline{p}^\star(t,y)
=
\overline{p}^\star(t,x,v)
=
p^\star(t,x)p^\star(v\mid x),\quad t\in[0,T].
\end{equation}
This factorization preserves the prescribed $x$-marginal:
\begin{equation}
\int_V\overline{p}^\star(t,x,v)dv
=
p^\star(t,x)\int_Vp^\star(v\mid x)dv
=
p^\star(t,x).
\end{equation}
We assume $\overline{p}^\star(t,x,v)>0$ on the region where the coefficients below are constructed, so division by $\overline{p}^\star$ is well-defined.

\subsection{Construction of the $x$-Direction Probability Current}

Following Construction~\ref{cons:partition-mixing-ashs}, let $J^\star=[J_x^{\star\top},\ J_v^{\star\top}]^\top$ denote the target augmented probability current. We prescribe its $x$-component by
\begin{equation}
\label{eq:constructed-x-current}
J_x^\star(t,x,v)=p^\star(v\mid x)J(t,x).
\end{equation}
Integrating over $v$ gives $\int_VJ_x^\star(t,x,v)dv=J(t,x)$. On the augmented guard,~\eqref{eq:conditional-density-guard-trace} yields
\begin{equation}
\label{eq:constructed-guard-flux}
\begin{aligned}
\overline{\alpha}(t,s,v)
&=J_x^\star(t,s,v)\cdot n_S(s) \\
&=\kappa_Q(s,v)J(t,s)\cdot n_S(s)
=\alpha(t,s)\kappa_Q(s,v).
\end{aligned}
\end{equation}
Since the augmented normal is $\overline{n}(s,v)=(n_S(s),0_n)$, this guard flux is independent of $J_v^\star$.

\subsection{Construction of the $x$-Dynamics}

To preserve the continuous $x$-dynamics of~\ref{eq:stochastic-hybrid-sys}, choose $f_x(t,x,v)=f(t,x)$ and $D_{xx}(t,x,v)=D(t,x)$. The cross-diffusion matrix $D_{xv}(t,x,v)\in\mathbb R^{n\times n}$ is chosen so that the resulting $x$-current equals~\eqref{eq:constructed-x-current}:
\begin{equation}
\label{eq:cross-diffusion-equation}
\begin{aligned}
\nabla_v\cdot\left(D_{xv}(t,x,v)\overline{p}^\star(t,x,v)\right)
=f(t,x)\overline{p}^\star(t,x,v)
-\nabla_x\cdot\left(D(t,x)\overline{p}^\star(t,x,v)\right)
-J_x^\star(t,x,v).
\end{aligned}
\end{equation}
The right-hand side has zero integral over $V$. Indeed,
\begin{equation}
\label{eq:x-flow}
\begin{aligned}
&\int_V\left[
f(t,x)\overline{p}^\star(t,x,v)
-\nabla_x\cdot\left(D(t,x)\overline{p}^\star(t,x,v)\right)
-J_x^\star(t,x,v)
\right]dv \\
&\quad =f(t,x)p^\star(t,x)
-\nabla_x\cdot\left(D(t,x)p^\star(t,x)\right)
-J(t,x)=0.
\end{aligned}
\end{equation}
For each $i=1,\ldots,n$, define the component-wise equation
\begin{equation}
\label{eq:cross-diffusion-residual}
r_i(t,x,v)
=
f_i(t,x)\overline{p}^\star(t,x,v)
-\sum\nolimits_{j=1}^{n}\frac{\partial}{\partial x_j}
\left(D_{ij}(t,x)\overline{p}^\star(t,x,v)\right)
-J_{x,i}^\star(t,x,v).
\end{equation}
Then~\eqref{eq:cross-diffusion-equation} is equivalent to
\begin{equation}
\label{eq:cross-diffusion-row-equation}
\sum\nolimits_{j=1}^{n}\frac{\partial}{\partial v_j}
\left(D_{xv,ij}(t,x,v)\overline{p}^\star(t,x,v)\right)
=r_i(t,x,v),\quad i=1,\ldots,n.
\end{equation}
For a bounded connected domain $V$ with smooth boundary, let $\psi_i(t,x,\cdot)$ be the zero-mean solution of the Neumann--Poisson problem
\begin{equation}
\label{eq:cross-diffusion-poisson}
\begin{aligned}
\Delta_v\psi_i(t,x,v)&=r_i(t,x,v), &&v\in V, \\
\nabla_v\psi_i(t,x,v)\cdot n_{\partial V}(v)&=0, &&v\in\partial V, \\
\int_V\psi_i(t,x,v)dv&=0.
\end{aligned}
\end{equation}
Here, $n_{\partial V}$ is the outward unit normal to $\partial V$. The compatibility condition $\int_Vr_i(t,x,v)dv=0$ follows from~\eqref{eq:x-flow}, and the zero-mean condition fixes the additive constant. Under the stated regularity assumptions, define
\begin{equation}
\label{eq:cross-diffusion-from-potential}
D_{xv,ij}(t,x,v)
=
\frac{\partial_{v_j}\psi_i(t,x,v)}{\overline{p}^\star(t,x,v)},
\quad i,j=1,\ldots,n.
\end{equation}
Substituting into~\eqref{eq:cross-diffusion-row-equation} gives
\begin{equation}
\sum\nolimits_{j=1}^{n}\frac{\partial}{\partial v_j}
\left(D_{xv,ij}\overline{p}^\star\right)
=\sum\nolimits_{j=1}^{n}\frac{\partial^2\psi_i}{\partial v_j^2}
=\Delta_v\psi_i=r_i.
\end{equation}
Thus,~\eqref{eq:cross-diffusion-equation} holds. The Neumann condition also gives $(D_{xv}\overline{p}^\star)n_{\partial V}=0$ on $\partial V$. For $V=\mathbb R^n$, replace the boundary and zero-mean conditions by suitable decay and an additive normalization, and assume the corresponding Poisson problems admit solutions with vanishing flux at infinity.

To complete the augmented diffusion tensor, choose any $\eta>0$ and set
\begin{equation}
\label{eq:constructed-v-diffusion}
D_{vv}(t,x,v)
=D_{xv}(t,x,v)^\top D(t,x)^{-1}D_{xv}(t,x,v)+\eta I.
\end{equation}
The Schur complement satisfies $D_{vv}-D_{xv}^\top D^{-1}D_{xv}=\eta I\succ0$. Hence, $\overline{D}$ in~\eqref{eq:constructed-aug-diffusion-matrix} is positive definite.

\subsection{Construction of the $v$-Direction Probability Current}

In the interior $\Sigma^\circ$, the target current must satisfy
\begin{equation}
\label{eq:constructed-aug-conservation}
\frac{\partial\overline{p}^\star(t,x,v)}{\partial t}
=-\nabla_x\cdot J_x^\star(t,x,v)
-\nabla_v\cdot J_v^\star(t,x,v),\quad (x,v)\in\Sigma^\circ.
\end{equation}
Define the residual
\begin{equation}
\label{eq:constructed-v-current-residual}
r^v(t,x,v)
=-\frac{\partial\overline{p}^\star(t,x,v)}{\partial t}
-\nabla_x\cdot J_x^\star(t,x,v).
\end{equation}
Then $J_v^\star$ is required to solve
\begin{equation}
\label{eq:constructed-v-current-equation}
\nabla_v\cdot J_v^\star(t,x,v)=r^v(t,x,v).
\end{equation}
On a region where the entire $v$-fiber avoids the reset interfaces and the marginal equation has no reset contribution, we have
\begin{equation}
\label{eq:constructed-v-current-compatibility}
\begin{aligned}
\int_Vr^v(t,x,v)dv
&=-\frac{\partial}{\partial t}\int_V\overline{p}^\star(t,x,v)dv
-\nabla_x\cdot\int_VJ_x^\star(t,x,v)dv \\
&=-\frac{\partial p^\star(t,x)}{\partial t}
-\nabla_x\cdot J(t,x)=0.
\end{aligned}
\end{equation}
Under this compatibility condition, one possible choice is a gradient current. For bounded $V$, let $\varphi(t,x,\cdot)$ solve
\begin{equation}
\label{eq:constructed-v-current-poisson}
\begin{aligned}
\Delta_v\varphi(t,x,v)&=r^v(t,x,v), &&v\in V, \\
\nabla_v\varphi(t,x,v)\cdot n_{\partial V}(v)&=0, &&v\in\partial V, \\
\int_V\varphi(t,x,v)dv&=0.
\end{aligned}
\end{equation}
Define
\begin{equation}
\label{eq:constructed-v-current-from-potential}
J_v^\star(t,x,v)=\nabla_v\varphi(t,x,v).
\end{equation}
Then $\nabla_v\cdot J_v^\star=\Delta_v\varphi=r^v$ and $J_v^\star\cdot n_{\partial V}=0$. The same decay qualification as above applies when $V=\mathbb R^n$. The gradient choice fixes one admissible current; the density evolution alone does not uniquely determine $J_v^\star$.


\subsection{Recovery of the Augmented Drift and Diffusion}

Given $J_v^\star$, define the $v$-component of the drift in $\Sigma^\circ$ by
\begin{equation}
\label{eq:constructed-v-drift}
f_v(t,x,v)
=
\frac{
J_v^\star(t,x,v)
+\nabla_x\cdot\left(D_{xv}(t,x,v)^\top\overline{p}^\star(t,x,v)\right)
+\nabla_v\cdot\left(D_{vv}(t,x,v)\overline{p}^\star(t,x,v)\right)
}{\overline{p}^\star(t,x,v)}.
\end{equation}
All density factors in~\eqref{eq:constructed-v-drift} are the joint density $\overline{p}^\star$. 

As in Construction~\ref{cons:approximate-ashs}, choose
\begin{equation}
\label{eq:constructed-diffusion-factor}
\overline{g}(t,x,v)
=
\begin{bmatrix}
g(t,x) & 0 \\
g_{xv}(t,x,v) & g_v(t,x,v)
\end{bmatrix},
\quad
g_{xv}=D_{xv}^\top D^{-1}g,
\quad
g_v=\sqrt{2\eta}I.
\end{equation}
Using $gg^\top=2D$, we obtain
\begin{equation}
\frac{1}{2}\overline{g}\overline{g}^\top
=
\begin{bmatrix}
D & D_{xv} \\
D_{xv}^\top & D_{xv}^\top D^{-1}D_{xv}+\eta I
\end{bmatrix}
=\overline{D}.
\end{equation}
The resulting augmented dynamics are
\begin{equation}
\label{eq:constructed-time-dependent-ashs}
\begin{aligned}
dy&=\overline{f}(t,y)dt+\overline{g}(t,y)d\overline{w},
\quad (x,v)\notin\Sigma^-, \\
x^+&=x^-+v^-,\quad v^+=v^-,
\quad (x^-,v^-)\in\Sigma^-.
\end{aligned}
\end{equation}
Here, $\overline{w}=[w_x^\top,\ w_v^\top]^\top$ is a $2n$-dimensional Wiener process. Evaluating the probability current at the target density gives
\begin{equation}
\overline{f}(t,y)\overline{p}^\star(t,y)
-\nabla_y\cdot\left(\overline{D}(t,y)\overline{p}^\star(t,y)\right)
=
\begin{bmatrix}
J_x^\star(t,x,v) \\
J_v^\star(t,x,v)
\end{bmatrix}
=J^\star(t,y),\quad y\in\Sigma^\circ.
\end{equation}

In a region where $\chi^\circ=1$ and the marginal equation has no reset contribution, $p^\star(v\mid x)=\mathcal N(v;\mu,\delta^2I)$ is independent of $x$. Hence, $r_i=0$, and we may choose $D_{xv}=0$ and $J_v^\star=0$. With $\eta=\lambda\delta^2$,~\eqref{eq:constructed-v-drift} reduces to $f_v=-\lambda(v-\mu)$, giving the interior OU dynamics in Section~\ref{sec:cons-aug-shs}:
\begin{equation}
dv=-\lambda(v-\mu)dt+\sqrt{2\lambda}\,\delta\,dw_v.
\end{equation}
This expression applies on $V=\mathbb R^n$; a bounded auxiliary domain additionally requires a boundary mechanism compatible with the imposed no-flux condition.


\subsection{Numerical Realization of the Two Constructions}
\label{appx:construction-example}

We describe the one-dimensional realization of
Construction \ref{cons:partition-mixing-ashs} and \ref{cons:approximate-ashs}
used in the analytical example.
The continuous state satisfies
$dx=dt+\sigma\,dw$, with $\sigma=0.005$ and guard $x=1$.
The auxiliary state represents the reset displacement,
$v=x^+-1$, so that the augmented reset is deterministic:
\begin{equation}
    R(1,v)=(1+v,v).
\end{equation}
For numerical simulation, we restrict the auxiliary variable to
$V=[v_{\min},v_{\max}]=[-1.15,-0.55]$ and use reflecting
boundaries.
Consequently, the augmented reset image lies on
$x=1+v$, with $x\in[-0.15,0.45]$.

\paragraph{Construction 1: conditional and joint target densities.}
Let $\kappa(v)$ denote the reset-displacement density and let
$\rho(v)$ denote the Gaussian interior profile.
Before restriction to $V$, the profiles used in the numerical
example are
\begin{equation}
\begin{aligned}
    \kappa_0(v)
    &=
    \frac{1}{3}\sum_{c\in\{0,0.15,0.3\}}
    \mathcal{N}(v;c-1,0.03^2),\\
    \rho_0(v)
    &=\mathcal{N}(v;-0.85,0.2^2).
\end{aligned}
\end{equation}
For the normalized target construction on a bounded auxiliary
domain, these profiles are replaced by
\begin{equation}
    \kappa(v)=
    \frac{\kappa_0(v)}{\int_V\kappa_0(u)\,du},
    \qquad
    \rho(v)=
    \frac{\rho_0(v)}{\int_V\rho_0(u)\,du},
    \qquad v\in V.
\end{equation}
Since $R$ leaves $v$ unchanged, the outgoing and incoming reset
fluxes have the same auxiliary-coordinate law.
The implementation therefore uses the reset-displacement profile
as a local extension near both the guard and the projected
reset band.
This extension is a numerical choice; it does not by itself
establish the joint-density trace on the slanted reset image.

We blend the reset and interior profiles using a twice
continuously differentiable weight.
Define
\begin{equation}
s(u)=
\begin{cases}
0, & u\le0,\\
6u^5-15u^4+10u^3, & 0<u<1,\\
1, & u\ge1,
\end{cases}
\end{equation}
and set
\begin{equation}
\begin{aligned}
a_+(x)&=1-s\left(\frac{x-0.45}{0.20}\right),\\
a_-(x)&=s\left(\frac{x-0.70}{0.30}\right),\\
a(x)&=\max\{a_+(x),a_-(x)\}.
\end{aligned}
\end{equation}
The two transition regions are disjoint.
Thus, $a=1$ on the projected reset band and at the guard,
while $a=0$ for $x\in[0.65,0.70]$.
The conditional target and the augmented target are
\begin{equation}
\label{eq:app-example-target}
    q(v\mid x)=a(x)\kappa(v)+(1-a(x))\rho(v),
    \qquad
    \overline{p}^{\star}(t,x,v)=p^{\star}(t,x)q(v\mid x).
\end{equation}
Normalization of the local profiles gives
$\int_Vq(v\mid x)\,dv=1$, and hence
$\int_V\overline{p}^{\star}(t,x,v)\,dv=p^{\star}(t,x)$.

In the numerical implementation, the time-dependent marginal
used to evaluate the coefficients is obtained from a conservative
finite-volume approximation of the original hybrid
Fokker--Planck equation,
\begin{equation}
    \partial_t p+\partial_x J
    =\alpha(t)\kappa(x-1),
    \qquad
    J=p-D_0\partial_xp,
    \qquad
    D_0=\frac{\sigma^2}{2}.
\end{equation}
The source is supported on the reset band.
The computational interval is $[-0.15,1]$, with an absorbing
right boundary and a no-flux left boundary.
We use 600 spatial cells and ten finite-volume substeps per
particle-integration step.
The tabulated density and its spatial derivative provide the
time-dependent current $J(t,x)$.
Separate Monte Carlo ensembles are used to compare the SHS
and A-SHS marginal distributions.

\paragraph{Construction 2: cross-diffusion and auxiliary drift.}
For normalized profiles, the one-dimensional Neumann problem
reduces to a single cumulative integral,
\begin{equation}
\label{eq:app-example-shape}
    H(v)=\int_{v_{\min}}^v
    \bigl(\kappa(u)-\rho(u)\bigr)\,du.
\end{equation}
It satisfies
$H(v_{\min})=H(v_{\max})=0$ and
$H'(v)=\kappa(v)-\rho(v)$.
Prescribing $J_x^\star=qJ$ gives
\begin{equation}
    \partial_v
    \left(D_{xv}\overline{p}^{\star}\right)
    =-D_0p^{\star}\partial_xq.
\end{equation}
Therefore, a scalar cross-diffusion coefficient is
\begin{equation}
\label{eq:app-example-cross}
    D_{xv}(x,v)
    =-\frac{D_0a'(x)H(v)}{q(v\mid x)}.
\end{equation}
On a source-free region where
$\partial_tp^\star+\partial_xJ=0$, the auxiliary current can
be chosen as
\begin{equation}
\label{eq:app-example-v-current}
    J_v^\star(t,x,v)=-a'(x)J(t,x)H(v).
\end{equation}
Indeed,
$\partial_vJ_v^\star=-J\partial_xq
=-\partial_t\overline{p}^\star-\partial_x(qJ)$
on such a region.

Choose an independent auxiliary-noise amplitude $\sigma_v=0.5$
and write $\eta_v=\sigma_v^2/2$.
The remaining diffusion coefficient and the auxiliary drift are
\begin{equation}
\begin{aligned}
    D_{vv}
    &=\frac{D_{xv}^2}{D_0}+\eta_v,\\
    h_v
    &=
    \frac{
    J_v^\star
    +\partial_x(D_{xv}\overline{p}^\star)
    +\partial_v(D_{vv}\overline{p}^\star)
    }{\overline{p}^\star}.
\end{aligned}
\end{equation}
The corresponding continuous dynamics are
\begin{equation}
\label{eq:app-example-augmented-sde}
\begin{aligned}
    dx &= dt+\sigma\,dw,\\
    dv &= h_v(t,x,v)\,dt
    +\frac{2D_{xv}(x,v)}{\sigma}\,dw
    +\sigma_v\,du,
\end{aligned}
\end{equation}
where $w$ and $u$ are independent Wiener processes.
Sharing $w$ between the two equations realizes the prescribed
cross-diffusion.
At $x=1$, we apply $(x^+,v^+)=(1+v^-,v^-)$ without
resampling the auxiliary state.

In the Gaussian interior, $a=a'=0$, so that
$D_{xv}=J_v^\star=0$.
Away from numerical regularization and auxiliary boundaries,
the $v$-dynamics reduce to
\begin{equation}
    dv=-\lambda(v-\mu)\,dt+\sigma_v\,du,
    \qquad
    \mu=-0.85,\qquad
    \lambda=\frac{\sigma_v^2}{2\delta^2}=3.125,
    \qquad \delta=0.2.
\end{equation}
Thus, the OU process is the interior limit of the construction,
rather than the complete augmented dynamics.
In the transition regions, the shared noise and the
time-dependent current contribute additional terms.

\paragraph{Numerical approximation and scope.}
The displayed simulations use 5000 particles per system,
step size $\Delta t=5\times10^{-4}$, and final time
$T=3$.
Both particle ensembles start at $x(0)=0.25$.
The coefficient table uses a regularized initial density with
standard deviation $0.05$.
The auxiliary integral is tabulated on 4000 grid points,
density denominators are floored at $10^{-3}$, and the
auxiliary drift is clipped to $[-200,200]$.



\section{Details of the Numerical Experiments}
In this appendix, we provide the details of the numerical experiments.
Unless stated otherwise, values are the mean $\pm$ sample standard deviation over five training seeds.
Bold entries indicate the lowest mean among complete numeric entries for each metric within each row; rows with only one available configuration do not provide a comparative ranking.

\subsection{System Setup}
\label{apx:system-setup}

In this section, we consider eleven benchmark settings constructed
from six stochastic hybrid systems.
The \emph{\textbf{Bouncing Ball}}, \emph{\textbf{Torus}},
\emph{\textbf{Klein Bottle}}, and \emph{\textbf{Klein--Torus}}
systems have two-dimensional observed states.
We additionally consider a single planar bouncing ball with a
four-dimensional state and two interacting planar bouncing balls
with an eight-dimensional state.
Except for Klein--Torus, each system is evaluated with both
Gaussian-mixture and uniform reset distributions.
The continuous dynamics are deterministic in these experiments;
randomness enters through the reset mechanism.

For the Torus and Klein Bottle examples, the dynamics evolve on
the fundamental domain $M=[0,1]^2$ with the constant vector field
\begin{equation}
    \dot{x}=c,\qquad
    c=[2,2\sqrt{2}]^\top.
\end{equation}
The guard consists of two neighboring boundaries,
$S=S_1\cup S_2$, where
\begin{equation}
    S_1=\{x\in M:x_1=1\},\qquad
    S_2=\{x\in M:x_2=1\}.
\end{equation}
Let $[a]_1=a-\lfloor a\rfloor$ denote wrapping into $[0,1)$.
The stochastic reset maps are
\begin{equation}
\begin{aligned}
r_{\mathrm{t}}(x,\xi)
&=
\begin{cases}
(0,[x_2+\xi]_1), & x\in S_1,\\
([x_1+\xi]_1,0), & x\in S_2,
\end{cases}\\
r_{\mathrm{k}}(x,\xi)
&=
\begin{cases}
(0,[x_2+\xi]_1), & x\in S_1,\\
([1-x_1+\xi]_1,0), & x\in S_2.
\end{cases}
\end{aligned}
\end{equation}
When $\xi=0$, these maps recover the canonical boundary
identifications of the torus and Klein bottle, respectively.
For the stochastic experiments, the reset perturbation is sampled
independently at each event from
\begin{equation}
\xi\sim
\begin{cases}
\displaystyle
\frac{1}{3}\sum_{\mu\in\{-0.2,0,0.2\}}
\mathcal{N}(\mu,0.01^2), & \text{GMM},\\[4pt]
\mathcal{U}(-0.2,0.2), & \text{Uniform}.
\end{cases}
\end{equation}
Initial states are sampled uniformly from $[0,1]^2$.

The \emph{\textbf{Klein--Torus}} example has the same continuous
dynamics and guard, but randomly selects between the two canonical
reset rules.
The reset on $S_1$ is $(1,x_2)\mapsto(0,x_2)$.
On $S_2$, we independently draw $B\sim\operatorname{Bernoulli}(1/2)$
and apply
\begin{equation}
    r_{\mathrm{kt}}(x,B)
    =
    \begin{cases}
        (x_1,0), & B=0,\\
        (1-x_1,0), & B=1.
    \end{cases}
\end{equation}
There is no additive reset perturbation in this example:
the randomness is entirely due to the choice of reset rule.

The \emph{\textbf{Bouncing Ball}} system has state
$x=[p,v]^\top$, with height $p$ and vertical velocity $v$.
Its continuous dynamics are
\begin{equation}
    \dot{p}=v,\qquad \dot{v}=-g,\qquad g=9.81.
\end{equation}
The guard $S_{\mathrm{b}}=\{(p,v):p=0,\ v<0\}$ represents
contact with the ground.
At each impact, the reset is
\begin{equation}
    p^+=0,\qquad
    v^+=-\alpha v^-+\varepsilon,\qquad
    \varepsilon\sim\mathcal{N}(0,0.05^2),
\end{equation}
where the restitution coefficient is sampled independently from
\begin{equation}
\label{eq:app-restitution}
\alpha\sim
\begin{cases}
\displaystyle
\frac{1}{2}\mathcal{N}(0.5,0.01^2)
+\frac{1}{2}\mathcal{N}(0.9,0.01^2), & \text{GMM},\\[4pt]
\mathcal{U}(0.25,0.90), & \text{Uniform}.
\end{cases}
\end{equation}
The Gaussian components are sampled without truncation.
We initialize $p(0)\sim\mathcal{U}(0.2,1.5)$ and
$v(0)=s u$, where $u\sim\mathcal{U}(5,7)$ and
$s\in\{-1,1\}$ is equiprobable.

For the \emph{\textbf{Planar Bouncing Ball}} examples, we consider
$k\in\{1,2\}$ equal-mass balls with radius $r=0.05$.
Each ball has state
$x_i=[p_{x,i},p_{y,i},v_{x,i},v_{y,i}]^\top$,
and the full observed state is
$x=[x_1^\top,\ldots,x_k^\top]^\top\in\mathbb{R}^{4k}$.
The continuous dynamics are
\begin{equation}
    \dot{p}_{x,i}=v_{x,i},\qquad
    \dot{p}_{y,i}=v_{y,i},\qquad
    \dot{v}_{x,i}=0,\qquad
    \dot{v}_{y,i}=-g.
\end{equation}
The physical side walls are located at $p_x=\pm0.5$,
the floor is located at $p_y=0$, and there is no ceiling.
Consequently, the admissible ball-center positions satisfy
\begin{equation}
    -0.45\le p_{x,i}\le0.45,\qquad p_{y,i}\ge0.05.
\end{equation}
A wall or floor collision is triggered when a ball reaches the
corresponding contact boundary while moving into the surface.
The normal velocity component is reset according to
\begin{equation}
    v_{n,i}^+=-\alpha v_{n,i}^-,
\end{equation}
while the tangential component remains unchanged.
The restitution coefficient follows
\eqref{eq:app-restitution}.
Unlike the one-dimensional bouncing ball, these planar systems
do not include additive velocity noise at impact.

For two balls, contact additionally occurs when
$\|p_i-p_j\|=2r$ and the balls are approaching each other.
Let
\begin{equation}
    n_{ij}=\frac{p_i-p_j}{\|p_i-p_j\|},\qquad
    s_{ij}=(v_i^--v_j^-)^\top n_{ij}<0.
\end{equation}
The equal-mass collision reset is
\begin{equation}
\begin{aligned}
    v_i^+ &= v_i^-
    -\frac{1+\alpha}{2}s_{ij}n_{ij},\\
    v_j^+ &= v_j^-
    +\frac{1+\alpha}{2}s_{ij}n_{ij},
\end{aligned}
\end{equation}
with a newly sampled restitution coefficient.
Initial positions are non-overlapping, and initial velocities
are rescaled to satisfy a total mechanical-energy budget of
at most $kg$, for unit masses.
The implementation corrects numerical overlap and uses a small
inward position margin of $10^{-6}$ after contact.

For each benchmark setting, we generate 4096 training trajectories
and 512 independent test trajectories with a nominal duration of
$5\,\mathrm{s}$ and integration step $\Delta t=0.01\,\mathrm{s}$.
Collision and boundary-crossing events are handled during simulation,
so the stored observation timestamps need not form a uniform grid.
We normalize each coordinate using the training-set mean and standard
deviation and reuse those statistics for testing.
All qualitative trajectory plots are transformed back to the original
physical or fundamental-domain coordinates.

\subsection{Model Structure}
\label{apx:model-structure}

We compare the proposed model with CHyLL
\citep{teng2026embedding}, Latent SDE
\citep{li2020scalable}, SDE Matching
\citep{bartosh2025sde}, and Neural Stochastic Flows
\citep{kiyohara2025neural}.
We additionally evaluate Latent SDE augmented with
$\mathcal{L}_x$ and $\mathcal{L}_{\mathrm{x}}^{\mathrm{traj}}$.
This variant uses the same architecture as Latent SDE and is
denoted by \texttt{Latent SDE + x losses} in the figures.
The stochastic baselines use the authors' implementation cores,
with conditional initial-state adapters where needed for
forecasting from a given observation.

Let $n=d_x$ denote the observed-state dimension.
All latent-state models use $d_z=4n$ in the default comparison.
NSF instead models transitions directly in the observed state
space and does not use a latent SDE of dimension $4n$.
We use the standard network widths for b-ball, Torus, Klein,
Klein--Torus, and the single planar ball.
For the two-ball examples, we double the hidden widths within
each model family.

\begin{table}[t]
\scriptsize
\centering
\caption{Hidden-width settings. The large configuration is used
only for the two planar balls. $H$ is the hidden width for
Proposed and CHyLL; $L$ and $C$ are the hidden and context widths
for Latent SDE; $M$ is the hidden width for SDE Matching;
$F$ is the hidden width for NSF.}
\label{tab:app-network-widths}
\begin{tabular}{lccccc}
\toprule
Configuration & $H$ & $L$ & $C$ & $M$ & $F$\\
\midrule
Standard & 256 & 128 & 64  & 100 & 64\\
Large    & 512 & 256 & 128 & 200 & 128\\
\bottomrule
\end{tabular}
\end{table}

The \emph{\textbf{Proposed}} model uses a conditional diffusion
sampler to generate latent initial states from the observed
initial condition.
Its score network has two hidden layers of width $H$ with
SiLU activations.
We use 20 DDIM sampling steps and a scalar diffusion-time input.
The discrete noise schedule is linearly spaced between
$0.1/20$ and $4/20$.
The latent drift and diffusion networks each contain two
hidden layers of width $H$, and the decoder contains three
hidden layers of width $H$.
These networks use ReLU hidden activations and linear outputs.

The \emph{\textbf{CHyLL}} baseline
\citep{teng2026embedding} uses a deterministic MLP encoder,
a latent ODE, and an MLP decoder.
Their hidden-layer widths are $[H]\times2$,
$[H]\times2$, and $[H]\times3$, respectively.
The encoder, vector field, and decoder use ReLU hidden activations.
This baseline is trained with its deterministic reconstruction
and dynamics objectives, rather than with the stochastic
distribution-matching objectives.
For a fixed initial observation, repeated predictions coincide.


The \emph{\textbf{Latent SDE}} baseline
\citep{li2020scalable} uses a GRU with hidden width $L$,
followed by a linear projection to a context of dimension $C$.
The initial posterior parameters are obtained from this context
through a linear layer.
Both the posterior and prior drift networks contain two
Softplus hidden layers of width $L$.
The diffusion is diagonal: each latent coordinate has a
scalar-input network with one Softplus hidden layer of width $L$
and a sigmoid output.
A linear projection maps latent states to observations.
For conditional forecasting, an additional MLP with two
Softplus hidden layers of width $L$ maps the initial observation
to the mean and log standard deviation of a diagonal Gaussian
initial prior.
The observation likelihood has fixed standard deviation $0.01$
in normalized coordinates.
The KL weight is annealed to one over the first 1000 updates.
The \texttt{Latent SDE + x losses} variant adds
$\mathcal{L}_x$ and $\mathcal{L}_{\mathrm{x}}^{\mathrm{traj}}$
with unit weights, without changing these networks.

The \emph{\textbf{SDE Matching}} baseline
\citep{bartosh2025sde} uses a GRU of width $M$ and a
time-conditioned MLP with two SiLU hidden layers of width $M$
to parameterize its Gaussian posterior process.
The prior drift has one Softplus hidden layer of width $M$.
Its diagonal diffusion uses one scalar network per latent
coordinate, each with one Softplus hidden layer of width $M$
and a sigmoid output.
The observation map is linear, with fixed likelihood standard
deviation $0.01$.
We condition the initial Gaussian prior on the initial
observation using an MLP with one Softplus hidden layer of
width $M$.
Training uses the official simulation-free objective.
The implementation uses the fixed time conversion
$\tilde{t}=t/0.31$ and returns forecasts in the original time units.
At evaluation, the observation mean is used without adding
likelihood noise.

The \emph{\textbf{NSF}} baseline
\citep{kiyohara2025neural} uses the official autonomous
affine-coupling stochastic-flow and bridge models.
The Gaussian-parameter and coupling-conditioner networks each
have two hidden layers of width $F$ with SiLU activations,
and each flow uses four affine-coupling layers.
Training alternates five bridge updates with one flow update,
using the official consistency objective.
Predictive trajectories are generated recursively from the
conditional transition kernels, rather than by independently
sampling each time marginal.

\begin{table}[t]
\scriptsize
\centering
\caption{Network components used in the baseline comparison.
$[h]\times \ell$ denotes $\ell$ hidden layers of width $h$;
output layers are not counted.
Widths are specified in \Cref{tab:app-network-widths}.
Latent SDE with additional losses has the same architecture
as Latent SDE.}
\label{tab:app-model-components}
\resizebox{\linewidth}{!}{%
\begin{tabular}{lccccc}
\toprule
Component & Proposed & CHyLL & Latent SDE
& SDE Matching & NSF\\
\midrule
Initial-state network
& Score: $[H]\times2$
& $[H]\times2$
& $[L]\times2$
& $[M]\times1$
& Conditional flow\\
Posterior context
& N/A
& N/A
& GRU($L$), context $C$
& GRU($M$), $[M]\times2$
& Bridge model\\
Prior drift
& $[H]\times2$
& $[H]\times2$
& $[L]\times2$
& $[M]\times1$
& N/A\\
Diffusion
& $[H]\times2$, matrix
& N/A
& $[L]\times1$, per coordinate
& $[M]\times1$, per coordinate
& N/A\\
Observation decoder
& $[H]\times3$
& $[H]\times3$
& Linear
& Linear
& Observed-state flow\\
Flow conditioner
& N/A
& N/A
& N/A
& N/A
& $[F]\times2$\\
\bottomrule
\end{tabular}%
}
\end{table}

We use a learning rate of $10^{-3}$, batch size 256, and
training windows of 32 stored observations.
Each run uses 10,000 training steps; for NSF, these are
10,000 flow updates with the additional alternating bridge updates.
The proposed model uses ten conditional samples during training.
All methods are evaluated over five training seeds,
$\{1101,2202,3303,4404,5505\}$.

The latent ODE/SDE forecasts use Euler or Euler--Maruyama
integration on the supplied prediction grid.
For a 32-point grid, this gives 31 integration steps.
Predictions are interpolated to the recorded observation times
when needed.
NSF instead recursively samples its transition kernels.
Forecasting conditions only on the initial observation;
posterior access to a training window is not used to provide
future observations during testing.

For the horizon-based comparisons, we evaluate the saved
checkpoints on all 512 test trajectories using ten predictions
per initial condition.
We report results at $1\,\mathrm{s}$, $3\,\mathrm{s}$,
and the full stored horizon.
For a fixed physical horizon, we retain observations up to the
time cutoff and pad shorter retained sequences by repeating
their final retained state and timestamp.
The same protocol is applied across methods.
Metrics are computed in normalized coordinates, while the
qualitative figures display unnormalized trajectories.
\subsection{Comparative Studies}
\label{appx:baselines}
\subsubsection{Horizon 1s}
\Cref{tab:baseline-conditional-energy-1s,tab:baseline-trajectory-energy-1s} report the conditional and unconditional path distribution losses, respectively, both based on energy distance.
Evaluation uses 512 test trajectories in normalized state space, 32-point windows, and ten predictive samples per condition.
Latent SDE~\citep{li2020scalable} and SDE Matching~\citep{bartosh2025sde} are our task-adapted reimplementations; the $\star$ variant of Latent SDE additionally uses the state distribution loss $\mathcal L_x$ and the unconditional path distribution loss $\mathcal L_{\mathrm{x}}^{\mathrm{traj}}$.
NSF~\citep{kiyohara2025neural} uses the official implementation and recursively samples trajectories from conditional transition kernels.
Diverged denotes a numerical failure or either test-set path distribution loss exceeding $10^3$ in any seed.
The same criterion applies to CHyLL; surviving seeds are not averaged separately.
\TableBaselineConditionalEnergyHorizonOneSecond
\TableBaselineTrajectoryEnergyHorizonOneSecond

\subsubsection{Horizon 3s}
\Cref{tab:baseline-conditional-energy-3s,tab:baseline-trajectory-energy-3s} report the conditional and unconditional path distribution losses over $3\,\mathrm{s}$ from the first stored observation.
Our method yields complete numeric results and the lowest mean losses for both metrics across all 11 benchmark settings.
SDE Matching triggers the divergence criterion in all settings, while both Latent SDE variants diverge on all five topological benchmarks.

\TableBaselineConditionalEnergyHorizonThreeSeconds
\TableBaselineTrajectoryEnergyHorizonThreeSeconds

\subsubsection{Horizon around 5s}
\Cref{tab:baseline-conditional-energy-full,tab:baseline-trajectory-energy-full} report the losses over the complete stored trajectory, approximately $5\,\mathrm{s}$.
Our method yields complete numeric results in all 11 settings and achieves the lowest mean losses wherever a competing method also has complete results.
On Torus (GMM), 1-ball (GMM), and 1-ball (Uniform), all competing methods trigger the divergence criterion.

\TableBaselineConditionalEnergyHorizonFull
\TableBaselineTrajectoryEnergyHorizonFull

\subsection{Ablations}
\label{appx:ablations}
Percentage changes in the ablation summaries are computed per benchmark relative to the default model and then averaged over the 11 settings; unconditional results are given in parentheses.
In the loss-function ablations, the indicated objectives are removed by setting their weights to zero; other settings are unchanged.

\subsubsection{Horizon 1s}

\textbf{Latent Dimension:} \Cref{tab:ablation-split-latent-conditional-energy-loss-1s,tab:ablation-split-latent-x-trajectory-loss-1s} report the conditional and unconditional path distribution losses for $d_z\in\{d_x,2d_x,3d_x,4d_x\}$.
The diffusion noise dimension follows $d_z$ for the 1D ball and topological systems and is fixed at $d_x/2$ for the ball systems; other settings are unchanged.

\textbf{Stochastic Encoder:} \Cref{tab:ablation-split-sampler-conditional-energy-loss-1s} compares the conditional diffusion sampler with a deterministic MLP encoder using both path distribution losses.
Both variants use ten SDE trajectories per condition; other settings are unchanged.

\textbf{Loss Function Design:} \Cref{tab:ablation-split-weights-conditional-energy-loss-1s,tab:ablation-split-weights-x-trajectory-loss-1s} report the conditional and unconditional path distribution losses for each benchmark under the loss-function ablations.
As summarized in \Cref{tab:ablation-summary-1s}, removing $\mathcal L_z$ increases the mean conditional (unconditional) loss by $19.5\%$ ($246.6\%$), while removing both path distribution losses increases them by $23.3\%$ ($54.7\%$).
Both modifications worsen both metrics in all 11 settings.


\TableLatentConditionalEnergyHorizonOneSecond

\TableLatentTrajectoryEnergyHorizonOneSecond

\TableSamplerAblationHorizonOneSecond

\TableLossAblationConditionalEnergyHorizonOneSecond

\TableLossAblationUnconditionalEnergyHorizonOneSecond

\subsubsection{Horizon 3s}

\textbf{Latent Dimension:} \Cref{tab:ablation-split-latent-conditional-energy-loss-3s,tab:ablation-split-latent-x-trajectory-loss-3s} report the latent-dimension ablation over $3\,\mathrm{s}$.
As summarized in \Cref{tab:ablation-summary-3s}, using $d_z=d_x$ increases the mean conditional (unconditional) loss by $25.9\%$ ($417.1\%$), with both metrics worsening in all 11 settings.
Using $d_z=2d_x$ or $3d_x$ gives smaller mean increases of $1.5\%$ ($15.0\%$) or $0.7\%$ ($14.4\%$), respectively.

\textbf{Stochastic Encoder:} \Cref{tab:ablation-split-sampler-conditional-energy-loss-3s} reports the sampler ablation over $3\,\mathrm{s}$.
Replacing the stochastic encoder with a deterministic MLP increases the mean losses by $2.9\%$ ($13.9\%$), with degradation in $7/11$ ($9/11$) settings.

\textbf{Loss Function Design:} \Cref{tab:ablation-split-weights-conditional-energy-loss-3s,tab:ablation-split-weights-x-trajectory-loss-3s} report the conditional and unconditional path distribution losses for each benchmark under the loss-function ablations.
As summarized in \Cref{tab:ablation-summary-3s}, removing $\mathcal L_z$ increases the mean losses by $63.8\%$ ($1611.9\%$), while removing both path distribution losses increases them by $12.5\%$ ($24.6\%$).
Both modifications worsen both metrics in all 11 settings.

\TableAblationSummaryHorizonThreeSeconds

\TableLatentConditionalEnergyHorizonThreeSeconds

\TableLatentTrajectoryEnergyHorizonThreeSeconds

\TableSamplerAblationHorizonThreeSeconds

\TableLossAblationConditionalEnergyHorizonThreeSeconds

\TableLossAblationUnconditionalEnergyHorizonThreeSeconds

\subsubsection{Horizon around 5s}

\textbf{Latent Dimension:} \Cref{tab:ablation-split-latent-conditional-energy-loss-full,tab:ablation-split-latent-x-trajectory-loss-full} report the latent-dimension ablation over the complete stored trajectory.
Using $d_z=d_x$ triggers the divergence criterion on b-ball (Uniform), while $d_z=2d_x$ triggers it on 2-balls (GMM).
Using $d_z=3d_x$ increases the mean losses by $20.9\%$ ($396.2\%$), largely driven by 2-balls (GMM), where the unconditional loss reaches $82.6$, compared with $1.89$ for the default $d_z=4d_x$.
These results show that comparable short-horizon performance does not necessarily imply comparable performance over the full trajectory.

\textbf{Stochastic Encoder:} \Cref{tab:ablation-split-sampler-conditional-energy-loss-full} reports the sampler ablation over the complete stored trajectory.
Replacing the stochastic encoder with a deterministic MLP increases the mean losses by $3.1\%$ ($14.2\%$), with degradation in $7/11$ ($10/11$) settings.

\textbf{Loss Function Design:} \Cref{tab:ablation-split-weights-conditional-energy-loss-full,tab:ablation-split-weights-x-trajectory-loss-full} report the conditional and unconditional path distribution losses for each benchmark under the loss-function ablations.
Removing $\mathcal L_z$ triggers the divergence criterion on 2-balls (GMM).
As summarized in \Cref{tab:ablation-summary-full}, removing both path distribution losses increases the mean losses by $9.9\%$ ($17.4\%$), with both metrics worsening in all 11 settings.

\TableAblationSummaryHorizonFull

\TableLatentConditionalEnergyHorizonFull

\TableLatentTrajectoryEnergyHorizonFull

\TableSamplerAblationHorizonFull

\TableLossAblationConditionalEnergyHorizonFull

\TableLossAblationUnconditionalEnergyHorizonFull

\clearpage

\subsection{Visualizations}
\label{appx:visualizations}
We provide qualitative comparisons on all eleven datasets in
Figures~\ref{fig:app-bb1-gmm}--\ref{fig:app-bb2-two-uniform}.
The rows follow the method names shown in the legends:
\texttt{Proposed} denotes our method;
\texttt{CHyLL} denotes the deterministic latent ODE framework
of \citet{teng2026embedding};
\texttt{Latent SDE} denotes the variational latent SDE baseline
\citep{li2020scalable};
\texttt{Latent SDE + x losses} denotes the same baseline augmented
with $\mathcal{L}_x$ and $\mathcal{L}_{\mathrm{x}}^{\mathrm{traj}}$;
\texttt{SDE Matching} denotes the simulation-free variational
baseline \citep{bartosh2025sde}; and
\texttt{NSF} denotes the conditional stochastic-flow baseline
\citep{kiyohara2025neural}.

For each dataset, all methods are conditioned on the same initial
observation from a fixed test trajectory.
Blue curves show the ground-truth (GT) trajectory, and orange curves
show five predictive rollouts.
One rollout is emphasized for readability without selection by
prediction error.
The five CHyLL predictions coincide because this baseline is deterministic.
We use training seed 1101 and test trajectory index 0 throughout.
Predictions are generated over the full recorded horizon
(approximately five seconds), without subsequent conditioning on GT.
All signals are displayed in their original, unnormalized coordinates.
These plots assess full-horizon forecasting behavior rather than
directly visualizing the short-window errors reported in the tables.
Since the systems are stochastic, individual predictive samples are
not expected to reproduce the particular GT realization exactly.

For two-coordinate systems, the two columns show the state coordinates
as functions of time.
For the planar bouncing-ball systems, each row additionally includes
spatial trajectories: the four time-series panels show
$(x,y,v_x,v_y)$ for the single ball or the first ball,
while the right-hand panels show the spatial paths of all balls.
The boundary lines indicate ball-center contact locations,
$x=\pm0.45$ and $y=0.05$, accounting for the ball radius $r=0.05$.

Time-series limits are initially determined by the GT range with
20\% padding on each side; spatial panels initially show the
collision region with additional margins.
These limits are retained whenever at least one predicted point
is visible in the corresponding panel.
Only when no predicted point is visible are the limits expanded
to include the actual signal range.
Thus, predictions may leave the displayed region without being
numerically non-finite.
The label \texttt{Non-finite} instead indicates that a complete
finite rollout could not be obtained; only GT is shown in that row.

\begin{figure}[p]
    \centering
    \includegraphics[width=\linewidth]{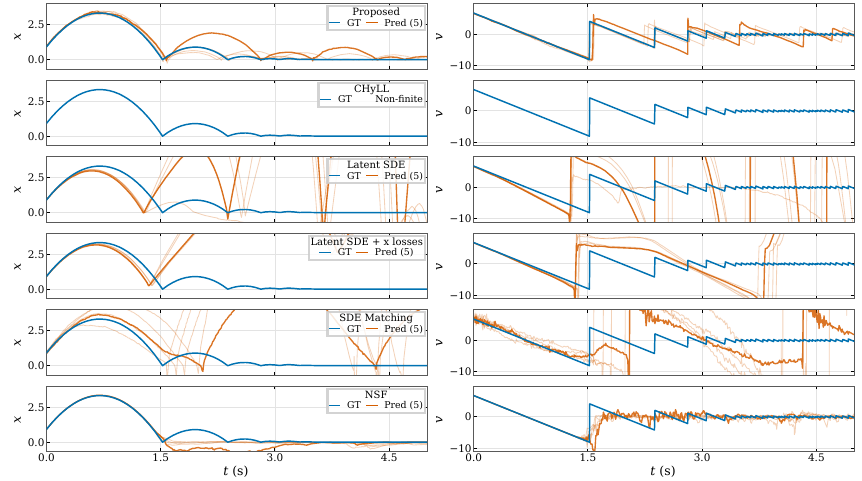}
    \caption{b-ball (GMM): one-dimensional bouncing ball with
    GMM-distributed restitution coefficients. The two columns
    show position and velocity over time.}
    \label{fig:app-bb1-gmm}
\end{figure}

\begin{figure}[p]
    \centering
    \includegraphics[width=\linewidth]{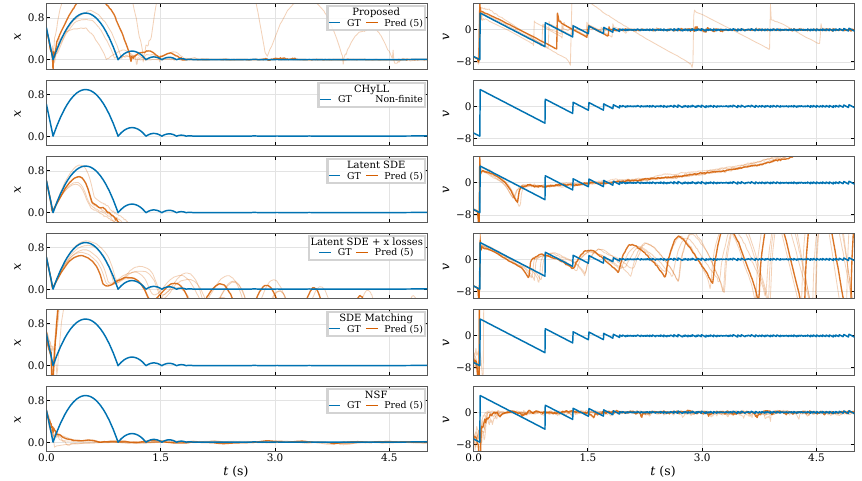}
    \caption{b-ball (Uniform): one-dimensional bouncing ball with
    uniformly distributed restitution coefficients. The two
    columns show position and velocity over time.}
    \label{fig:app-bb1-uniform}
\end{figure}

\begin{figure}[p]
    \centering
    \includegraphics[width=\linewidth]{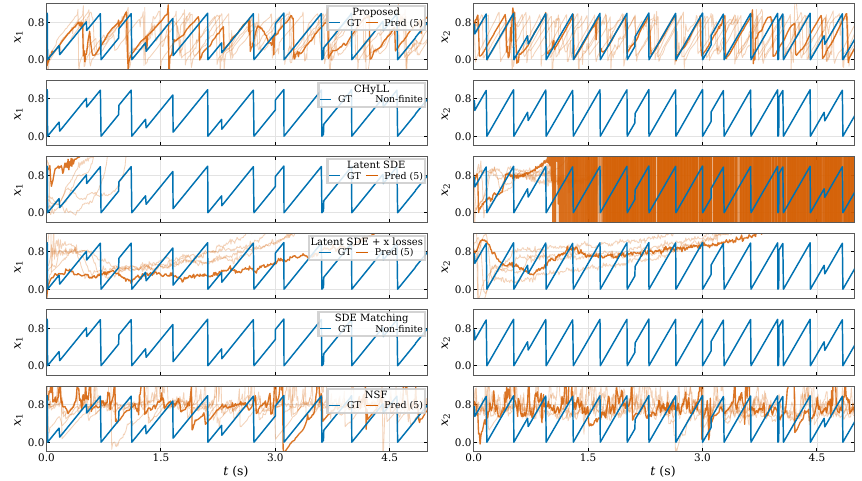}
    \caption{Torus (GMM): coordinate trajectories in the fundamental
    domain with GMM-distributed reset perturbations.}
    \label{fig:app-torus-gmm}
\end{figure}

\begin{figure}[p]
    \centering
    \includegraphics[width=\linewidth]{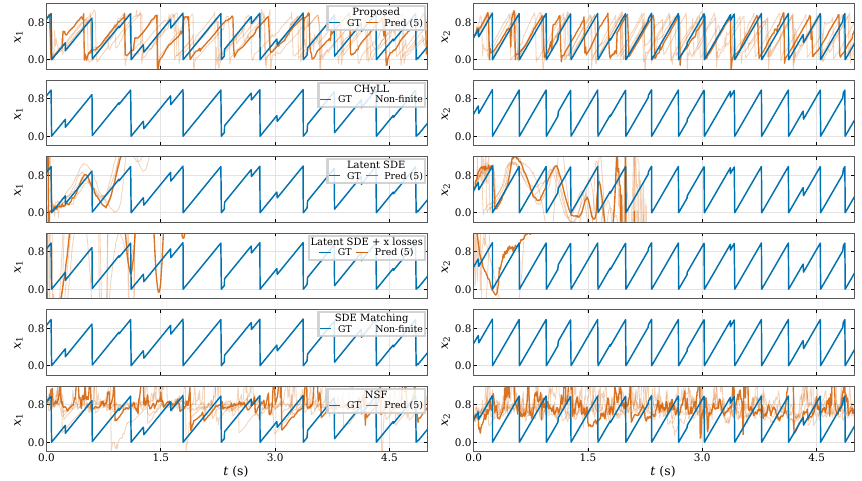}
    \caption{Torus (Uniform): coordinate trajectories in the fundamental
    domain with uniformly distributed reset perturbations.}
    \label{fig:app-torus-uniform}
\end{figure}

\begin{figure}[p]
    \centering
    \includegraphics[width=\linewidth]{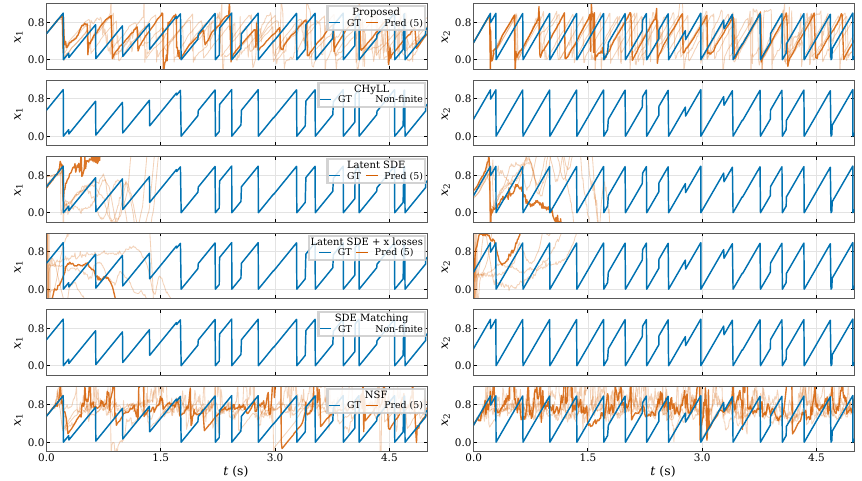}
    \caption{Klein (GMM): coordinate trajectories in the fundamental
    domain with GMM-distributed reset perturbations and
    Klein-bottle boundary gluing.}
    \label{fig:app-klein-gmm}
\end{figure}

\begin{figure}[p]
    \centering
    \includegraphics[width=\linewidth]{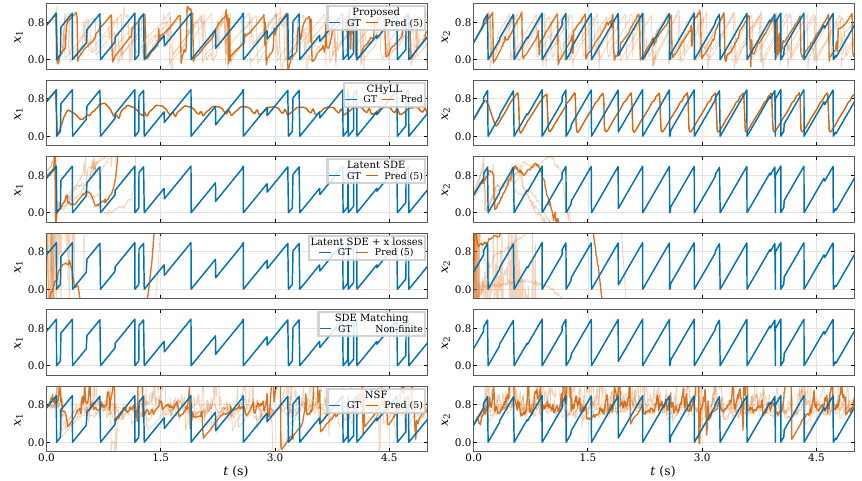}
    \caption{Klein (Uniform): coordinate trajectories in the fundamental
    domain with uniformly distributed reset perturbations and
    Klein-bottle boundary gluing.}
    \label{fig:app-klein-uniform}
\end{figure}

\begin{figure}[p]
    \centering
    \includegraphics[width=\linewidth]{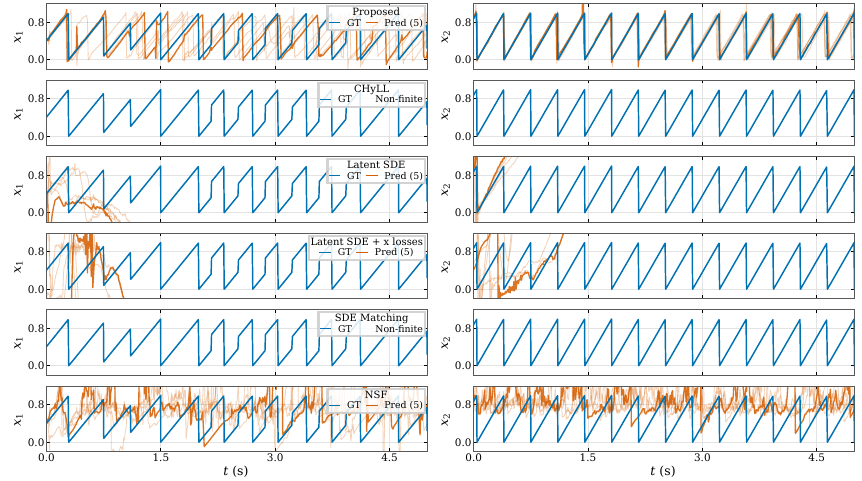}
    \caption{Klein--Torus: coordinate trajectories with random
    selection between canonical Klein-bottle and torus resets
    at the top boundary, each with probability $1/2$.}
    \label{fig:app-kt}
\end{figure}

\begin{figure}[p]
    \centering
    \includegraphics[width=\linewidth]{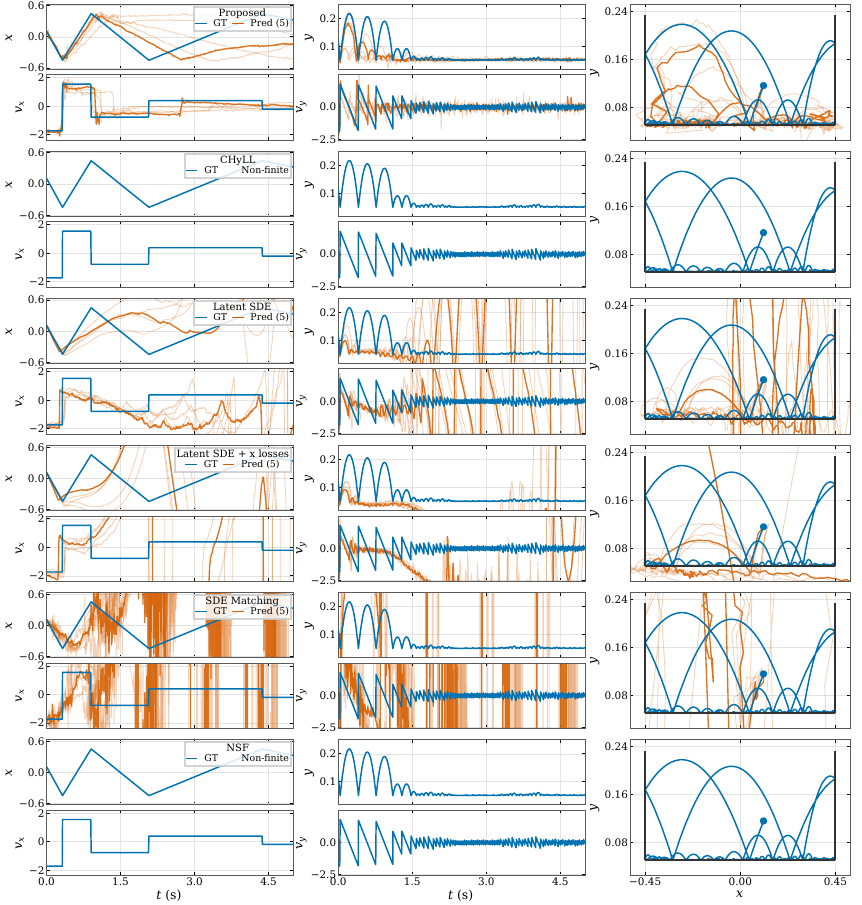}
    \caption{1-ball (GMM): a single planar bouncing ball with
    GMM-distributed wall and floor restitution coefficients.
    The left $2\times2$ panels show $(x,y)$ above $(v_x,v_y)$;
    the right panel shows the spatial path with radius-adjusted
    ball-center contact boundaries.}
    \label{fig:app-bb2-one-gmm}
\end{figure}

\begin{figure}[p]
    \centering
    \includegraphics[width=\linewidth]{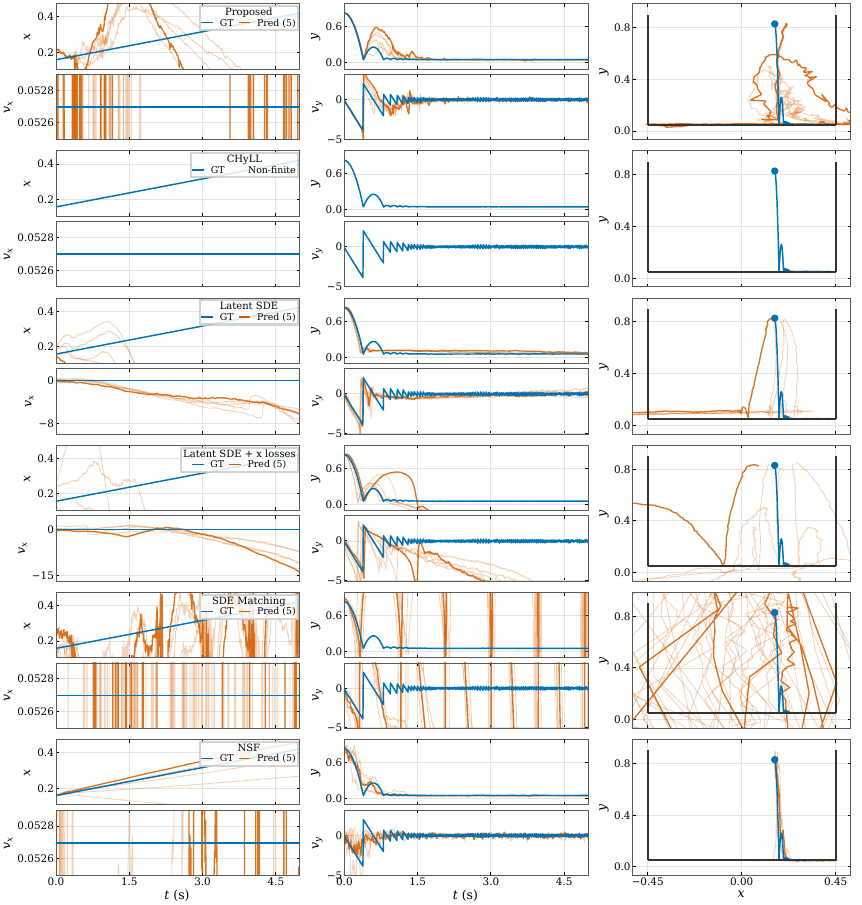}
    \caption{1-ball (Uniform): a single planar bouncing ball with
    uniformly distributed wall and floor restitution coefficients.
    The left panels show the four state components over time;
    the right panel shows the spatial path and ball-center
    contact boundaries.}
    \label{fig:app-bb2-one-uniform}
\end{figure}

\begin{figure}[p]
    \centering
    \includegraphics[width=\linewidth]{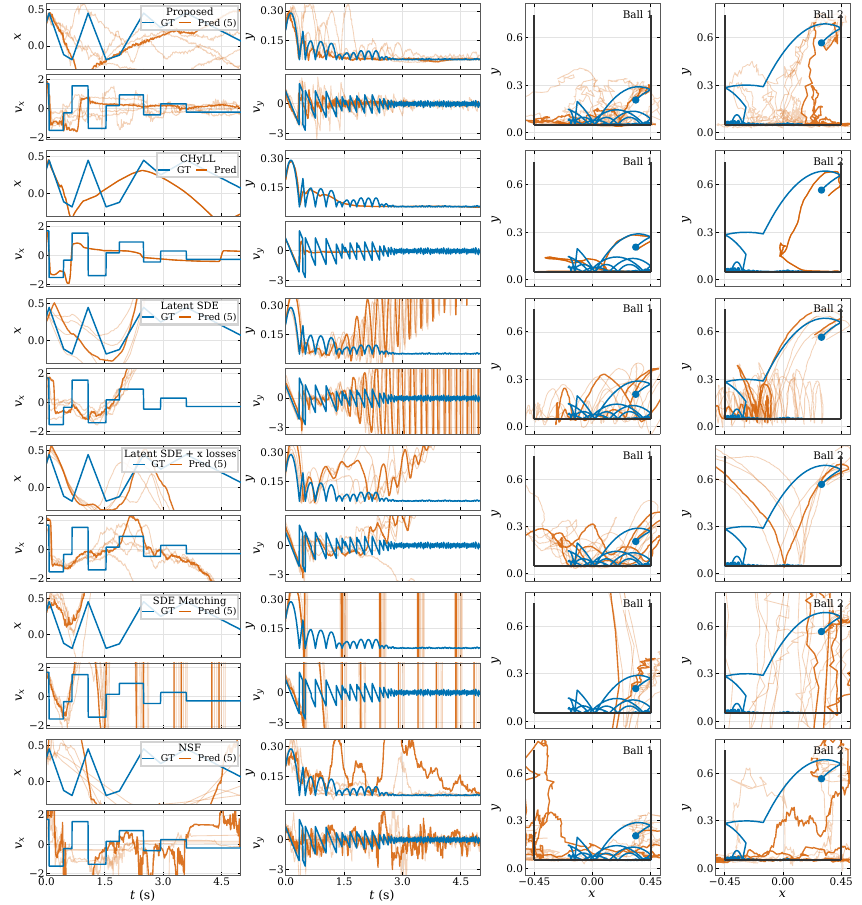}
    \caption{2-balls (GMM): two interacting planar bouncing balls
    with GMM-distributed wall and floor restitution coefficients.
    The left panels show the first ball's four state components;
    the two right panels show the spatial paths of both balls
    and their ball-center contact boundaries.}
    \label{fig:app-bb2-two-gmm}
\end{figure}

\begin{figure}[p]
    \centering
    \includegraphics[width=\linewidth]{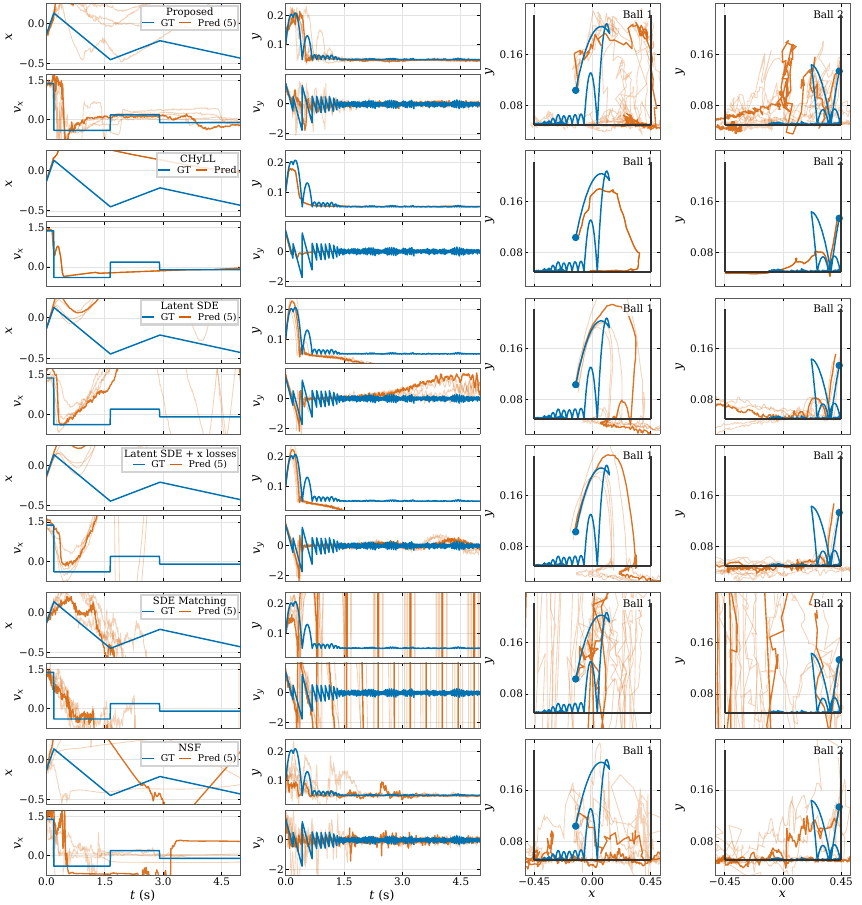}
    \caption{2-balls (Uniform): two interacting planar bouncing balls
    with uniformly distributed wall and floor restitution coefficients.
    The left panels show the first ball's four state components;
    the two right panels show the spatial paths of both balls
    and their ball-center contact boundaries.}
    \label{fig:app-bb2-two-uniform}
\end{figure}

\end{document}